\documentclass[onefignum,onetabnum]{siamonline220329}

\usepackage{lipsum}
\usepackage{amsfonts}
\usepackage{graphicx}
\usepackage{epstopdf}
\usepackage{todonotes}
\usepackage{mathtools}
\usepackage{mathrsfs}
\usepackage{makecell}
\usepackage{amsmath,amsfonts,amssymb}
\usepackage{multirow}
\usepackage{verbatim}
\usepackage{caption}
\usepackage{pgfkeys}

\usepackage{arydshln}
\usepackage{tikz,pgflibraryshapes}
\usetikzlibrary{shadows}
\usetikzlibrary{arrows}

\usepackage{dirtytalk}

\usepackage{upgreek}

\usepackage{algorithmic}
\ifpdf
  \DeclareGraphicsExtensions{.eps,.pdf,.png,.jpg}
\else
  \DeclareGraphicsExtensions{.eps}
\fi

\newcommand{\R}{{\mathbb R}}

\newcommand{\E}{{\mathbb E}}

\newcommand{\F}{{\mathcal{F}}}

\newcommand{\Hc}{{\mathcal{H}}}
\newcommand{\Lc}{{\mathcal{L}}}
\newcommand{\Gc}{{\mathcal{G}}}

\def\V{\mathcal{V}}
\def\U{\mathcal{U}}

\newcommand{\Ub}{{\bf{U}}}
\newcommand{\ub}{{\bf{u}}}
\newcommand{\Vb}{{\bf{V}}}
\newcommand{\vb}{{\bf{v}}}

\newcommand{\cN}{{\mathcal{N}}}
\newcommand{\Tr}{\operatorname{Tr}}
\newcommand{\Cov}{\operatorname{Cov}}

\def\<{\big\langle}
\def\>{\big\rangle}

\newcommand{\norm}[1]{\vert\vert#1\vert\vert}

\DeclarePairedDelimiterX{\ip}[2]{\langle}{\rangle}{#1, #2}

\DeclareMathOperator*{\argmin}{arg\,min}

\usepackage{color,soul}
\usepackage{mathtools}
\definecolor{darkred}{rgb}{.7,0,0}
\definecolor{darkgreen}{rgb}{.1,.5,0}

\newcommand{\md}[1]{{\color{orange}{#1}}}

\usepackage{caption}
\usepackage{subcaption}

\usepackage{enumitem}
\setlist[enumerate]{leftmargin=.5in}
\setlist[itemize]{leftmargin=.5in}
\newsiamremark{remark}{Remark}
\newsiamremark{hypothesis}{Hypothesis}
\crefname{hypothesis}{Hypothesis}{Hypotheses}
\newsiamthm{claim}{Claim}
\newsiamremark{example}{Example}
\newtheorem{Theorem}{Theorem}[section]

\newtheorem{Proposition}[Theorem]{Proposition}
\newtheorem{Lemma}[Theorem]{Lemma}

\newtheorem{Definition}[Theorem]{Definition}

\newtheorem{Condition}[Theorem]{Condition}
\newtheorem{Problem}{Problem}

\title{Kernel Methods are Competitive for Operator Learning}
\author{Pau Batlle\thanks{Computing and Mathematical Sciences, Caltech, Pasadena, CA
    (\email{pbatllef@caltech.edu}, \email{mdarcy@caltech.edu},
    \email{owhadi@caltech.edu})}
  \and Matthieu Darcy\footnotemark[1] \footnote{Corresponding author.}
  \and Bamdad Hosseini\thanks{Department of Applied Mathematics, University of Washington, Seattle, WA
  (\email{bamdadh@uw.edu})}
  \and Houman Owhadi\footnotemark[1]
}
\usepackage{amsopn}

\ifpdf
\hypersetup{
  pdftitle={Kernel Methods are Competitive for Operator Learning},
  pdfauthor={}
}
\fi

\begin{document}

\maketitle

% REQUIRED
\begin{abstract}
    We present a general kernel-based framework for learning operators between Banach spaces 
    along with a priori error analysis and comprehensive numerical 
    comparisons with popular neural net (NN)  approaches such as Deep Operator Networks (DeepONet) \cite{Lu2021} and Fourier Neural Operator (FNO) \cite{fno}. 
    % In recent years, 
    % NN operator learning methods  have gained a lot of popularity 
    % and have become the subject of intense research. Although these methods have been benchmarked 
    % against each other, they have not been benchmarked against more classical and simpler methods 
    % including kernel-based approaches. 
    We consider the setting where the input/output spaces of target operator $\mathcal{G}^\dagger\,:\, \mathcal{U}\to \mathcal{V}$ are reproducing kernel Hilbert spaces (RKHS),  the data comes in the form of partial  observations $\phi(u_i), \varphi(v_i)$ of input/output functions $v_i=\mathcal{G}^\dagger(u_i)$ ($i=1,\ldots,N$), and the measurement operators $\phi\,:\, \mathcal{U}\to \mathbb{R}^n$ and $\varphi\,:\, \mathcal{V} \to \R^m$ are linear.  Writing $\psi\,:\, \mathbb{R}^n \to \mathcal{U}$ and $\chi\,:\, \mathbb{R}^m \to \mathcal{V}$ for the optimal recovery maps associated with $\phi$ and $\varphi$,  we approximate $\mathcal{G}^\dagger$ with $\bar{\mathcal{G}}=\chi \circ \bar{f} \circ \phi$ where $\bar{f}$ is an optimal recovery approximation of  $f^\dagger:=\varphi \circ \mathcal{G}^\dagger \circ \psi\,:\,\mathbb{R}^n \to \mathbb{R}^m$.  
    We show that, even when using vanilla kernels (e.g., linear or Mat\'{e}rn), our approach is competitive in terms of cost-accuracy trade-off and either matches or beats the performance of 
    NN methods on a majority of benchmarks.  Additionally,  our framework offers several advantages inherited from kernel methods: simplicity, interpretability, convergence guarantees, a priori error estimates, and Bayesian uncertainty quantification. As such, it can serve as a natural benchmark for operator learning.
    \newline
\end{abstract}

\section{Introduction}\label{section: introduction}
Operator learning is a well-established field going back at least to the 1970s 
with the articles \cite{almroth1978automatic, noor1980reduced} who introduced the 
reduced basis method as a way speeding up expensive model evaluations. In the most 
broad sense operator learning arises in the solution of stochastic PDEs \cite{ghanem2003stochastic}, 
emulation of computer codes \cite{kennedy2001bayesian}, reduced order modeling (ROM) \cite{lucia2004reduced}, and numerical homogenization \cite{owhadi_scovel_2019}. In recent years, 
and with the rise of machine learning, operator learning has become the focus of extensive 
research with the development of neural net (NN) methods such as Deep Operator Nets \cite{Lu2021}
and Fourier Neural Nets \cite{fno} among many others. While these NN methods are 
often benchmarked against each other \cite{DeepONet-FNO}, they are rarely compared with 
the aforementioned classical approaches. 
Furthermore, the theoretical analysis of NN methods 
is often limited to density/universal approximation results; 
showing the existence of a network of a requisite size achieving a certain error rate, 
without guarantees whether this network is computable in practice (see for example \cite{deng2022approximation, kovachki2021universal}).

In order to alleviate the aforementioned shortcomings we present 
a mathematical framework for approximation of 
mappings between Banach spaces using the theory of operator valued
reproducing Kernel Hilbert spaces (RKHS) and Gaussian Processes (GPs). Our 
abstract framework is: (1) mathematically simple and interpretable, (2) convenient to implement,
(3)  encompasses some of the classical approaches such as linear methods; and (4) 
comes with a priori error analysis and convergence theory. 
We further present  extensive benchmarking of our kernel method  with 
the DeepONet and FNO approaches and show that the kernel approach either matches or 
outperforms NN methods in most benchmark examples. 

In the remainder of this section we give a summary of our 
methodology and results: 
We pose the operator learning problem 
in \Cref{subsec:OPL-setup} before presenting a running example in 
\Cref{subsec:running-example} which is used to outline our 
proposed framework and main theoretical results in \Cref{subsec:proposed-solution,subsec:convergence-theory} as well as brief numerical results in \Cref{subsec:numerical-demonstration}. Our main contributions are summarized in 
\Cref{subsec:summary-of-contributions} followed by a literature review in 
\Cref{subsec:literature-review}.

\subsection{The operator learning problem}\label{subsec:OPL-setup}
% We present a general Gaussian Process (GP)/kernel framework for learning/approximating nonlinear operators between separable Banach spaces. To describe this
Let $\U$ and $\V$ be two (possibly infinite-dimensional) separable Banach spaces and suppose that
\begin{equation}
\Gc^\dagger :\mathcal{U} \to \mathcal{V}\,
\end{equation}
is an arbitrary (possibly nonlinear) operator. Then, broadly speaking, the goal of operator learning is to approximate $\Gc^\dagger$
from a finite number $N$ of input/output data on $\Gc^\dagger$.
For our framework,  we consider the setting where the input/output data are only partially observed through a finite collection of linear measurements which we formalize as follows:

\begin{Problem}\label{sec: problem}
Let $\{u_i, v_i\}_{i = 1}^N$ be $N$ elements of $\U\times \V$ such that
\begin{equation}
        \Gc^\dagger(u_i) = v_i, \quad \text{for } i =1, \dots, N\,.
\end{equation}
 Let $\phi: \mathcal{U} \to \R^{n}$ and $\varphi: \mathcal{V} \to \R^{m}$
be bounded linear operators. 
 Given the data $\{\phi(u_i), \varphi(v_i)\}_{i=1}^N$ approximate $\Gc^\dagger$.
\end{Problem}

\subsection{Running example}\label{subsec:running-example}
To give context to the above problem and our solution method we briefly outline 
a running example to which the reader can refer to throughout the rest of this 
section. Consider the following elliptic PDE, which is of broad interest in
geosciences and material science:
\begin{equation}\label{eq:Darcy-1D}
  \left\{
  \begin{aligned}
    -\text{div} \: e^{u} \: \nabla v & = w,  && \text{in} \quad \Omega, \\
    v &=0, && \text{on} \quad \partial \Omega\,. 
\end{aligned}
\right.
\end{equation}
where  $\Omega=(0,1)^2$, $u \in H^3(\Omega)$, $w \in H^1(\Omega)$  and $v \in H^3(\Omega)\cap H^1_0(\Omega)$. For a fixed forcing term $w$, we
wish to approximate the nonlinear operator mapping the diffusion coefficient $u$ to the solution $v$, i.e., $\Gc^\dagger: u \mapsto v$. 
In this case we may take $\U \equiv H^3(\Omega)$  and $\V \equiv H^3(\Omega)\cap H^1_0(\Omega)$.
We further assume that a training data set is available in the form of limited observations of input-out pairs. As a canonical example,
 consider the evaluation bounded and linear operators
 \begin{equation}\label{eq:pointwise-eval-functionals}
   \begin{aligned}
     \phi: u \mapsto \left( u(X_1), u(X_2), \dots, u(X_{n}) \right)^T \quad \text{and} \quad
     \varphi: v \mapsto \left( v(Y_1), v(Y_2), \dots, v(Y_{m}) \right)^T,
 \end{aligned}
 \end{equation}
 where the $\{ X_j \}_{j=1}^{n}$ and $\{ Y_j \}_{j=1}^{m}$ are distinct
 collocation points in the domain $\Omega$ as well as pairs $\{u_i, v_i\}_{i=1}^N$
 that satisfy the PDE \eqref{eq:Darcy-1D}.
 Then our goal is to approximate
 $\Gc^\dagger$ from the training data set   $\{ \phi(u_i), \varphi(v_i) \}_{i=1}^N$
 \footnote{Choosing $\phi, \varphi$ as pointwise evaluation functionals is common to many applications, although our abstract framework readily accommodates other
   choices such as integral operators and basis projections}.
   
 \subsection{The proposed solution}\label{subsec:proposed-solution}
   Our setup
 naturally gives rise to a  commutative diagram depicted in \Cref{fig:commutative-diagram}.
 Here the map $f^\dagger: \R^n \to \R^m$ explicitely defined as 
 \begin{equation}\label{eqdeffdagger}
f^\dagger:=\varphi \circ \Gc^\dagger\circ \psi
\end{equation}
 is a mapping between
 finite-dimensional Euclidean spaces, and is therefore amenable to numerical approximation. However,
 in order to approximate $\Gc^\dagger$ we also need the reconstruction maps
 $\psi: \R^{n} \to \U$ and $\chi: \R^{m} \to \V$.

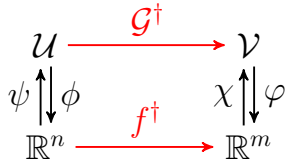
\begin{figure}[htp]
 \centerline{\scalebox{0.9}{
\begin{tikzpicture}[->,>=stealth',shorten >=1pt,auto,node distance=3cm,
                    thick,main node/.style={font=\sffamily\Large\bfseries}]

\node[main node] (1) {$\U$};
\node[main node] (2)  [right of=1] {$\V$};
\node[main node] (3) [below of=1,node distance=1.5cm] {$\R^{n}$};
\node[main node] (4) [below of=2,node distance=1.5cm] {$\R^{m}$};
%\node[main node] (5) [above of=2,node distance=1.5cm] {$x_5$};
%\node[main node] (6) [above of=3,node distance=1.5cm] {$x_6$};

\path[every node/.style={font=\sffamily\Large\bfseries},red]
    (1) edge node [above,red ] {$\Gc^\dagger$} (2);

\path[every node/.style={font=\sffamily\Large\bfseries},red]
    (3) edge node [above,red ] {$f^\dagger$} (4);

%\path[every node/.style={font=\sffamily\Large\bfseries}]
%    (3) edge  node  [ right ] {$\psi$} (1)
%    (1) edge node [ left ] {$\phi$} (3);

\begin{scope}[transform canvas={xshift=.2em}]
\tikzstyle{every to}=[draw]
\draw (1) to[style={font=\sffamily\Large\bfseries}] node[right] {$\phi$} (3);
\end{scope}

\begin{scope}[transform canvas={xshift=-.2em}]
\tikzstyle{every to}=[draw]
\draw (3) to[style={font=\sffamily\Large\bfseries}] node[left] {$\psi$} (1);
\end{scope}

\begin{scope}[transform canvas={xshift=.2em}]
\tikzstyle{every to}=[draw]
\draw (2) to[style={font=\sffamily\Large\bfseries}] node[right] {$\varphi$} (4);
\end{scope}

\begin{scope}[transform canvas={xshift=-.2em}]
\tikzstyle{every to}=[draw]
\draw (4) to[style={font=\sffamily\Large\bfseries}] node[left] {$\chi$} (2);
\end{scope}

\end{tikzpicture}}}

\caption{Commutative diagram of our operator learning setup.}
\label{fig:commutative-diagram}
\end{figure}

Our proposed solution is to endow $\U$ and $\V$ with an RKHS structure and use kernel/GP regression to identify the maps $\psi$ and $\chi$.
As a prototypical example we consider the situation where $\U$ is an RKHS of functions  $u\,:\,\Omega \to \R$ defined by a kernel $Q\,:\, \Omega \times \Omega \rightarrow \R$ and $\V$ is an RKHS of functions $u\,:\,D \to \R$ defined by a kernel $K\,:\, D \times D \rightarrow \R$.
For our
 running example, we have $D=\Omega$, and we can take $Q$ and $K$ to be Mat\'{e}rn like kernels, e.g., the Green's function of elliptic PDEs (possibly on $\Omega$ or restricted to $\Omega$) with appropriate regularity.  One can also choose $Q, K$ to be smoother kernels such that their RKHSs are embedded in $\U$ and $\V$.
 
We then define $\psi$ and $\chi$ as the following optimal recovery maps \footnote{
It is possible to define the optimal recovery maps $\psi, \chi$ 
in the setting where $\phi$ and $\psi$ are nonlinear, following the 
general framework of \cite{gp_pdes, owhadi2022computational, owhadi2023ideas}. However, in this setting 
the closed form formulae  \eqref{eq:psi-and-chi-representer-formulae}
no longer hold.
}
:
 \begin{equation}\label{eq:psi-chi-optimal-recovery}
   \begin{aligned}
     \psi(U) & := \argmin_{w \in \U} \| w\|_Q \quad \text{s.t.} \quad \phi(w) =U,\\
     \chi(V) & := \argmin_{w \in \V} \| w\|_K \quad \text{s.t.} \quad \varphi(w) = V,\\
 \end{aligned}
 \end{equation}
 where $\| \cdot \|_Q$ and $\| \cdot \|_K$ are the RKHS norms arising from their pertinent kernels.

 In the case where $\phi$ and $\varphi$ are pointwise evaluation maps ($\phi(u)=(u(X_1),\ldots,u(X_n))$ and $\varphi(v)=(v(Y_1),\ldots,v(Y_m))$ where the $X_i$ and $Y_j$ are pairwise distinct collocation points in $\Omega$ and $D$), our optimal recovery maps can be expressed in closed form using standard representer theorems for kernel interpolation \cite{representer-theorem}:
 \begin{equation}\label{eq:psi-and-chi-representer-formulae}
   \psi(U)(x) = Q(x, X) Q(X,X)^{-1} U, \qquad \chi(V)(y) = K(y, Y) K(Y,Y)^{-1} V,
 \end{equation}
 where $Q(X,X)$ and $K(Y,Y)$ are kernel matrices with entries $Q(X,X)_{ij} = Q(X_i, X_j)$ and
 $K(Y,Y)_{ij} = K(Y_i, Y_j)$ respectively, while  $Q(x, X)$ and $K(y, Y)$ denote
 row-vector fields with entries $Q(x, X)_i = Q(x, X_i)$ and $K(y, Y)_i = K(y, Y_i)$.

 We further propose to approximate $f^\dagger$ by optimal recovery  in a
 vector-valued RKHS. Let $\Gamma\,:\, \R^n \times \R^n \to \Lc(\R^m)$ be a matrix valued kernel \cite{alvarez2012kernels}; here $\Lc(\R^m)$ is the space of 
 % bounded linear operators from 
 % $\R^m$ to $\R^m$, i.e., the space of 
 $m\times m$ matrices)
 with RKHS $\mathcal{H}_\Gamma$ equipped with the norm $\| \cdot \|_\Gamma$ 
 \footnote{See \Cref{secOpvalker} for a review of operator-valued kernels or the reference \cite{kadri2016operator}.}
 and proceed to approximate $f^\dagger$ by the map $\bar{f}$ defined as
 \begin{equation*}
   \bar{f} := \argmin_{f \in \mathcal{H}_\Gamma} \| f \|_\Gamma \quad \text{s.t.} \quad f(\phi(u_i)) = \varphi(v_i) \quad \text{for} \quad  i =1, \dots, N.
 \end{equation*}
 A simple and practical choice for $\Gamma$
 is the diagonal kernel 
 \begin{equation}\label{eq:diagonal-kernel}
 \Gamma( U, U') = S( U, U') I    
 \end{equation}
 where $S: \R^n \times \R^n \to \R$ is an arbitrary scalar-valued kernel, such as RBF, Laplace, or
 Mat{\'e}rn, and $I$ is the  $m\times m$ identity matrix. More complicated choices, such as sums of kernels or replacing the identity matrix for a fixed positive definite matrix, implying correlations 
 between various input or output correlations, are also possible. However, these
 may lead to greater computational cost and we observe empirically that the simple choice of the identity matrix already provides good performance.
 Then we can approximate the components of $\bar{f}$ via the  independent optimal recovery problems
 \begin{equation}\label{eqakdbdhebd}
   \bar{f}_j := \argmin_{g \in \mathcal{H}_S} \| g \|_S  \quad \text{s.t.} \quad
   g(\phi(u_i)) = \varphi_j(v_i), \quad \text{for} \quad i=1, \dots, N
 \end{equation}
 for $j = 1, \dots, m$. Here we wrote $\varphi_j(v_i)$ for the entry $j$ of the vector $\varphi(v_i)$ and, 
 as our notation suggests, $\mathcal{H}_S$ is the
 RKHS of $S$ equipped with the norm $\| \cdot \|_S$. Since \eqref{eqakdbdhebd} is a standard optimal recovery problem, each $\bar{f}_j$ can be identified by the usual representer formula:
 \begin{equation}\label{eq:f-representer-formula}
   \bar{f}_j(U) = S( U, \Ub) S( \Ub, \Ub)^{-1} \Vb_{\cdot,j},
 \end{equation}
 where $\Ub := ( \phi(u_1), \dots, \phi(u_N) )$ and $\Vb_{\cdot,j}:= ( \varphi_j(v_1), \dots, \varphi_j(v_N))^T$
 and  $S(U, \Ub)$ is a block-vector and  $S(\Ub, \Ub)$ is a block-matrix
 defined in an analogous manner to those in \eqref{eq:psi-and-chi-representer-formulae}.
 By combining equations \eqref{eq:psi-and-chi-representer-formulae} and \eqref{eq:f-representer-formula}
 we obtain the operator 
\begin{equation}\label{eq:bar-G-def}
    \bar{\Gc}:= \chi \circ \bar{f} \circ \phi
\end{equation}
 as an approximation
 to $\Gc^\dagger$.
 We provide further details and generalize the proposed framework in Section~\ref{sec: proposed solution}
 % The proposed framework generalizes 
 to the setting where $\phi$ and $\varphi$ are obtained from arbitrary linear measurements 
 (e.g., integral operators as in tomography)
 % (e.g., via integrals of $u$ along subsets of $\Omega$ as in tomography) 
 and $\U$ and $\V$ may not be spaces of continuous functions. 
 % (see Sec.~\ref{sec: proposed solution}).

\subsection{Convergence guarantee}\label{subsec:convergence-theory}

Under suitable regularity assumptions on $\Gc^\dagger$, our method comes with worst-case convergence guarantees as the number of data points $N$ , i.e., input-output pairs and 
% goes to infinity and 
the number of collocations points $n$ and $m$ go to infinity. We present here a condensed version of this result
and defer the proof to \Cref{sec: convergence}. Below we write $B_R(\mathcal{H})$ 
for the  ball of radius $R > 0$ in a normed space $\Hc$.

\begin{Theorem}[Condensed version of Thm.~\ref{thm:main2}]
% For  $R>0$, write .
\label{thm:main-intro}
Suppose  it holds that:
\begin{enumerate}[label=(\ref{thm:main-intro}.\arabic*)]
\item { (\it Regularity of the domains $\Omega$ and $D$)} $\Omega$ and $D$ are compact sets of finite dimensions $d_\Omega$ and $d_D$ and with Lipschitz boundary. \label{assump:I1}
\item { \it Regularity of the kernels $Q$ and $K$.}
Assume that $\mathcal{H}_Q \subset H^s(\Omega)$  and $\mathcal{H}_K \subset H^t(D)$ for some $s > d_\Omega/2$ and some $t > d_D/2$ with inclusions indicating continuous embeddings.\label{assump:I2}
  \item {(Space filling property of collocation points)}The fill distance  between the collocation points $\{ X_i \}_{i=1}^n \subset \Omega$ and the $\{ Y_j\}_{j=1}^m \subset D$ goes to zero as $n\rightarrow \infty$ and $m\rightarrow \infty$. \label{assump:I3}
 \item { (\it Regularity of the operator $\Gc^\dagger$)}
The operator $\Gc^\dagger$ is continuous from $H^{s'}(\Omega)$ to $\Hc_K$
 for some $s'\in (0,s)$ as well as from $\U$ to $\V$ and all its  Fr\'{e}chet derivatives
are bounded on $B_R(\mathcal{H}_Q)$ for any $R > 0$. \label{assump:I4}
%  For some $s'\in (0,s)$,
%   $\Gc^\dagger$ is continuous from $H^{s'}(\Omega)$ to $\Hc_K$. Assume that  $\Gc^\dagger$ is continuous from $\U$ to $\V$ and that all its  Fr\'{e}chet derivatives
% are bounded on $B_R(\mathcal{H}_Q)$.
 \item {\it (Regularity of the kernels $S^n$)}
Assume that for any $n\geq 1$ and any compact subset $\Upsilon$ of $\R^n$,
     the RKHS  of $S^n$ restricted to  $\Upsilon$  is contained in $H^{r}(\Upsilon)$ for some $r > n/2$ and contains $H^{r'}(\Upsilon)$ for some $r'>0$ that may depend on $n$. \label{assump:I5}
 \item { \it (Resolution and space-filling property of the data)}  Assume that for $n$ sufficiently large, the data points $(u_i)_{i=1}^N \subset B_R(\mathcal{H}_Q)$ belong to the range of $\psi^n$ and are space filling in the sense that they become dense in  $\phi^n(B_R(\mathcal{H}_Q))$ as $N\rightarrow \infty$.
\label{assump:I6}
\end{enumerate}
Then, for  all $t'\in  (0,t)$,
\begin{equation}
  \lim_{n,m\rightarrow \infty} \lim_{N\rightarrow \infty} \sup_{u \in B_R(\mathcal{H}_Q)}\| \Gc^\dagger(u) - \chi^m \circ \bar{f}^{m,n}_N \circ \phi^n(u) \|_{H^{t'}(D)} \to 0\,,
\end{equation}
where our notation makes the dependence of $\psi, \phi, \chi, S$ and $\bar{f}$ on $n,m$ and $N$ explicit.
\end{Theorem}
We note that Assumptions~\ref{assump:I1}--\ref{assump:I3} are standard, and concern the accuracy 
of the optimal recovery maps $\phi^n$ and $\chi^m$ as $n,m \to \infty$. 
Assumptions~\ref{assump:I4}--\ref{assump:I5} are less standard and amount to regularity assumptions on the 
map $\Gc^\dagger$ while Assumption~\ref{assump:I6} concerns the acquisition and regularity 
of the training data set.

In \Cref{sec: convergence} we also present \Cref{thm:main}
as the quantitative analogue of the above result which characterizes how the
 speed of convergence depends on the regularity of the operator $\Gc^\dagger$ and the choice of $\phi$ and $\varphi$ in the setting
 of pointwise measurement operators. We also comment on how  this analysis could be extended to other linear measurements.

% Assumptions 1)-3) concern the accuracy of reconstruction operators $\psi$ and $\chi$, whereas assumptions 4)-6) concern the approximation of $\mathcal{G}^\dagger$.

\subsection{Numerical Framework}\label{subsec:numerical-demonstration}
% \bh{[Please emphasize that the implementation requires a series of 
% Kernel regression problems all of which can be implemented conveniently 
% using off the shelf software and are amenable to well-known numerical 
% analysis techniques such as sparse or low-rank approximation of kernel matrices.]}

Returning to our running example, we implement the proposed framework for learning the non-linear operator mapping $u$ to $v$ in \eqref{eq:Darcy-1D}. 
We consider 1000 inputs and outputs of $u$ and  $v$. The data is taken from \cite{DeepONet-FNO} and the experimental setup is 
discussed further in \Cref{sec: darcy problem}.
% : we used a discretized grid of resolution $29 \times 29$, with the data generated by the MATLAB Partial Differential Equation Toolbox.
We take $\varphi$ to be of the form \eqref{eq:pointwise-eval-functionals} with $m=841$ while we 
define $\phi$ through a PCA pre-processing step. More precisely, let $\phi_{\text{pointwise}}$
be of the form \eqref{eq:pointwise-eval-functionals} with $n = 841$. Choose $n_{\text{PCA}} = 202$ 
(this value captures $95\%$ of  the empirical variance of our training data)
and define 
 \begin{equation}
     \phi(u) = \Pi_{\text{PCA}} \circ \phi_{\text{pointwise}} (u) \in \R^{202}.
 \end{equation}
In other words, we take our $\phi$ map to be the linear map that computes the first 202 PCA coefficients 
of the input functions $u$ given on a uniform grid; observe that we do not use PCA pre-processing 
on the output data here although we do this for some of our other examples in Section \ref{sec: Numerical results} for better performance.

% yielding \bh{measurement operators of the form \eqref{eq:pointwise-eval-functionals} with}
% % the following measurements:
% %  \begin{equation}
% %    \begin{aligned}
% %      \phi: u \mapsto \left( u(X_1), u(X_2), \dots, u(X_{n}) \right)^T \quad \text{and} \quad
% %      \varphi: v \mapsto \left( v(Y_1), v(Y_2), \dots, v(Y_{m}) \right)^T,
% %  \end{aligned}
% %  \end{equation}
%  $n = m = 841$. \bh{We will pre-process $\phi(u)$ by applying a PCA} transformation with $n_{PCA} = 202$ PCA modes (chosen to capture $95 \%$ of the empirical variance of the training data). We do not however pre-process the target $\varphi(v)$ in this example \bh{(although we do this for some of our 
%  other examples in  Section \ref{sec: Numerical results}).}
% This yields the following linear measurement operator for $u$:

 % \pb{which acts as $\phi$ in diagram \ref{fig:commutative-diagram}
% and we must learn the map $f^\dagger: \R^{202} \to \R^{841}$. 
With $\phi$ and $\varphi$ identified (recall \Cref{fig:commutative-diagram})  
we proceed to implement our kernel method using the simple choice of a diagonal kernel $S(U, U')I$ where $S$ is a rational quadratic (RQ) kernel (see \Cref{app: kernels}). This choice transforms the problem into $841$ independent kernel regression problems, each corresponding to one component of $f^\dagger$ 
(i.e., the $f^\dagger_j$'s). 

We used the PCA and  kernel regression modules of the  \texttt{scikit-learn} Python library \cite{scikit-learn} to implement our algorithm. This implementation automatically selects the best kernel parameters  by maximizing the marginal likelihood function \cite{rasmussen} jointly for all problems. 
Our proposed method can therefore be implemented conveniently using off-the-shelf software. 
\Cref{fig: darcy_combined_figure} illustrates examples of  the inputs and outputs of our operator learning problem.
Despite the simple implementation of our method, 
 we are able to obtain competitive accuracy as shown in \Cref{tab:darcy results} 
 where the relative testing $L^2$ loss of our method is compared to other popular algorithms.
 Moreover, our approach is amenable to well-known numerical analysis techniques, such as sparse or low-rank approximation of kernel matrices, to reduce its complexity. For the present example (and those in \Cref{sec: Numerical results}) we only consider \say{vanilla} kernel methods which compute \eqref{eq:f-representer-formula} by computing the full Cholesky factors of the matrix $S( \Ub, \Ub)$.
% The only pre-processing step we apply is preconditioning the inputs $U$ of the kernel $S$ using the Cholesky factors of $Q$, a procedure described in greater detail in Section \ref{sec: cholesky factors}. However, note that similar performance can be achieved by using the PCA algorithm instead, which can also be implemented using scikit-learn. The difference between PCA and the Choelsky preconditioning is explored in greater detail in Section \ref{sec: choelsky-pca}. 

\begin{figure}[htp]
    \centering
    \begin{subfigure}[t]{\textwidth}
    \centering
    \subfloat[\centering Training input]{{\includegraphics[width=0.195\textwidth]{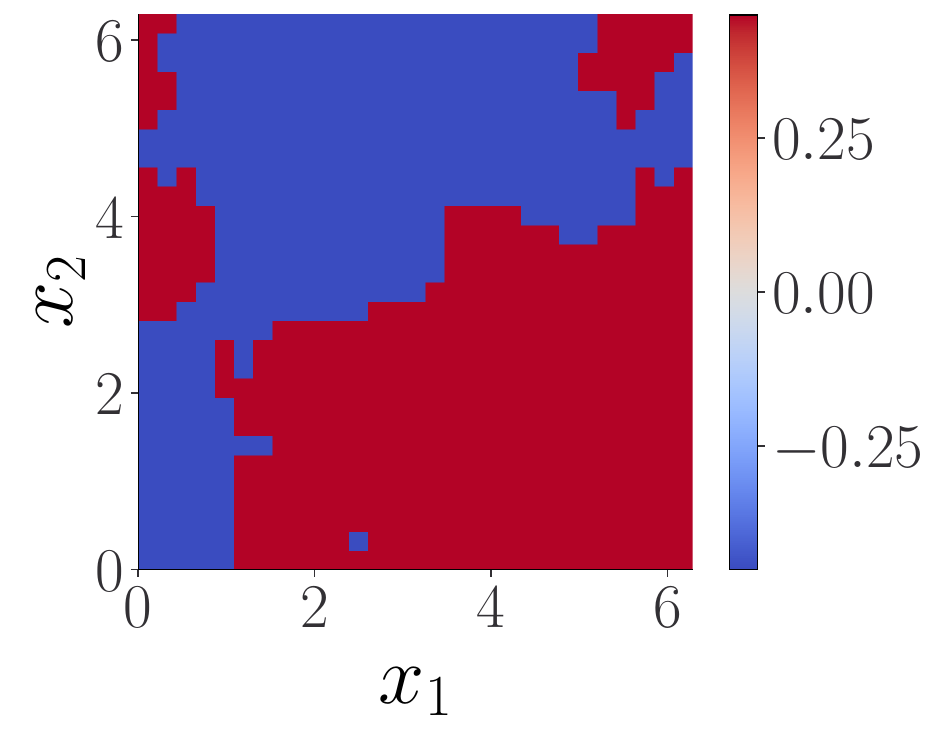} }}%
    \subfloat[\centering Training output]{{\includegraphics[width=0.195\textwidth]{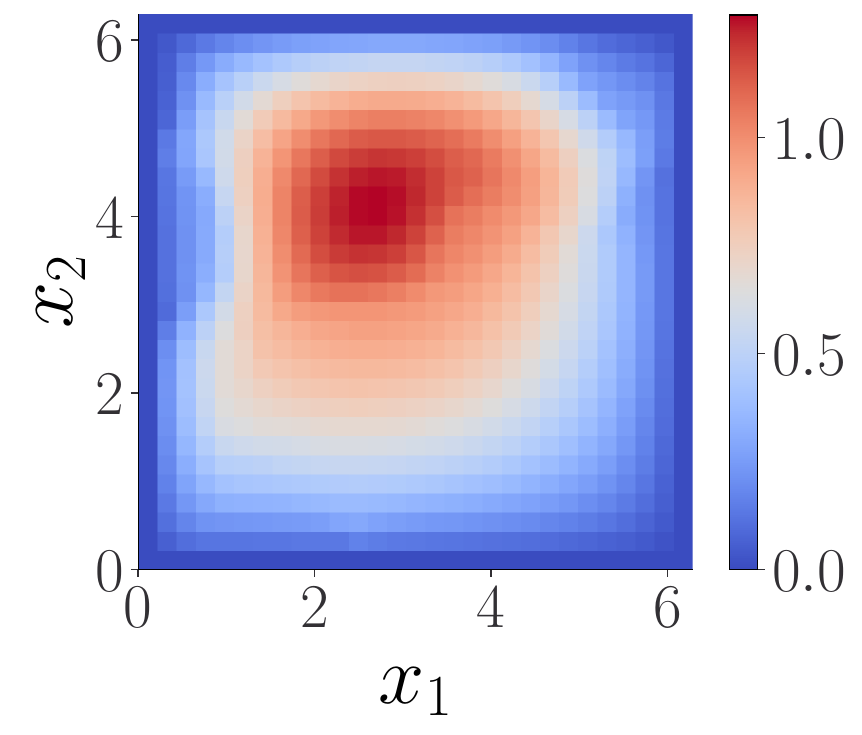} }}%
    % \end{subfigure}
    % \begin{subfigure}[t]{0.59\textwidth}
    % \centering
    \subfloat[\centering True test]{{\includegraphics[width=0.195\textwidth]{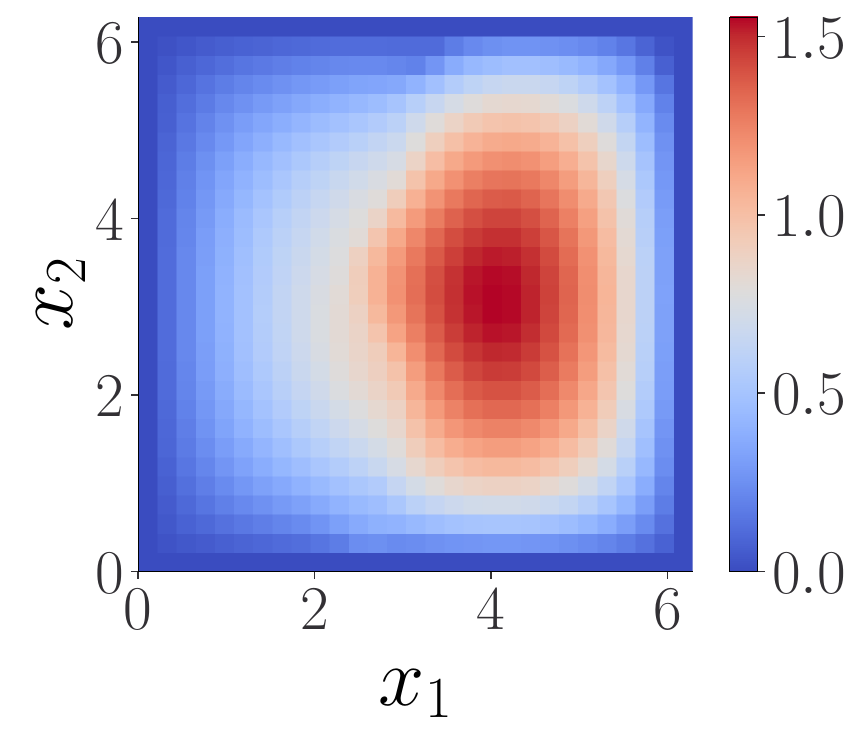}}}%
    \subfloat[\centering Predicted test]{{\includegraphics[width=0.195\textwidth]{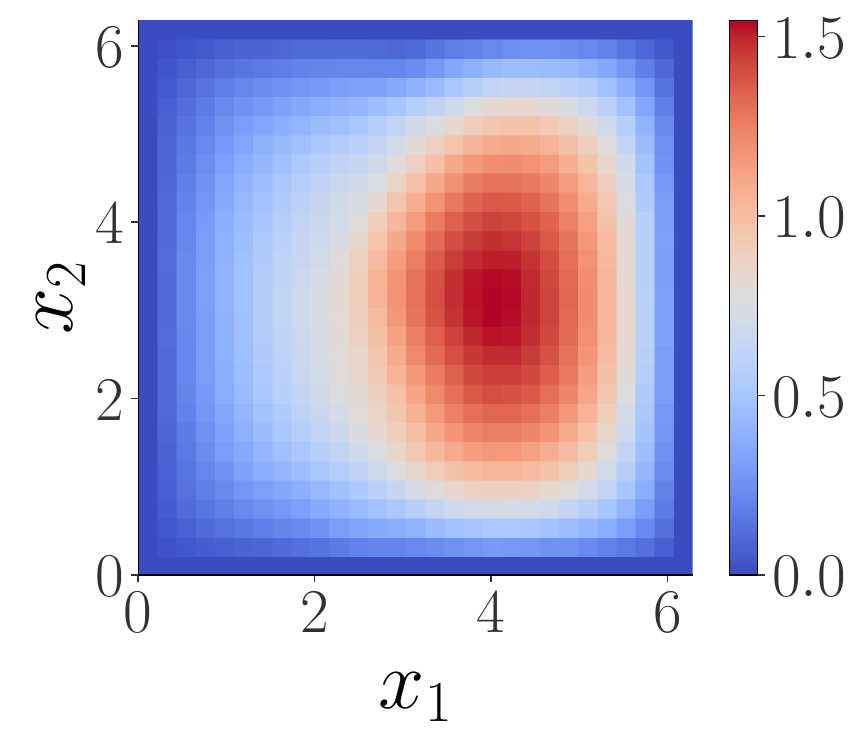}}}%
    \subfloat[\centering Pointwise error]{{\includegraphics[width=0.195\textwidth]{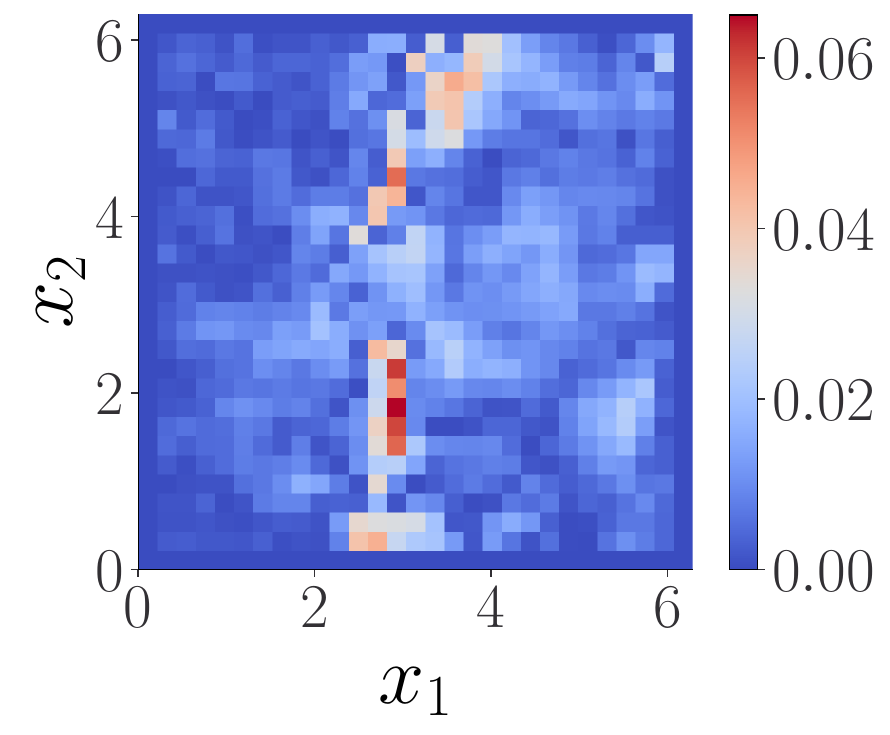}}}%
    \end{subfigure}
     \caption{Example of training data and test prediction and pointwise errors for the Darcy flow problem \eqref{eq:Darcy-1D}.}
    \label{fig: darcy_combined_figure}
\end{figure}

\begin{table}[htp]
    \centering
    \begin{tabular}{|l|c|}
    \hline
        Method & Accuracy \\
        \hline
        DeepONet & 2.91  \% \\
        FNO &  2.41 \%  \\
        POD-DeepONet & 2.32 \%\\
        \hdashline
        Linear  &  6.74 \%\\
        Rational quadratic & 2.87\%\\
        \hline
    \end{tabular}
    \caption{The $L^2$ relative test error of the Darcy flow problem in our running example. The kernel approach 
    is compared with variations of DeepONet and FNO. Results of our kernel method
    are presented below the dashed line with the pertinent choice of the kernel $S$.}
    \label{tab:darcy results}
\end{table}

\subsection{Summary of contributions}\label{subsec:summary-of-contributions}

The main results of the article concern the properties, performance, and error analysis of the map $\bar{\Gc}$ defined in \eqref{eq:bar-G-def}. Our contributions can be summarized under four
categories:

\begin{enumerate}
\item {\bf An abstract kernel framework for operator learning:} In \Cref{sec: proposed solution}, we propose a framework for operator learning using kernel methods with several desirable properties. A family of methods of increasing complexity is proposed that includes linear models and diagonal kernels as well as non-diagonal kernels which capture output correlations. These properties make our approach 
 ideal for benchmarking purposes. Furthermore, the methodology is: (i) applicable to any choice of the
 linear functionals $\varphi$ and $\phi$; (ii) minimax optimal with respect to an implicitly defined operator-valued kernel; and (iii) is mesh-invariant. We emphasize in \cref{remark: mesh invariance} that our optimal recovery maps can be applied to \emph{any} operator learning after training to obtain a mesh-invariant pipeline.
%, and hence not tied to a specific discretization
% We stress that our general framework is more general than the example presented above in two key aspects:
% \begin{enumerate}
%     \item 
%     \item Its implementation is not limited to diagonal kernels.\pb{At some point in the theory we make the diagonal kernel assumption, is it just for simplicity?}
% \end{enumerate}

\item {\bf Error analysis and convergence rates for $\bar{\Gc}$:} In \Cref{sec: convergence}, we develop rigorous worst-case a priori error bounds and convergence guarantees for our method: \Cref{thm:main} provides
quantitative error bounds while \Cref{thm:main2} (the detailed version of \Cref{thm:main}) shows the convergence of $\bar{\Gc} \to \Gc$ 
under appropriate conditions.

\item \textbf{A simple to use vanilla kernel method:} While our abstract kernel method is quite general, our numerical implementation in \Cref{sec: Numerical results} focuses on a simple, easy-to-implement version using  diagonal kernels of the form \eqref{eq:diagonal-kernel}. Off-the-shelf software, 
such as the kernel regression modules of \texttt{scikit-learn},
can be employed for this task. We empirically observe low training times and robust choice of hyperparameters. 
These properties further suggest that kernel methods are a good  baseline for benchmarking of 
more complex methods.

\item {\bf Competitive performance.} 
In \Cref{sec: Numerical results} we present a series of numerical experiments on benchmark PDE problems 
from the literature and observe that our simple implementation of the kernel approach is competitive in 
terms of complexity-accuracy tradeoffs in comparison to several NN-based methods. Since kernel methods can be interpreted as an infinite-width, one-layer NN, the results raise the question of how much of a role the depth of  a 
deep NN plays in the performance of algorithms for the purposes of operator learning.
\end{enumerate}

\subsection{Review of relevant literature}\label{sec:revi-relev-liter}

% Various approaches can be taken towards the approximation of $\Gc^\dagger$, for example:
%  polynomial and Galerkin approximation techniques \cite{devore, webster, etc}
% expand $u$ and $v$ in appropriate bases, truncate the expansions, and proceed to approximate $\Gc^\dagger$ with another, finite dimensional map
% map $\Gc$; 
% Reduced Basis Models (RBMs) \cite{review papers, charbel}
% and deep learning (DL) methods \cite{FNO, DeepONet} assume that
% a training data set of input and output pairs $\{ u_i, v_i\}_{i=1}^N$ is available, and utilize
% this training set in various ways in order to construct the approximate map $\Gc$ via
% regression or interpolation; a detailed literature review will be given in 
%  Section~\ref{sec:revi-relev-liter}.

In the most broad sense, operator learning is the problem of 
approximating a mapping between two infinite-dimensional function spaces 
\cite{pca, de2022cost}. In recent years, this problem 
has become an area of intense research in the scientific machine 
learning community with a particular focus on parametric or stochastic PDEs. 
However, the approximation of such parameter to solution maps 
has been an area of intense research in the computational mathematics and engineering 
communities, going back at least to the reduced basis method introduced in the
1970s \cite{almroth1978automatic, noor1980reduced} as a way of speeding up 
the solution of families of parametric PDEs in applications that require many PDE solves such 
as design \cite{economon2016su2, martins2013multidisciplinary, bendsoe2003topology, boncoraglio2021active}, uncertainty quantification (UQ) \cite{sudret2017surrogate, martin2012stochastic, huang2022iterated}, and multi-scale modeling \cite{weinan2011principles, fish1997computational, feyel2000fe2, kovachki2022multiscale}. 
In what follows we give a brief summary of the various areas and 
methodologies that overlap with operator learning;
we cannot provide an exhaustive list of references due to space, but refer the reader to key contributions and surveys where further references can be found.

\paragraph{Deep learning techniques}

The use of NNs for operator learning goes back at least to the 90s and 
the seminal works of Chen and Chen \cite{chen1995universal, chen1995approximation}
 who proved a universal approximation theorem for NN approximations to operators. 
The use and design of NNs for operator learning has 
become popular in the last five years as a consequence of growing interest in
NNs for scientific computing starting with  the article \cite{zhu2018bayesian} which used autoencoders 
to build surrogates for UQ of subsurface flow models. Since then many different 
approaches have been proposed some of which use specific architectures or target particular 
families of PDEs \cite{Hesthaven-PCA-Net, khoo2019switchnet, fno, khoo2021solving, Lu2021, gin2021deepgreen, boulle2022learning, kropfl2022operator, kissas2022learning}.
The most relevant of among these methods to our proposed framework 
are the DeepONet family \cite{Lu2021, wang2021learning, DeepONet-FNO, wang2022improved}, 
FNO \cite{fno}, and PCA-Net 
\cite{Hesthaven-PCA-Net, pca}
where the main novelty appears to be 
the use of novel, flexible, and expressive NN architectures that 
allow the algorithm to learn and adapt 
the bases that are selected for the input and outputs of the 
solution map as well as possible nonlinear dependencies between the basis coefficients.
Although not part of our comparisons, we note that \cite{fan2019bcr,fan2019multiscale,fan2019multiscaleH} obtained competitive accuracy by using deep neural networks with architectures inspired by conventional fast solvers.

\paragraph{Classical numerical approximation methods}\label{subsec:literature-review}

Operator learning has been the subject of intense research in the computational mathematics literature in the context of stochastic Galerkin methods \cite{ghanem2003stochastic, xiu2009efficient}, polynomial chaos \cite{xiu2010numerical, xiu2002wiener}, reduced basis methods \cite{noor1980reduced, maday2002priori} and numerical homogenization \cite{owhadi2007metric, owhadi2015bayesian, owhadi_scovel_2019, altmann2021numerical}.
In the setting of stochastic and parametric PDEs, the
the goal is often to approximate the solution of a PDE
as a function of a random or uncertain parameter.
The well-established approach to such problems is
to pick or construct appropriate bases for the input parameter
and the solution of the PDE and then
construct a parametric, high-dimensional map, that
transforms the input basis coefficients to the output
coefficients. Well-established methods such as
polynomial chaos, stochastic finite element methods, reduced basis methods   \cite{ghanem2003stochastic, xiu2010numerical, cohen2015approximation,
hesthaven2016certified, lucia2004reduced}
fall within this category. A vast amount of literature in applied mathematics exists on this subject, and the theoretical
analysis of these methods is extensive; see for example \cite{beck2012optimal, Chkifa2012, chkifa2014high, nobile2008sparse, nobile2008anisotropic, gunzburger2014stochastic} and references therein.

\paragraph{Operator compression}
For solving PDEs, the objectives of operator learning are also similar to those of operator compression \cite{feischl2020sparse, kropfl2022operator} as
formulated in numerical homogenization \cite{owhadi_scovel_2019, altmann2021numerical} and reduced order modeling (ROM) \cite{amsallem2008interpolation, lucia2004reduced}, i.e., 
the approximation of the solution operator from pairs of solutions and source/forcing terms.
While both ROM and numerical homogenization seek operator compression through the identification of reduced basis functions that are as accurate as possible (this translates into low-rank approximations with SVD and its variants \cite{boulle2022learning}), numerical homogenization also requires those functions to be as localized as possible  \cite{maalqvist2014localization} and in turn leverages both low rank and sparse approximations. These localized reduced basis functions are known as Wannier functions in the physics literature \cite{marzari2012maximally}, and can be interpreted as linear combinations of eigenfunctions that are localized in both frequency space and 
the physical domain, akin to wavelets.
The hierarchical generalization of numerical homogenization \cite{owhadi2017multigrid} (gamblets) has led to the current state-of-the-art for operator compression of linear elliptic 
\cite{schafer2021KL, schafer2021compression} and parabolic/hyperbolic PDEs \cite{owhadi2017gamblets}. In particular, for arbitrary (and possibly unknown) elliptic PDEs \cite{schafer2021sparse} shows that the solution operator (i.e., the Green's function) can be approximated in near-linear complexity to accuracy $\epsilon$
from only $\mathcal{O}(\log^{d+1}(\frac{1}{\epsilon}))$ solutions of the PDE.

\paragraph{GP emulators}
In the case where the range of the operator of interest is finite
dimensional, then operator learning coincides with surrogate modeling techniques
that were developed in the UQ literature, such as GP surrogate modeling/emulation \cite{kennedy2001bayesian, bastos2009diagnostics}. When the kernels of the underlying GPs are also learned from data \cite{owhadi2019kernelkf, chen2021consistency}, GP surrogate modeling has been shown to offer a simple, low-cost, and accurate solution to learning dynamical systems
  \cite{hamzi2021learning}, geophysical forecasting
\cite{hamzi2021simple}, and radiative transfer emulation \cite{susiluoto2021radiative}, and the inference of the structure of convective storms from passive microwave observations
 \cite{prasanth2021kernel}. Indeed, our proposed kernel framework for operator learning can be interpreted as an
 extension of these well-established GP surrogates to the setting where the range of the operator is a
 function space.

\subsection{Outline of the article}\label{subsec:outline}
The remainder of the article is organized as follows: We present 
our operator learning framework in \Cref{sec: proposed solution}
for the generalized setting where $\phi, \varphi$ can be 
any collection of bounded and linear operators along with an interpretation  of our method from the GP perspective. Our convergence 
analysis and quantitative error bounds are presented in \Cref{sec: convergence} where we present the full version of \Cref{thm:main2}. 
 Our numerical 
experiments, implementation details, and 
benchmarks against FNO and DeepONet are collected in 
\Cref{sec: Numerical results}. We discuss future directions 
and open problems in \Cref{sec: conclusions}. The appendix collects
a review of operator valued kernels and GPs along with other auxiliary details.

%\todo{Double check and update before submission}

\section{The RKHS/GP framework for operator learning}\label{sec: proposed solution}
We now present our general kernel framework for operator learning, i.e., the proposed solution to Problem \ref{sec: problem}. We 
emphasize that here we do not require the spaces $\U$ and $\V$ to be spaces of continuous functions and in particular, 
we do not require the maps $\phi$ and $\varphi$ to be obtained from pointwise measurements.
To describe this, we will introduce the dual spaces of $\U$ and $\V$ to define optimal recovery with respect to kernel operators rather than just kernel functions. 
 
Write $\U^\ast$ and $\V^\ast$ for the duals of $\U$ and $\V$, and write $[ \cdot, \cdot ]$ for the pertinent duality pairings.
Assume that $\U$ is endowed with a quadratic norm $\| \cdot \|_Q$, i.e., there exists a linear bijection $Q: \U^\ast \to \U$ that is symmetric  ($[\phi_a,Q\phi_b]=[ \phi_b,Q\phi_a]$), positive ($[ \phi_a,Q\phi_a]>0$ for $\phi_a \ne 0$), and such
 that
$ \| u\|^2_Q = [ Q^{-1} u, u ], \: \forall u \in \U\, .$

As in \cite[Ch.~11]{owhadi_scovel_2019},
although $\U$, and $\U^\ast$ are also  Hilbert spaces under $\|\cdot\|_Q$ and its dual norm $\|\cdot\|_Q^\ast$ (with inner products $\<u,v\>_Q=[Q^{-1}u,v]$ and $\<\phi_a,\phi_b\>_Q^\ast=[\phi_a,Q\phi_b]$), we will keep using the Banach space terminology to emphasize the fact that our dual pairings will not be based on the
inner product through the Riesz representation theorem,
  but on a  different realization of the dual space, as this setting is more practical.
  
  If $\U$ is a space of continuous functions on a subset $\Omega \subset \R^{d_\Omega}$ then $\U^\ast$ contains delta Dirac functions and, to simplify notations, we also write
$Q(x,y):=[ \updelta_x,  Q \updelta_y ]$  for $x,y\in \R^{d_\Omega}$ to denote the kernel induced by the operator $Q$. Note that in that case, $\U$ is a RKHS with norm $\|\cdot\|_Q$ induced by the kernel $Q$.
Since $\phi$ is bounded and linear, its entries $\phi_i$ (write $\phi:=(\phi_1,\ldots,\phi_{n})$) must be elements of $\U^\ast$. We assume those elements to be linearly independent.
Write $\psi\,:\, \R^{n}\to \U$ for the linear operator defined by
\begin{equation}\label{eqpsidef}
\psi(Y):= (Q \phi)\, Q(\phi,\phi)^{-1} Y \text{ for } Y \in \R^{n},
\end{equation}
where we write $Q(\phi,\phi)$ for the $n\times n$ symmetric positive definite (SPD) matrix with entries $Q(\phi_i,\phi_j):=[\phi_i, Q \phi_j]$ \footnote{For linear measurements involving derivatives the computation 
of these kernel matrices requires the computation of  derivatives of the kernels; see \cite{gp_pdes} for practical examples and considerations.} and $Q \phi$ for $(Q \phi_1,\ldots, Q \phi_{n})\in \U^{n}$.
As described in \cite[Chap.~11]{owhadi_scovel_2019},
for $u\in \U$, given $\phi(u)=Y$, $\psi(Y)$ is the minmax optimal recovery of $u$ when using the relative error in $\|\cdot\|_Q$-norm as a loss.

Similarly, assume that $\V$ is endowed with a quadratic norm $\| \cdot \|_K$, defined by the symmetric positive linear bijection $K: \V^\ast \to \V$.
Write $\varphi:=(\varphi_1,\ldots,\varphi_{m})$ and assume the entries of $\varphi$ to be linearly independent elements of $\V^\ast$.
Using the same notations as in \eqref{eqpsidef}  write $\chi\,:\, \R^{m}\to \V$ for the linear operator defined by
\begin{equation}\label{eqchidef2}
\chi(Z):= (K \varphi)\, K(\varphi,\varphi)^{-1} Z\text{ for } Z \in \R^{m}\,.
\end{equation}
Then, as above, for $v\in \V$, given $\varphi(v)=Z$, $\chi(Z)$ is the minmax optimal recovery of $v$ when using the relative error in $\|\cdot\|_K$-norm as a loss.

Write $\Lc(\R^{m})$ for the space of bounded linear operators mapping $\R^{m}$ to itself , i.e., 
$m \times m$ matrices.
Let $\Gamma\,:\, \R^{n}\times \R^{n}\to \Lc(\R^{m})$ be a matrix-valued kernel \cite{alvarez2012kernels} defining an RKHS $\Hc_\Gamma$ of continuous functions $f\,:\,\R^{n}\to \R^{m}$ equipped with an RKHS norm $\|\cdot\|_\Gamma$.
For $i\in \{1,\ldots,N\}$, write $U_i:=\phi(u_i)$ and $V_i:=\varphi(v_i)$.
Write $\Ub$ and $\Vb$ for the block-vectors with entries $U_i$ and $V_i$.
Write $\Gamma(\Ub,\Ub)$ for the $N\times N$ block-matrix with entries $\Gamma(U_i,U_j)$ and assume $\Gamma(U,U)$ to be invertible (which is satisfied if $\Gamma$ is non-degenerate and $U_i\not=U_j$ for $i\not=j$).
Let $f^\dagger$ be an element of $\Hc_\Gamma$ and write 
$f^\dagger(\Ub)$ for the block vector with entries $f^\dagger(U_i)$.
Then given $f^\dagger(\Ub)=\Vb$ it follows that
\begin{equation}\label{ewqkehdu}
\bar{f} ( U ):= \Gamma( U, \Ub) \Gamma(\Ub,\Ub)^{-1} \Vb\,,
\end{equation}
is the minimax optimal recovery of $f^\dagger$,
where  $\Gamma(\cdot, \Ub)$ is the block-vector with entries $\Gamma(\cdot, U_i)$.

To this end, we propose to approximate the ground truth operator  
% Our proposed solution is then to approximate the operator 
$\Gc^\dagger$ with
\begin{equation}
\bar{\Gc}:=\chi \circ \bar{f}\circ \phi\,,
\end{equation}
also recall Figure~\ref{fig:commutative-diagram}.
Combining \eqref{eqchidef2} and \eqref{ewqkehdu} we further infer that $\bar{\Gc}$ 
admits the following explicit representer formula
    \begin{equation}\label{eq:bar-G-rep-formula}
\bar{\Gc}(u)= (K \varphi)\, K(\varphi,\varphi)^{-1} \Gamma(\phi(u), \Ub) \Gamma(\Ub,\Ub)^{-1} \Vb.
\end{equation}
% Diagram \ref{fig:commutative-diagram} summarizes our solution. 

% \todo{\tiny BH: There's paragraph here that I commented out as it
% interrupts the flow. Please make sure this is ok.}
% {\color{red}
% % Given the optimal recovery properties of $\psi$ and $\chi$, if $\Gc^\dagger$ was continuous with respect to the norms induced by $Q$ and $K$ on $\U$ and $\V$, then given  infinite data $N=\infty$ observed at finite resolution ($n, m<\infty$), our best strategy is to reconstruct
% % \begin{equation}\label{eqdeffdagger}
% % f^\dagger:=\varphi \circ \Gc^\dagger\circ \psi
% % \end{equation}
% %  and then approximate $\Gc^\dagger$ with $\hat{\Gc}:=\chi \circ f^\dagger \circ \phi$.
% % Given finite data $N<\infty$ observed at finite resolution, our proposed solution can then be reduced to approximating $f^\dagger$ with $\bar{f}$.
% }

% We will control the approximation error between $\Gc^\dagger$ and $\bar{\Gc}$ by controlling the approximation error between $\Gc^\dagger$ and $\hat{\Gc}$ and between $f^\dagger$ and $\bar{f}$.

In the remainder of this section we will provide more details and observations regarding our 
approximate operator $\bar{\Gc}$ that is useful later in Section~\ref{sec: convergence} and 
of independent interest.

\subsection{The kernel and RKHS associated with $\bar{\Gc}$}\label{subsec:opvalker}
The explicit formula \eqref{eq:bar-G-rep-formula} suggests that the operator $\bar{\Gc}$ is an 
element of an RKHS defined by an operator-valued kernel, which we now characterize.
For $u_1,u_2\in \U$ and $v\in \V$ write
\begin{equation}\label{eqjhybfvf}
G(u_1,u_2) v:=(K\varphi)\,(K(\varphi,\varphi))^{-1}\Gamma(\phi(u_1),\phi(u_2)) (K(\varphi,\varphi))^{-1} \varphi(v).
\end{equation}
It turns out that $G: \U \times \U \to \Lc(\V)$ is a well-defined operator-valued kernel 
whose RKHS contains operators of the form $\bar{\Gc}$.
\begin{Proposition}\label{prop:opvalued-RKHS-of-barG}
The kernel $G$ in \eqref{eqjhybfvf} is an operator-valued kernel. 
 Write $\Hc_G$ for its RKHS and  $\|\cdot\|_G$ for the associated norm. 
 Then it holds that 
 $\Gc \in \Hc_G$ if and only if $\Gc= \chi \circ f \circ \phi$ for 
$f= \varphi \circ \Gc \circ \psi \in \Hc_\Gamma$ and 
$\|\Gc\|_G=\|f\|_\Gamma\,.$
\end{Proposition}
\begin{proof}
    Since $G$ is Hermitian and positive, we deduce that $G$ is an operator-valued kernel. Indeed for
$\tilde{u}_1,\ldots,\tilde{u}_m\in \U $ and $\tilde{v}_1,\ldots,\tilde{v}_m\in \V$,
using  $\<\tilde{v}_i, K\varphi_s\>_K=\varphi_s(\tilde{v}_i)$ and the fact that $\Gamma$ is a matrix-valued kernel we have
\begin{equation}\label{eqlkjednkjd}
\begin{aligned}
 \<\tilde{v}_i, G(\tilde{u}_i,\tilde{u}_j) \tilde{v}_j\>_K &=   \varphi(\tilde{v}_i)^T (K\varphi)\,(K(\varphi,\varphi))^{-1}\Gamma(\phi(\tilde{u}_i),\phi(\tilde{u}_j)) (K(\varphi,\varphi))^{-1} \varphi(\tilde{v}_j)\\ 
 & = \< G(\tilde{u}_j,\tilde{u}_i) \tilde{v}_i,  \tilde{v}_j\>_K\, ,
\end{aligned}
\end{equation}
where we used  $\<\tilde{v}_i, K\varphi_s\>_K=\varphi_s(\tilde{v}_i)$ and the fact that $\Gamma$ is a matrix-valued kernel.
Furthermore, summing \eqref{eqlkjednkjd}, we deduce that
$\sum_{i,j=1}^m \<\tilde{v}_i, G(\tilde{u}_i,\tilde{u}_j) \tilde{v}_j\>_K \geq 0$.
From \eqref{eqjhybfvf} we infer
\begin{equation}\label{eqjhyvttvbfvf}
\sum_{j=1}^m G(u,\tilde{u}_j) \tilde{v}_j=\chi\circ f\circ \phi(u)
\end{equation}
with the function
\begin{equation}\label{eqlhbhbdehdbds}
f( U )=\sum_{j=1}^m \Gamma( U ,\phi(\tilde{u}_j)) (K(\varphi,\varphi))^{-1} \varphi(\tilde{v_j})\,.
\end{equation}
 Furthermore using the reproducing property of $G$ and \eqref{eqlkjednkjd} we have
$\left\|\sum_{j=1}^m G(u,\tilde{u}_j) \tilde{v}_j \right\|_G^2=\|f\|_\Gamma^2\,.$
Therefore the closure of the space of operators of the form \eqref{eqjhyvttvbfvf} with respect to the RKHS norm induced by $G$ is the space of functions of the form $\chi \circ f \circ \phi$ where $f$ lives in the closure of functions of the form \eqref{eqlhbhbdehdbds} with respect to the RKHS norm induced by $\Gamma$. We deduce that $\Hc_G =\{\chi \circ f \circ \phi\mid f\in \Hc_\Gamma\}$. The uniqueness of $f$ in the representation $\Gc=\chi \circ f \circ \phi$ for $f\in \Hc_G$ follows from  $f=\varphi \circ \Gc \circ \psi $ following the 
identities  $\varphi \circ \chi=I_d$ and $\phi \circ \psi=I_d$.
\end{proof}

Using the above result we can further characterize $\bar{\Gc}$ and $\bar{f}$ via optimal recovery 
problems in $\Hc_G$ and $\Hc_\Gamma$ respectively. In what follows 
we will write $\ub$ for the $N$ vector whose entries are the $u_i$, and $\Gc(\ub)$ for the
$N$ vector whose entries are $\Gc^\dagger(u_i)$.

\begin{Proposition}\label{prop:bar-G-bar-f-optimal-recovery}
The operator $\bar{\Gc}$ is the minimizer of
\begin{equation}\label{eqlkjbdebddbs}
\begin{cases}
\text{Minimize }& \|\Gc\|_G^2\\
\text{Over } &\Gc \in \Hc_G \text{ such that } \varphi\circ \Gc(\ub)=\varphi \circ \Gc^\dagger(\ub)\,.
\end{cases}
\end{equation}
while the map
$\bar{f}$ is the minimizer of
\begin{equation}\label{eqbarfvar}
\begin{cases}
\text{Minimize }& \|f\|_\Gamma^2\\
\text{Over } &f \in \Hc_\Gamma \text{ such that } f\circ \phi(\ub)=\varphi \circ \Gc^\dagger(\ub)\,.
\end{cases}
\end{equation}
\end{Proposition}

\begin{proof}
    By Proposition~\ref{prop:opvalued-RKHS-of-barG} $\bar{G}$ is completely identified by $\bar{f}$ and 
    $\| \bar{G} \|_G = \| \bar{f} \|_\Gamma$. Then solving \eqref{eqlkjbdebddbs} is equivalent to solving 
    \eqref{eqbarfvar}.
    The statement regarding $\bar{f}$ follows directly from representer formulae for 
    optimal recovery with matrix-valued kernels.
\end{proof}

\subsection{Regularizing $\bar{G}$ by operator regression}

As is often the case with optimal recovery/kernel regression the estimator for $\bar{f}$ in 
\eqref{ewqkehdu} is susceptible to numerical error due to ill-conditioning of the kernel 
matrix $\Gamma(\Ub, \Ub)$. To overcome this issue we regularize our estimator by adding a 
small diagonal  perturbation  to this matrix. More precisely,
let $\gamma>0$ and write $I$ for the identity matrix. We then define the regularized map
\begin{equation}\label{ewqkehdub}
\bar{f}_\gamma (U):= \Gamma( U, \Ub) \big(\Gamma(\Ub,\Ub)+\gamma I\big)^{-1} \Vb\,.
\end{equation}
This regularized map gives rise to the regularized  approximate operator
\begin{equation*}
\bar{\Gc}_{\gamma}:=\chi \circ \bar{f}_\gamma\circ \phi\,,
\end{equation*}
which admits the following representer formula
\begin{equation}\label{eq:def-bar-G-gamma}
\bar{\Gc}_{\gamma} (u)= (K \varphi)\, K(\varphi,\varphi)^{-1} \Gamma(\phi(u), \Ub)  \big(\Gamma(\Ub,\Ub)+\gamma I\big)^{-1} \Vb\,.
\end{equation}
We can further characterize this operator as the solution to an operator regression problem.
% Then we have the following theorem.
\begin{Proposition}\label{thmjgytycrc}
$\bar{\Gc}_{\gamma}$ is the solution to
\begin{equation}\label{eqiheuddh1}
\text{Minimize }_{\Gc \in \Hc_G}  \|\Gc\|_G^2+  \gamma^{-1}|\varphi\circ \Gc(\ub)-  \varphi\circ \Gc^\dagger(\ub) |^2.
\end{equation}
\end{Proposition}

\begin{proof}
By Proposition~\ref{prop:opvalued-RKHS-of-barG}, 
$\Gc=\chi\circ f \circ \phi$ solves \eqref{eqiheuddh1} if and if $f$ solves
\begin{equation}\label{eqiheuddh2}
\text{Minimize }_{f \in \Hc_\Gamma}  \|f\|_\Gamma^2+\gamma^{-1} |f(\Ub)-\Vb|^2\,.
\end{equation}
It then follows, by standard representer theorems for  matrix-valued kernel regression
(see Section~\ref{app:operator-valued-ridge-regression})
that  $\bar{f}_\gamma$ is the minimizer of \eqref{eqiheuddh2}.
\end{proof}

\subsection{Interpretation as conditioned operator valued GPs}\label{secGPint}
Our kernel approach to operator learning has a natural GP regression interpretation 
that is compatible with Bayesian inference and UQ pipelines. We present some facts and observations 
in this direction.

Write $\xi \sim \cN(0,G)$ for the centered operator-valued GP with covariance kernel $G$ \footnote{See 
\Cref{subsecuideydiueyd} for a review of operator valued GPs.} and 
 $\zeta \sim \cN(0,\Gamma)$ for a centered vector valued GP with covariance kernel $\Gamma$. Then it 
 is straightforward to show  
 that the law of $\xi$ is equivalent to  that of $\chi \circ \zeta \circ \phi$.
% \paragraph{Regularization in $\R^{m}$}
 Let $Z=(Z_1,\dots,Z_N)$ be a random block-vector,  independent from $\xi$, with i.i.d. entries $Z_j \sim \cN(0,\gamma I_{m})$
 for $j = 1, \dots, N$; here $\gamma \ge 0$ and $I_m$ is the $m \times m$ identity matrix.
 
 % entries ($\gamma\geq 0$ and $I_{\R^{m}}$ is the identity map on $\R^{m}$).
Then
$\xi$ conditioned on $\varphi\circ \xi(\ub)=\varphi(\vb)+Z$ is an operator-valued GP with mean $\bar{\Gc}_{\gamma}$, 
as in \eqref{eq:def-bar-G-gamma},and conditional covariance kernel 
\begin{equation*}
\begin{split}
 G^\perp &(u,u') v=  (K\varphi)\,(K(\varphi,\varphi))^{-1}\Gamma(\phi(u_1),\phi(u_2))\\
& \big(\Gamma(\phi(u),\phi(u'))-\Gamma(\phi(u),\Ub) (\Gamma(\Ub,\Ub)+\gamma I)^{-1} \Gamma(\Ub,\phi(u'))\big) (K(\varphi,\varphi))^{-1} \varphi(v)
\end{split}
\end{equation*}
 Furthermore, the law of
$\xi$ conditioned on $\varphi \circ \xi(\ub)=\varphi(\vb)+Z$ is equivalent to that of $\chi \circ \zeta^\perp \circ \phi$ where $\zeta^\perp\sim \cN(\bar{f}_\gamma, \Gamma^\perp)$ is the GP $\zeta$ conditioned on $\zeta(\Ub)=\Vb+Z'$, whose mean is $\bar{f}_\gamma$ as in \eqref{ewqkehdub} and conditional covariance kernel is
\begin{equation*}
 \Gamma^\perp(U,U')=\Gamma(U,U')-\Gamma(U,\Ub) (\Gamma(\Ub,\Ub)+\gamma I)^{-1} \Gamma(\Ub,U').   
\end{equation*}
We also use the GP approach to
derive an alternative regularization of \eqref{eqiheuddh1} in \Cref{sec:alternative-regularization}.

% \paragraph{Regularization in $\V$}

\subsection{Measurement and mesh invariance}

As argued in \cite{fno}, mesh invariance is a key property for operator learning methods, i.e, 
the learned operator should be generalizable at test time beyond the specific discretization that was 
used during training. In our framework, this translates to being able to predict the output of a test input function $\tilde{u}$ given only a linear measurement $\tilde{\phi}(\tilde{u})$, where $\tilde{\phi}$ was unknown at training time. For example $\tilde{\phi}$ could be of the same form as $\phi$ (say \eqref{eq:pointwise-eval-functionals})
but on a finer or coarser grid. Similarly, we may choose to output with an operator $\tilde{\varphi}$ 
which is a coarse/fine version of $\varphi$. Our proposed framework can easily provide mesh invariance 
using additional optimal recovery and measurement operators
at the input and outputs of the operator $\bar{\Gc}$ as depicted in \Cref{fig:mesh-invariance}. 
In fact, we can not only accommodate modification of the grid but completely different measurement 
operators at testing time. For example, while $\phi, \varphi$ may be of the form \eqref{eq:pointwise-eval-functionals}
we may take $\tilde{\phi}$ and $\tilde{\varphi}$ to be integral operators such as Fourier or Radon transforms.

 Let us describe our approach to mesh invariance in detail. 
 Given bounded and linear operators  
 $\tilde{\phi}: \mathcal{U} \to \R^{\tilde{n}}$
 % \in (\mathcal{U}^*)^{\tilde{n}}$, 
 and $\tilde{\varphi}: \mathcal{V} \to \R^{\tilde{m}}$ 
 % \in (\mathcal{V}^*)^{\tilde{m}}$, such that 
 we can approximate $\tilde{\varphi}(\mathcal{G}^\dagger(\tilde{u}))$ 
 using the map $\bar{f}$ obtained from \eqref{ewqkehdu} defined in terms of our training. 
 To achieve mesh invariance we simply need 
 a consistent approach to interpolate/extend the testing measurment operators to those 
 used for training and we achieve this using the optimal recovery map $\tilde{\psi}$ that is 
 defined from $\tilde{\phi}$ analogously to $\psi$ in \eqref{eqpsidef}. This setup gives rise 
 to a natural approximation of $\Gc^\dagger$ in terms of the function 
 $h^\dagger: \R^{\tilde{n}} \to \mathbb{R}^{\tilde{m}}$ depicted in \Cref{fig:mesh-invariance} 
 which in turn can be approximated with $\bar{h} :=  \tilde{\varphi}\circ
    \chi \circ \bar{f} \circ \phi \circ \tilde{\psi} \equiv \tilde{\varphi} \circ \bar{\Gc} \circ \tilde{\psi}$. 
    This expression 
    further gives rise to another approximation to $\Gc^\dagger$ given by the operator 
    $\tilde{\Gc} = \tilde{\chi} \circ \bar{h} \circ \tilde{\phi}$. 

    \begin{remark}\label{remark: mesh invariance}
        Observe that the definition of $\bar{h}$ (and consequently $\bar{\Gc}$)
        is independent of the fact that $\bar{f}$  is constructed using the kernel approach. Thus, 
        the optimal recovery maps $\chi$ and $\tilde{\psi}$ can be used to retrofit any 
        fixed-mesh operator learning algorithm, to become mesh-invariant and able to use arbitrary linear measurements of the function $\tilde{u}$ at test time.
    \end{remark}

%  given the data $\tilde{\phi}(\tilde{u})$, a function $h^\dagger$ mapping $\mathbb{R}^{\tilde{n}}$ to $\mathbb{R}^{\tilde{m}}$
% is implicitly defined using the optimal recovery function $\tilde{\psi}$, that is analogously defined as $\psi$ in . Then, since our method approximates $f^\dagger$ with the optimal recovery $\bar{f}$, $h^\dagger$ can be naturally approximated  by .

% In addition to these cases, our method can seamlessly incorporate information from arbitrary linear functionals (pointwise derivatives, local averages, ...) not used at training time. 

% \todo[inline]{Can we write a sense in which $\bar{h}$ is optimal? If so, is it worth doing so?}

\begin{figure}
 \centerline{\scalebox{0.9}{
\begin{tikzpicture}[->,>=stealth',shorten >=1pt,auto,node distance=3cm,
                    thick,main node/.style={font=\sffamily\Large\bfseries}]

\node[main node] (1) {$\mathcal{U}$};
\node[main node] (2)  [right of=1] {$\mathcal{V}$};
\node[main node] (3) [below of=1,node distance=1.5cm] {$\mathbb{R}^{n}$};
\node[main node] (4) [below of=2,node distance=1.5cm] {$\mathbb{R}^{m}$};
\node[main node] (5) [below left= 0.5cm and 2.5cm of 3] {$\mathbb{R}^{\tilde{n}}$};
\node[main node] (6) [below right= 0.5cm and 2.5cm of 4]{$\mathbb{R}^{\tilde{m}}$};

\path[every node/.style={font=\sffamily\Large\bfseries},red]
    (1) edge node [above,red ] {$\mathcal{G}^\dagger$} (2);

\path[every node/.style={font=\sffamily\Large\bfseries},red]
    (3) edge node [above,red ] {$f^\dagger$} (4);

\path[every node/.style={font=\sffamily\Large\bfseries},red]
    (5) edge node [above,red ] {$h^\dagger:= \tilde{\varphi}\circ
    \chi \circ f^\dagger \circ \phi \circ \tilde{\psi}$} (6);

\begin{scope}[transform canvas={xshift=.2em}]
\tikzstyle{every to}=[draw]
\draw (1) to[style={font=\sffamily\Large\bfseries}] node[right] {$\phi$} (3);
\end{scope}

\begin{scope}[transform canvas={xshift=-.2em}]
\tikzstyle{every to}=[draw]
\draw (3) to[style={font=\sffamily\Large\bfseries}] node[left] {$\psi$} (1);
\end{scope}

\begin{scope}[transform canvas={xshift=.2em}]
\tikzstyle{every to}=[draw]
\draw (2) to[style={font=\sffamily\Large\bfseries}] node[right] {$\varphi$} (4);
\end{scope}

\begin{scope}[transform canvas={xshift=-.2em}]
\tikzstyle{every to}=[draw]
\draw (4) to[style={font=\sffamily\Large\bfseries}] node[left] {$\chi$} (2);
\end{scope}

\begin{scope}[transform canvas={xshift=-.3em}]
\tikzstyle{every to}=[draw]
\draw (6) to[style={font=\sffamily\Large\bfseries}] node {$\tilde{\chi}$} (2);
\end{scope}

\begin{scope}[transform canvas={xshift=+.3em}]
\tikzstyle{every to}=[draw]
\draw (2) to[style={font=\sffamily\Large\bfseries}] node {$\tilde{\varphi}$} (6);
\end{scope}

\begin{scope}[transform canvas={xshift=-.3em}]
\tikzstyle{every to}=[draw]
\draw (5) to[style={font=\sffamily\Large\bfseries}] node {$\tilde{\psi}$} (1);
\end{scope}

\begin{scope}[transform canvas={xshift=+.3em}]
\tikzstyle{every to}=[draw]
\draw (1) to[style={font=\sffamily\Large\bfseries}] node {$\tilde{\phi}$} (5);
\end{scope}
\end{tikzpicture}}}

\caption{Generalization of \cref{fig:commutative-diagram} to the mesh invariant 
setting where the measurement 
functionals are different at test time.}
\label{fig:mesh-invariance}
\end{figure}
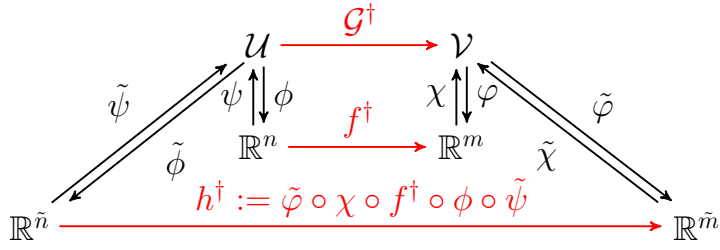

\section{Convergence and error analysis}\label{sec: convergence}

In this section, we present convergence guarantees and rigorous a priori error bounds for our proposed kernel
method for operator learning and give a detailed statement and proof of \Cref{thm:main}.
We assume that $\Hc_Q$ is a space of continuous functions from $\Omega\subset \R^{d_\Omega}$ and that $\Hc_K$ is a space of continuous functions from $D\subset \R^{d_D}$. Abusing notations we write $Q: \Omega \times \Omega \to \mathbb{R}^{d_\Omega}$ and $K: D \times D \to \mathbb{R}^{d_D}$ for the kernels induced by the operators $Q$ and $K$.
Let $X = (X_1, \dots, X_{n}) \subset \Omega$ and $Y = (Y_1, \dots, Y_{m}) \subset D$
be distinct collections of points and define their fill-distances
\begin{equation*}
    h_X := \max_{x' \in \Omega} \min_{x \in X} | x - x'|, \qquad
    h_Y := \max_{y' \in D} \min_{y \in Y} | y - y'|.
\end{equation*}
This section focuses on operators $\phi$ and $\varphi$ that are linear combinations of pointwise measurements in $X$ and $Y$. The presented results can be extended by using analogs of the sampling inequalities for other linear measurements, see \cite[Theorem 4.11, Lemma 14.34] {owhadi_scovel_2019} 
for a general framework that allows one to obtain such inequalities.

Let $L_Q$ and $L_K$ be invertible $n \times n$ and $m \times m$  matrices.
For $u\in \Hc_Q$ write $u(X)$ for the $n$-vector with entries $u(X_i)$ and let $\phi \,:\, \Hc_Q\to \R^{n}$ be the bounded linear map defined by
\begin{equation}\label{eqdiueds}
\phi(u)= L_Q u(X)\,.
\end{equation}
For $v\in \Hc_K$ write $v(Y)$ for the $m$-vector with entries $v(Y_j)$ and let $\varphi \,:\, \Hc_K\to \R^{m}$ be the bounded linear map defined by
\begin{equation}\label{eqdiueds2}
\varphi(v)= L_K v(Y)\,.
\end{equation}
Write $\|\phi\|:=\sup_{u \in \Hc_Q} |\phi(u)|/\|u\|_Q$ and $\|\psi\|:=\sup_{U' \in \R^{n}} \|\psi(U')\|_Q/|U'|$, 
and similarly
 $\|\varphi\|:=\sup_{v \in \Hc_K} |\varphi(v)|/\|v\|_K$ and $\|\chi\|:=\sup_{V' \in \R^{m}} \|\chi(V')\|_K/|V'|$.
We will also assume the following regularity conditions on the domains $\Omega, D$, the kernels $Q, K$, and the operator $\Gc^\dagger$.

\begin{Condition}\label{Condiwhdiwidh}
Assume that the following conditions hold.
% \todo[]{This numbering is slightly odd, is this the intended behavior?}
\begin{enumerate}[label=(\ref{Condiwhdiwidh}.\arabic*)]
    \item  $\Omega$ and $D$
are compact sets with Lipschitz boundary.
\label{cond:domains}
\item There exist indices $s > d_\Omega/2$ and $t > d_D/2$ so that
$\mathcal{H}_Q \subset H^s(\Omega)$  and
$\mathcal{H}_K \subset H^t(D)$, with  inclusions indicating continuous embeddings.
\label{cond:RKHS-embedding}
\item $\Gc^\dagger$ is a (possibly) nonlinear operator from $H^{s'}(\Omega)$ to $\Hc_K$
with $s' < s$
that satisfies,
\begin{equation}\label{eqiuygiugyy}
  \| \Gc^\dagger(u) - \Gc^\dagger(v) \|_{K}
     \le \omega \left( \| u - v \|_{H^{s'}(\Omega)} \right)\,,
\end{equation}
where $\omega: \mathbb{R} \to \mathbb{R}_+$ is the {\it modulus of continuity} of $\Gc^\dagger$.
\label{cond:cal-G}
\end{enumerate}
\end{Condition}

Note that conditions \ref{cond:RKHS-embedding} and \ref{cond:cal-G} imply
\begin{equation*}
 \| \Gc^\dagger \|_{B_R(\mathcal{H}_Q) \to \mathcal{H}_K} :=  \sup_{u \in B_R(\mathcal{H}_Q)} { \| \Gc^\dagger(u) \|_K } < +\infty\,.
\end{equation*}

% where we recall
% $B_R(\mathcal{H}_Q)$ denotes the ball of radius $R> 0$
% centered at the origin in $\mathcal{H}_Q$.

\begin{proposition}\label{thm:error-analysis-ideal}
Suppose that  \Cref{Condiwhdiwidh} holds. Let  $0<t' <  t$.
 Then there exist constants $ h_\Omega, h_D, C_\Omega, C_D >0$ such that if
$h_X < h_\Omega$ and $h_Y < h_D$, then 
\begin{equation*}
\| \mathcal{G}^\dagger(u) - \chi \circ f^\dagger \circ \phi(u) \|_{H^{t'}(D)}
\le C_D\,\omega \left( C_\Omega h_X^{s - s'} R \right) + C_D\, h_Y^{t - t'}
\big(\|\Gc^\dagger(0)\|_K+\omega(C_\Omega R)\big),
\end{equation*}
for any $u \in B_R(\mathcal{H}_Q)$, 
where $f^\dagger$ is defined as in \eqref{eqdeffdagger}.
\end{proposition}
\begin{proof}
By the definition of $f^\dagger$ and
 the triangle inequality we have
\begin{align*}
\| \mathcal{G}^\dagger(u) - \chi \circ \varphi \circ \mathcal{G}^\dagger \circ \psi^\dagger \circ \phi(u) \|_{H^{t'}(\Gamma)}
 \le & \| \mathcal{G}^\dagger(u) - \mathcal{G}^\dagger \circ \psi \circ \phi(u)  \|_{H^{t'}(\Gamma)}  \\
& + \| \mathcal{G}^\dagger \circ \psi \circ \phi (u)  - \chi \circ \varphi \circ  \mathcal{G}^\dagger \circ \psi \circ \phi (u) \|_{H^{t'}(\Gamma)} \\
& =: T_1 + T_2.
\end{align*}
Let us first bound $T_1$: By conditions  \ref{cond:RKHS-embedding} and \ref{cond:cal-G}, we have
\begin{equation*}
    T_1 \le  C_D \| \mathcal{G}^\dagger(u) - \mathcal{G}^\dagger \circ \psi \circ \phi(u)  \|_{K} \leq C_D \omega \left( \| u - \psi \circ \phi (u) \|_{H^{s'}(\Omega)} \right).
\end{equation*}
At the same time, since $(u - \psi \circ \phi (u))(X)=0$,  condition~\ref{cond:domains} and the sampling inequality for interpolation in Sobolev spaces \cite[Thm.~4.1]{arcangeli2007extension}, and
condition \ref{cond:RKHS-embedding} imply that
there exists a constant $h_\Omega> 0$ so that
if $h_X < h_\Omega$ then
\begin{equation}\label{eqlkjndekjdn}
    \| u - \psi \circ \phi (u) \|_{H^{s'}(\Omega)} \le C'_\Omega h_X^{s - s'} \| u - \psi \circ \phi (u) \|_{H^s(\Omega)}
    \le C_\Omega h_X^{s - s'} \| u - \psi \circ \phi (u)\|_Q \,,
\end{equation}
where $C'_\Omega , C_\Omega > 0$ are constants that are independent of $u$. Using
$\| u - \psi \circ \phi (u)\|_Q \leq \|u\|_Q$ \cite[Thm.~12.3]{owhadi_scovel_2019}  we deduce the
 desired bound
\begin{equation}\label{eq:T_1_bound}
    T_1 \le C_D \omega \left( C_\Omega h_\Omega^{s - s'} \| u\|_Q \right)\,.
\end{equation}
Let us now bound $T_2$: Once again, by the continuous embedding of condition~\ref{cond:RKHS-embedding} and the sampling inequality for interpolation in Sobolev spaces, we have that,  there exists $h_D >0$ so that if $h_Y < h_D$, then for any $v \in H^{t}(D)$
it holds that
\begin{equation*}
    \| v - \chi \circ \varphi (v) \|_{H^{t'}(D) } \le C'_D h_Y^{t - t'} \| v - \chi \circ \varphi (v)\|_{H^{t}(D)} \leq C_D h_Y^{t - t'} \| v - \chi \circ \varphi (v)\|_{K} \leq C_D h_Y^{t - t'} \| v\|_{K}\,.
\end{equation*}
Taking $v \equiv \mathcal{G}^\dagger \circ \psi \circ \phi(u)$, we deduce that
\begin{align*}
    T_2 & \le C_D h_Y^{t -t'} \| \mathcal{G}^\dagger  \circ \psi \circ \phi(u) \|_{K}, \\
        & \le C_D  h_Y^{t -t'} \big(\|\Gc^\dagger(0)\|_K+\omega(\| \psi \circ \phi(u) \|_{H^{s'}(\Omega)}) \big)\,.
\end{align*}
Using $\| \psi \circ \phi(u) \|_{H^{s'}(\Omega)}\leq C_\Omega \| \psi \circ \phi(u) \|_{Q}\leq C_\Omega \|u\|_Q$ concludes the proof.
\end{proof}

While \Cref{thm:error-analysis-ideal} gives an error bound for the distance between the maps $\Gc^\dagger$ and  $\varphi \circ f^\dagger \circ \phi$, 
we can never compute this map  when $N<\infty$ and so we have to approximate this map as well. Given the kernel $\Gamma$, our optimal recovery approximant for the map $f^\dagger$ is  $\bar{f}$ as in \eqref{ewqkehdu}, 
which we recall is the minimizer of \eqref{eqbarfvar}.

To proceed, we need to  consider another intermediary problem that defines an approximation
$\hat{f}$ to the map $f^\dagger$:
\begin{equation}\label{eqhatfvar}
\hat{f} :=\begin{cases}
\text{Minimize }& \|f\|_\Gamma^2\\
\text{Over } &f \in \Hc_\Gamma \text{ such that } f\circ \phi(\ub)=f^\dagger\circ \phi(\ub)\,.
\end{cases}
\end{equation}
We emphasize that the difference between the problems \eqref{eqbarfvar} and  \eqref{eqhatfvar} 
is
simply in the training data that is injected in the equality constraints, and this difference is quite subtle:

In  practical applications, observations may be taken from $\Gc^\dagger(u_i)$, which is
different from
$f^\dagger \circ \phi (u_i) \equiv \varphi \circ \Gc^\dagger \circ \psi \circ \phi(u_i)$.
To make our analysis simple, henceforth we assume the following condition on our input data. 
\begin{Condition}\label{cond-input-data-simplified}
The input data points $u_i$ satisfy
\begin{equation*}%\label{eqconddata}
u_i=\psi\circ \phi(u_i)\text{ for }i=1,\ldots,N\,,
\end{equation*}    
\end{Condition}
We observe that this condition implies
$\Gc^\dagger(u_i)=f^\dagger \circ \phi (u_i)$ and  $\bar{f}=\hat{f}$. Removing this assumption requires bounding some norm of the error $f^\dagger-\bar{f}$, and we postpone that analysis to a sequel paper as this step 
can become very technical.

The next step in our convergence analysis is then to control the error between 
the maps $\hat{f}$ and $f^\dagger$ which we
will achieve using similar arguments as in the proof of \Cref{thm:error-analysis-ideal}.
For our analysis,  we take $\Gamma$ to be a diagonal, matrix-valued
kernel, of the form \eqref{eq:diagonal-kernel} which we recall for 
reference
\begin{equation}\label{eq:diagonal-kernel-recall}
\Gamma( U, U' ) = S(U, U' ) I    
\end{equation}
where $I$ is the $m \times m$ identity matrix and $S: \R^n \times \R^n \to \R$ is a real valued kernel.

\begin{proposition}\label{thm:error-analysis-bar-g-and-g}
Suppose that \Cref{cond-input-data-simplified} holds.
    Let $\Upsilon \subset \R^{n}$ be a compact set with Lipschitz boundary and
    consider $U = ( U_1, \dots, U_N) \subset \Upsilon$ with
    fill distance
    \begin{equation*}
        h_\Upsilon := \max_{U' \in \Upsilon}  \min_{1\leq i \leq N} | U_i - U' |.
    \end{equation*}
    Let $\Gamma$ be of the form \eqref{eq:diagonal-kernel-recall},
    with $S$ restricted to the set $\Upsilon$, and suppose
    $\mathcal{H}_S \subset H^r(\Upsilon)$ for $r > n/2$ and that
    $f^\dagger_j \in \mathcal{H}_S$ for $j=1, \dots, m$. Then there exist
    constants $h'_\Upsilon, C_\Upsilon >0$ so that whenever $h_\Upsilon < h'_\Upsilon$ then for 
    any $r' < r$ it holds that
    \begin{equation*}
        \| f^\dagger_j - \hat{f}_j \|_{H^{r'}(\Upsilon)} \le C_\Upsilon h_\Upsilon^{r - r'} \| f_j^\dagger \|_{S}.
    \end{equation*}
\end{proposition}

\begin{proof}
    The proof is a direct consequence of the fact that the components of $\hat{f}$ are given by the optimal recovery
    problems \eqref{eqhatfvar} and the sampling inequality for interpolation in Sobolev spaces \cite[Thm.~4.1]{arcangeli2007extension} following the
    same arguments used in the proof of Theorem~\ref{thm:error-analysis-ideal}.
\end{proof}

We can now combine the above results to obtain the following theorem.

\begin{theorem}\label{thm:main}

    Suppose that Conditions \ref{Condiwhdiwidh} and \ref{cond-input-data-simplified} hold in addition to 
    those of \Cref{thm:error-analysis-bar-g-and-g}  with a set of inputs $(u_i)_{i=1}^N \subset B_R(\mathcal{H}_Q) $, the set
    $\Upsilon = \phi \left( B_R(\mathcal{H}_Q) \right)$, and index $n/2 < r' < r$.
    Then for any $u \in B_R(\mathcal{H}_Q)$, it holds that
\begin{equation}\label{eqliedhieuddd}
\begin{aligned}
\| \Gc^\dagger(u) - \chi \circ \bar{f} \circ \phi(u) \|_{H^{t'}(D)}
 \le & C_D\,\omega \left( C_\Omega h_X^{s - s'} R \right) + C_D\, h_Y^{t - t'}
\big(\|\Gc^\dagger(0)\|_K+\omega(C_\Omega R)\big)\\
& +  \sqrt{m} C_D C_\Upsilon \|\chi\| h_\Upsilon^{(r - r')}
    \max_{1 \le j \le m} \| f^\dagger_j\|_S
\end{aligned}
\end{equation}

\end{theorem}

\begin{proof}
    An application of the triangle inequality yields
\begin{equation*}
\begin{aligned}
    \| \Gc^\dagger(u) - \chi \circ \bar{f} \circ \phi(u) \|_{H^{t'}(D)}
 \le  &  \| \Gc^\dagger(u) - \chi \circ f^\dagger \circ \phi(u) \|_{H^{t'}(D)} \\
& + \| \chi \circ f^\dagger \circ \phi(u) - \chi \circ \hat{f} \circ \phi(u) \|_{H^{t'}(D)} \\
& + \| \chi \circ \hat{f} \circ \phi(u) - \chi \circ \bar{f} \circ \phi(u) \|_{H^{t'}(D)}
=: I_1 + I_2 + I_3.
\end{aligned}
\end{equation*}
We can bound $I_1$ immediately using \Cref{thm:error-analysis-ideal}. Furthermore, by \Cref{cond-input-data-simplified} we have that $I_3 = 0$. So it remains for us to bound $I_2$:
By the continuous embedding of
$\Hc_K$ into $H^{t'}(D)$ we can write
\begin{equation*}
   \begin{aligned}
    I_2  \le C_D \| \chi \circ f^\dagger \circ \phi(u) - \chi \circ \hat{f} \circ \phi(u) \|_{K} 
    &\le C_D\|\chi\| | f^\dagger \circ \phi(u) - \hat{f} \circ \phi(u) |    \\
        & \le C_D\|\chi\| \sqrt{\sum_{j=1}^{m} \| f^\dagger_j   - \hat{f}_j \|_{H^{r'}(\Upsilon)}^2 }\,,
   \end{aligned}
\end{equation*}
where the last line follows from the Sobolev embedding theorem and the assumption that $r' > n/2$.
Then an application of \Cref{thm:error-analysis-bar-g-and-g} yields,
\begin{equation*}
    I_2 \le  \sqrt{m} C_D C_\Upsilon \|\chi\| h_\Upsilon^{(r - r')}
    \max_{1 \le j \le m} \| f^\dagger_j\|_S.
\end{equation*}
{}
% Combining the bounds on $I_1, I_2,$ and $I_3$ completes the proof.
\end{proof}

\subsection{Convergence theorem}
Our next step will be to consider the limits $N, n, m \rightarrow \infty$ 
and show the convergence of $\bar{\Gc}$ to $\Gc^\dagger$. To obtain this 
result we first need to make assumptions on the regularity of the true 
operator $\Gc^\dagger$. 

% In order to present 
% The next step is to control the regularity of $f^\dagger=\varphi\circ \Gc^\dagger \circ \psi$.
For $k\geq 1$ write $D^k \Gc^\dagger$ for the functional derivative of $\Gc^\dagger$ of order $k$.
Recall that for $u\in \Hc_Q$, $D^k \Gc^\dagger(u)$ is a multilinear operator mapping $\otimes_{i=1}^k \Hc_Q$ to $\Hc_K$.
For $w_1,\ldots,w_k\in \Hc_Q$ write $[D^k \Gc^\dagger(u), \otimes_{i=1}^k w_i]$ for the (multilinear) action of $D^k \Gc^\dagger(u)$ on $\otimes_{i=1}^k w_i$ and write $\|D^k \Gc^\dagger(u)\|$ for the smallest constant such that for $w_1,\ldots,w_k\in \Hc_Q$,
\begin{equation}\label{eqliheudhdsh}
\big\|[D^k \Gc^\dagger(u), \otimes_{i=1}^k w_i]\big\|_{\Hc_K}\leq \|D^k \Gc^\dagger(u)\|  \prod_{i=1}^k \|w_i\|_{\Hc_Q}
\end{equation}

Similarly, for $k\geq 1$ write $D^k f^\dagger$ for the derivation tensor of $f^\dagger$ of order $k$ (the gradient for $k=1$ and the Hessian for $k=2$, etc).
Recall that for $U\in \R^{n}$, $D^k f^\dagger(U)$ is a multilinear operator mapping $\otimes_{i=1}^k \R^{n}$ to $\R^{m}$.
For $W_1,\ldots,W_k\in \R^{n}$ write $[D^k f^\dagger(U), \otimes_{i=1}^k W_i]$ for the (multilinear) action of $D^k f^\dagger(U)$ on $\otimes_{i=1}^k W_i$ and write $\|D^k f^\dagger(U)\|$ for the smallest constant such that for $W_1,\ldots,W_k\in \R^{n}$,
\begin{equation}\label{eqleklhdeiudh}
\big|[D^k f^\dagger(U), \otimes_{i=1}^k W_i]\big| \leq \|D^k f^\dagger(U)\|  \prod_{i=1}^k |W_i|
\end{equation}
where $|\cdot|$ is the Euclidean norm.
\begin{Lemma}\label{thmoepdjddei}
It holds true that 
$\|D^k f^\dagger(U)\| \le \|\varphi\| \|\psi\|^k \|D^k \Gc^\dagger\circ \psi(U)\|, \; \forall U \in \R^n.$
\end{Lemma}
\begin{proof}
The  chain rule and the linearity of $\varphi$ and $\psi$ imply that
\begin{equation*}
\begin{aligned}
[D^k f^\dagger(U), \otimes_{i=1}^k W_i]= \varphi [D^k \Gc^\dagger\circ \psi(U), \otimes_{i=1}^k \psi (W_i)]\,.
\end{aligned}
\end{equation*}
We then conclude the proof by writing
\begin{equation*}
\begin{aligned}
\big|[D^k f^\dagger(U), \otimes_{i=1}^k W_i]\big|& \le \|\varphi\| \|D^k \Gc^\dagger \circ \psi(U) \|  \prod_{i=1}^k \|\psi (W_i)\|_{\Hc_Q}\\
& \le \|\varphi\| \|\psi\|^k \|D^k \Gc^\dagger \circ \psi(U) \|  \prod_{i=1}^k |W_i|\,.
\end{aligned}
\end{equation*}
{}
\end{proof}

Let us now consider an infinite and dense sequence of points $X_1, X_2,  X_3,\dots$ of $\Omega$, such that the closure of $\cup_{i=1}^\infty \{X_i\}$ is the closure of $\Omega$. Write $X^n$ for the $n$-vector formed by the first $n$ points, i.e.,
\begin{equation}
X^n:=(X_1,\ldots,X_n)
\end{equation}
and 
let $L_Q^n$ be an arbitrary invertible $n\times n$ matrix. Further let $\phi^n\,:\, \Hc_Q\rightarrow \R^n$ be defined by
\begin{equation}\label{phi-L-Q}
\phi^n(u)=L_Q^n u(X^n)\,.
\end{equation}
Write $\psi^n$ for the corresponding optimal recovery $\psi$-map.
Similarly, we assume that we are given an infinite and dense sequence of points $Y_1, Y_2, Y_3, \dots$ of $D$, such that the closure of $\cup_{i=1}^\infty \{Y_i\}$ is the closure of $D$.
Write $Y^m$ for the $m$-vector formed by the first $m$ points, i.e.,
\begin{equation}
Y^m:=(Y_1,\ldots,Y_m)\,.
\end{equation}
Let $L_K^m$ be an arbitrary invertible $m\times m$ matrix and let $\varphi^m\,:\, \Hc_K\rightarrow \R^m$ be defined by
\begin{equation}\label{varphi-L-K}
\varphi^m(v)=L_K^m u(Y^m)\,.
\end{equation}
Write $\chi^m$ for the corresponding optimal recovery $\chi$-map.
We also assume that we are given a sequence of diagonal matrix-valued kernels $\Gamma^{m,n}\,:\, \R^n \times \R^n \to \Lc(\R^m)$ with scalar-valued kernels $S^n\,:\, \R^n \times \R^n \to \R$ as diagonal entries.
Write $\bar{f}^{m,n}_N$ for the corresponding minimizer of \eqref{eqbarfvar}
(also identified by the formula \eqref{ewqkehdu}) for the above setup.

\begin{theorem}\label{thm:main2}
   Let $m, n$ be the dimensionality of the input and output observations 
   $\phi: \mathcal U \to \mathbb{R}^n$ and
   $\varphi: \mathcal V \to \mathbb{R}^m$. Suppose that the closure of $\lim_{n\uparrow \infty} \cup_{i=1}^n \{X_i\}$ is equal to the closure of $\Omega$ and that the closure of $\lim_{m\uparrow \infty} \cup_{i=1}^m \{Y_i\}$ is equal to the closure of $D$.
    Suppose \Cref{Condiwhdiwidh} is satisfied and that 
    % that the regularity assumptions of Condition \ref{Condiwhdiwidh} are satisfied.
    % Let $R>0$ be arbitrary and 
    % % Using the notations of \eqref{eqliheudhdsh}
    % assume that
    \begin{equation}\label{eqregDkG}
     \sup_{u\in B_R(\mathcal{H}_Q)}\big\|[D^k \Gc^\dagger(u)\|<\infty\text{ for all }k\geq 1\,,
     \end{equation}
     for an arbitrary $R >0$.
      Assume that for any $n\geq 1$ and any compact set $\Upsilon$ of $\R^n$,
     the RKHS  of $S^n$ restricted to  $\Upsilon$ (which we write $\mathcal{H}_{S^n}(\Upsilon)$) is contained in $H^{r}(\Upsilon)$ for some $r > n/2$ and contains
     $H^{r'}(\Upsilon)$ for some $r'>0$ that may depend on $n$.
    Let $(u_i)_{i=1}^N$ be a sequence of inputs in $B_R(\mathcal{H}_Q)$.  Assume that there exists an integer $n_0$ such that for $n\geq n_0$, the data points $(u_i)_{i=1}^N$ satisfy
    \Cref{cond-input-data-simplified}, i.e., they 
     satisfy $u_i=\psi^{n} \circ \phi^{n}(u_i)$ for all $i\geq 1$. 
     Further assume that the $(\phi^n(u_i))_{1\leq i \leq N}$ are space filling in the sense 
     that %that they become dense in  $\phi^n(B_R(\mathcal{H}_Q))$ 
     for any $n \ge n_0$ we have
     % \todo{Do we need this for any $n$? or just $n \ge n_0$?}
     %, as $N\rightarrow \infty$ , i.e.,  for any $n\geq 1$,
    \begin{equation}\label{eqjguyguy}
    \lim_{N\rightarrow \infty}\sup_{u\in B_R(\mathcal{H}_Q)} \min_{1\leq i \leq N}\big|u_i(X^n)-u(X^n)|=0\,.
    \end{equation}
     Then for  any $t'\in (0,t)$, it holds that
\begin{equation}\label{eqlieredhieuddd}
\lim_{n, m \rightarrow \infty}\lim_{N\rightarrow \infty}\sup_{u \in B_R(\mathcal{H}_Q)}\| \Gc^\dagger(u) - \chi^m \circ \bar{f}^{m,n}_N \circ \phi^n(u) \|_{H^{t'}(D)}
=0
\end{equation}

\end{theorem}

\begin{proof}
Following \cite[Chap.~12.1]{owhadi_scovel_2019} define 
the projection $P_n^\U=\psi^n \circ \phi^n$  onto the range of $\psi^n$.
Since the points $X_i$ and $Y_j$ are dense in $\Omega$ and $D$ we have $h_{X^n}\downarrow 0$ as $n\rightarrow \infty$ and
$h_{Y^m}\downarrow 0$ as $m\rightarrow \infty$.
Given $n$, take
$\Upsilon = \phi^n \left( B_R(\mathcal{H}_Q) \right)$. Then \Cref{thmoepdjddei} and \eqref{eqregDkG} imply that
$\bar{f}^{m,n}_j\in H^{r'}(\Upsilon)$ for all $r'\geq 0$. Therefore $\bar{f}^{m,n}_j\in \mathcal{H}_{S^n}(\Upsilon)$. Now
\eqref{eqjguyguy} implies that for any $n$, the fill distance, in $\phi^n \left( B_R(\mathcal{H}_Q) \right)$, between the points $(\phi^n(u_i))_{1\leq i \leq N}$ goes to zero as $N\rightarrow \infty$.
Since the conditions of \Cref{thm:error-analysis-bar-g-and-g} are satisfied, we conclude by taking the limit $N\rightarrow \infty$ in \eqref{eqliedhieuddd} before taking the limit $m,n\rightarrow \infty$.
\end{proof}

\subsection{The effect of the $L_Q$ and $L_K$ preconditioners}\label{sec: cholesky factors}

We conclude this section and our discussion of convergence results, by 
highlighting the importance of the choice of the matrices 
$L^n_Q$ and $L^m_K$ in \eqref{phi-L-Q} and \eqref{varphi-L-K}.
It is clear from the bounds \eqref{eqliedhieuddd} and \eqref{eqleklhdeiudh} that our error estimates depend on the norms of the linear operators $\varphi^m, \psi^n$ and $\chi^m$.
To ensure that those norms do not blow up as $n, m\rightarrow \infty$ we can select the matrices $L^n_Q$ and $L^m_K$ to be the Cholesky factors of the precision matrices obtained from pointwise measurements of the kernels $Q$ and $K$, i.e.,
% Let $L_Q$ and $L_K$ be the Cholesky factors of $Q(X,X)^{-1}$ and $K(Y,Y)^{-1}$, i.e., $L_Q$ and $L_K$ are invertible lower triangular $n \times n$ and $m \times m$  matrices such that
\begin{equation}\label{eqchLqLK}
L^n_Q (L^n_Q)^T=Q(X^n,X^n)^{-1} \quad \text{and} \quad  L^m_K (L^m_K)^T = K(Y^m,Y^m)^{-1}.
\end{equation}
% Observe that under the choice \eqref{eqdiueds2} and \eqref{eqchLqLK},  $K(\varphi,\varphi)$ is the $m\times m$ matrix and regularizing in $\R^{m}$ is equivalent to regularizing in $\V$ as described in Sec.~\ref{secGPint}. Furthermore, 
We now obtain the following proposition.
\begin{Proposition}\label{propiehbdeyuid}
If $\phi^n$ is as in \eqref{phi-L-Q} and $L^n_Q$ as in \eqref{eqchLqLK}, then $\|\phi^n\|= 1$ and $\|\psi^n\|=1$.
If $\varphi^m$ is as in \eqref{eqdiueds2} and $L^m_K$ as in \eqref{eqchLqLK}, then $\|\varphi^m\|= 1$ and $\|\chi^m\|=1$.
\end{Proposition}
\begin{proof}
For $u\in \Hc_Q$, $|\phi^n(u)|^2=u(X^n)^T Q(X^n,X^n)^{-1} u(X^n)=\| \psi^n\circ \phi^n(u)\|_Q^2$. Since $\psi^n\circ \phi^n$ is a projection  \cite[Chap.~12.1]{owhadi_scovel_2019} we deduce that  $\|\phi^n\|=1$.  Using $\psi^n(U')=Q(\cdot,X^n) L^n_Q U'$ leads to
$\|\psi^n(U')\|_Q^2= |U'|^2$ and $\|\psi^n\|=1$. The proof of $\|\varphi^n\|= 1$ and $\|\chi^n\|=1$ is similar.
\end{proof}

We note that although useful for obtaining tighter approximation errors, this particular choice for the matrices $L^n_Q$ and $L^m_K$ is not required for convergence if one first takes the limit $N\rightarrow \infty$ as in \Cref{thm:main2}, which does not put any requirements on the matrices $L^n_Q$ and $L^m_K$ beyond invertibility.

\section{Numerics}\label{sec: Numerical results}
In this section, we present numerical experiments and benchmarks that compare a 
straightforward implementation of our kernel operator learning framework to 
state-of-the-art NN-based techniques. We discuss some implementation details of 
our method in \Cref{sec: implementation considerations} followed by 
the setup of experiments and test problems in \Cref{sec: experimental setup,sec: numerical results}. A detailed discussion of our findings is presented in \Cref{sec: discussion}.

\subsection{Implementation considerations}\label{sec: implementation considerations}
Below we summarize some of the key details in the implementation of our kernel approach for 
operator learning for benchmark examples. Our code to reproduce the experiments can be found in a public repository\footnote{\href{https://github.com/MatthieuDarcy/KernelsOperatorLearning/}{https://github.com/MatthieuDarcy/KernelsOperatorLearning/}}. 
\subsubsection{Choice of the kernel {$\Gamma$}}
Following our theoretical discussions in \Cref{sec: proposed solution,sec: convergence}, we primarily 
take $\Gamma$ to be a diagonal kernel of the form \eqref{eq:diagonal-kernel-recall}. 
This implies that our estimation of $\bar{f}$ can be split into independent problems for each 
of its components $\bar{f}_j$ in the RKHS of the scalar kernel $S$. 
In our experiments, we investigate different choices of $S$ belonging to the families of  
the linear kernel, rational quadratic, and Mat{\'e}rn; see \Cref{app: kernels} for 
detailed expressions of these kernels.  
The rational quadratic kernel has two parameters, the lengthscale $l$ and the exponent $\alpha$. We tuned 
these parameters using standard cross validation or log marginal likelihood maximization over the 
training data (see \cite[p.112]{rasmussen} for a detailed description). 
The Mat\'{e}rn kernel is parameterized by two positive parameters: a smoothness parameter $\nu$ and the length scale $l$. The smoothness parameter $\nu$ controls the regularity of the RKHS and we considered  $\nu \in \big\{\frac{1}{2}, \frac{3}{2}, \frac{5}{2}, \frac{7}{2}, \infty \big\}$. In practice we found that
$\nu = \frac{5}{2}$ almost always had the best performance. For a fixed choice of $\nu$ we tuned the 
length scale $l$ similarly to the rational quadratic kernel. 
We implemented the kernel regressions of the $\bar{f}_j$ and parameter tuning algorithms in scikit-learn for low-dimensional examples and manually in JAX for high-dimensional examples.

\subsubsection{Preconditioning and dimensionality reduction}\label{sec: dim red}
Following \eqref{eqdiueds} and \eqref{eqdiueds2} and the discussion in \Cref{sec: cholesky factors}, we 
consider two preconditioning strategies for our pointwise measurements, i.e., choices of the 
matrices $L_Q$ and $L_K$: (1) we consider the Cholesky factors of the underlying covariance 
matrices as in \eqref{eqchLqLK}; (2) we use PCA projection matrices 
of the input and output functions computed  from the training data. 
 We truncated the PCA expansions to preserve $(0.90, 0.95, 0.99)$ of the variance.  The use of PCA  
 in learning mappings between infinite dimensional spaces was proposed in \cite{krischer1993model} and  recently revisited in \cite{pca, Hesthaven-PCA-Net}. 

% \subsubsection{Preconditioning} Our convergence results in \cref{sec: convergence} require the preconditioning the vectors $u(X)$ and $v(X)$ by Cholesky factors of the precision matrices
% \begin{align}
%     &Q(X,X)^{-1}=L_Q L_Q^T \\
%     & K(Y,Y)^{-1} =L_K L_K^T \\
%     &\phi(u)= L_Q u(X) \\
%     &\varphi(v)= L_K v(Y).
% \end{align}
% The choice of the kernels $Q$ and $K$ is equivalent to the choice of function spaces $\mathcal U$ and $\mathcal V$, which are the RKHS of such kernels. In our numerical experiments, we choose $Q$ to be some elliptic operator and $K$ to be the Green's function of an elliptic operator (i.e. a Mat\'{e}rn kernel with prescribed regularity). A comparison between the Cholesky preconditioning and PCA preconditioning for low dimensional examples is presented in \cref{sec: discussion}.

% In our numerical experiments, we choose $Q$ to be $(-\Delta + \tau I)^{s}$ for some choice of $s$ and $K$ to be $ (-\Delta + \tau I)^{-t}$ for some choice of $t$ and appropriate boundary conditions. Typically,  $s = t = 1$, which is equivalent to choosing $\mathcal U = H^{-1}, \mathcal V = H^{1}$ with appropriate boundary conditions. 

\subsection{Experimental setup}\label{sec: experimental setup}

We compare the test performance of our method with different choices of the kernel $S$ of increasing complexity using the examples in \cite{de2022cost} and \cite{DeepONet-FNO} and their reported test relative $L^2$ loss (see \cref{eq:loss} below). 
We use the data provided by these papers for the training set and the test set\footnote{See \url{https://github.com/Zhengyu-Huang/Operator-Learning} and \url{https://github.com/lu-group/deeponet-fno}, respectively, for the data}. 
Both articles provide performance comparisons between different variants of Neural Operators (most notably FNO and DeepONet) on a variety of PDE operator learning tasks, where the data is sampled independently from a distribution $(Id, \mathcal{G}^\dagger)^\#\mu$ supported on $\mathcal U \times \mathcal V$, where $\mu$ is a specified (input) distribution on $\mathcal U$. 
The example problems are outlined in detail in \Cref{sec: numerical results}; a summary of 
the specific PDEs, problem type, and distribution $\mu$ for each test is given in \cref{table:examples}. 
In some instances the train-test split of the data was not clear from the available online repositories 
in which case we re-sampled them from the assumed distribution $\mu$. The datasets from \cite{DeepONet-FNO} contain 1000 training data-points per problem (which we will refer to as the \say{low-data regime}), whereas the datasets from \cite{de2022cost} contain 20000 training data-points (which we will refer to as the \say{high-data} regime). We make this distinction because the complexity of kernel methods, unlike that of neural networks, may depend on the number of data-points.  

Following the suggestion of
 \cite{de2022cost} we not only compare test errors and training complexity but also the  complexity of operator learning  at the inference/evaluation stage in \Cref{subsec:complex}. For the examples in \cite{de2022cost}, we investigate the accuracy-complexity trade-off of our method against the reported values of that article.

\begin{table}[H]
    \centering
    \begin{tabular}{|l|l|l|l|c|}
    \hline
        \textbf{Equation} & \textbf{Input} & \textbf{Output} & \textbf{Input Distribution $\mu$}  \\
        \hline
        Burger's & Initial condition & Solution at time $T$ & Gaussian field (GF)  \\
        \hline
            Darcy problem & Coefficient & Solution  & Binary function of GF   \\
            \hline
        Advection I & Initial condition & Solution at time $T$  & Random square waves   \\
        \hline
        Advection II & Initial condition & Solution at time $T$ & Binary function of GF  \\
        \hline
        Helmholtz & Coefficient & Solution & Function of Gaussian field \\
        \hline
        Structural mechanics & Initial force & Stress field & Gaussian field  \\
        \hline
        Navier Stokes & Forcing term & Solution at time $T$ & Gaussian field \\
        \hline
    \end{tabular}
    \caption{Summary of datasets used for benchmarking. The first three examples were considered in \cite{DeepONet-FNO}, and the last four were taken from \cite{de2022cost}.}
    \label{table:examples}
\end{table}

\subsubsection{Measures of accuracy}\label{sec: accuracy}

As our first performance metric we 
 measured the  accuracy of models  by a relative loss on the output space $\mathcal V$:
\begin{equation}
    \mathcal{R}( \Gc ) = \mathbb{E}_{u \sim \mu} \Bigg[\frac{\norm{\Gc^\dagger(u) - \Gc(u)}_{\mathcal V}}{\norm{\Gc^\dagger(u)}_{\mathcal V}} \Bigg]
\end{equation}
% \todo{are we sure the norms are not squared? \md{Andrew's paper explicitly defines it in this way. George's paper does not, but looking at their code, it seems that this is the correct metric.}}
where $\Gc^\dagger$ is true operator and $\Gc$ is a candidate  operator. Following previous works, we often took $\norm{u}_{\mathcal V} = \norm{u}_{L^2} := (\int u(x)^2dx)^{\frac{1}{2}}$, which in turn is discretized using the trapezoidal rule. In practice, we do not have the access to the underlying probability measure $\mu$ and we compute the empirical loss on a withheld test set:
\begin{equation}\label{eq:loss}
    \mathcal{R}_N(\Gc) = \frac{1}{N} \sum_{n=1}^N \Bigg[\frac{\norm{\Gc^\dagger(u^n) - \Gc(u^n)}_{\mathcal V}}{\norm{\Gc^\dagger(u^n)}_{\mathcal V}} \Bigg], \qquad u^i \sim \mu.
\end{equation}

\subsubsection{Measures of complexity}\label{subsec:complex}
For our second performance metric we considered the complexity of operator 
learning algorithms at the inference stage (i.e., evaluating the 
learned operator).
Complexity at inference time is the main metric used in \cite{de2022cost} to compare numerical methods for operator learning. The motivation is that training of the methods can be performed in an offline fashion, and therefore the cost per test example dominates in the limit of many test queries. In particular, they compare the online evaluation costs of the neural networks by computing the requisite floating point operations (FLOPs) per test example. We adopt this metric as well for the methods not based on neural networks that we develop in this work, and we compare, when available, the cost-accuracy tradeoff with the numbers reported in \cite{de2022cost}. We computed the FLOPs with the same assumptions as in the original work: a matrix-vector product where the input vector is in $\mathbb{R}^n$ and the output vector is in $\mathbb{R}^m$ amounts to $m(2-1)$ flops, and non-linear functions with $n$-dimensional inputs (activation functions for neural networks, kernel computations for kernel methods)
are assumed to have  cost $ \mathcal{O}(n)$. 

\remark[Training complexity]{
While the inference complexity of a model  eventually dominates the cost of training 
during applications, the training cost cannot be ignored since the allocated 
computational resources during this stage may still be limited and the resulting 
errors will have a profound impact on the quality and performance of the 
learned operators. Therefore numerical methods in which the offline data assimilation step is cheaper, faster, and more robust will always be preferred. Computing the exact number of FLOPs at training time is difficult to estimate for NN methods, as it depends on the optimization algorithms used, the hyperparameters and the optimization over such hyperparameters, among many other factors. Therefore in this work we limit the training complexity evaluation to the qualitative observation that
kernel methods provided in this work are significantly simpler at training time, as they have no NN weights, they do not require the use of stochastic gradient descent, and have few or no hyperparameters which can be tuned using standard methods such as grid search or gradient descent in a low-dimensional 
space.
}

\subsection{Test problems and qualitative results}\label{sec: numerical results}

Below we outline the setup of each of our  benchmark problems.
In all cases, $\mathcal U$ and $\mathcal V$ are spaces of real-valued functions with input domains $\Omega, D\subset \mathbb{R}^k$ for $k = 1$ or $2$. Whenever $\Omega = D$, we simply write $\mathcal{D}$ for both.

\subsubsection{Burger's equation}\label{sec: burger}
Consider the one-dimensional Burger's equation:
\begin{equation}\label{eq:burgers}
\begin{aligned}
    \frac{\partial w}{\partial t } + w \frac{\partial w}{\partial x} &= \nu \frac{\partial^2 w}{\partial x^2}, && (x, t) \in (0,1) \times (0,1 ], \\
    w(x, 0) &= u(x), && x \in (0,1)
\end{aligned}
\end{equation}
with $\mathcal{D} = (0,1)$, and periodic boundary conditions. The viscosity parameter $\nu$ is set to  $0.1$. We learn the operator mapping the initial condition $u$ to 
$v = w(\cdot, 1)$,
the solution at time $t = 1$, i.e., $\mathcal{G}^\dagger: w(\cdot,0) 
\mapsto w(\cdot,1).$

%\begin{equation}\label{eq: burger solution}
%    \mathcal{G}^\dagger: u_0(x) \mapsto u(x,1).
%\end{equation}
The training data is generated by sampling the initial condition $u$ from a GP with a Riesz kernel, denoted by $\mu = \mathcal{GP}(0, 625(-\Delta  + 25I )^{-2}))$. As in \cite{DeepONet-FNO}, we used a spatial resolution with 128 grid points to represent the input and output functions, and used 1000 instances for training and 200 instances for testing. \Cref{fig:BurgersB} shows an example of training input and output pairs as well as a test example along with its pointwise error.

\begin{figure}[htp]%
    \centering
    \subfloat[\centering \footnotesize Training input]{{\includegraphics[width=0.195\columnwidth]{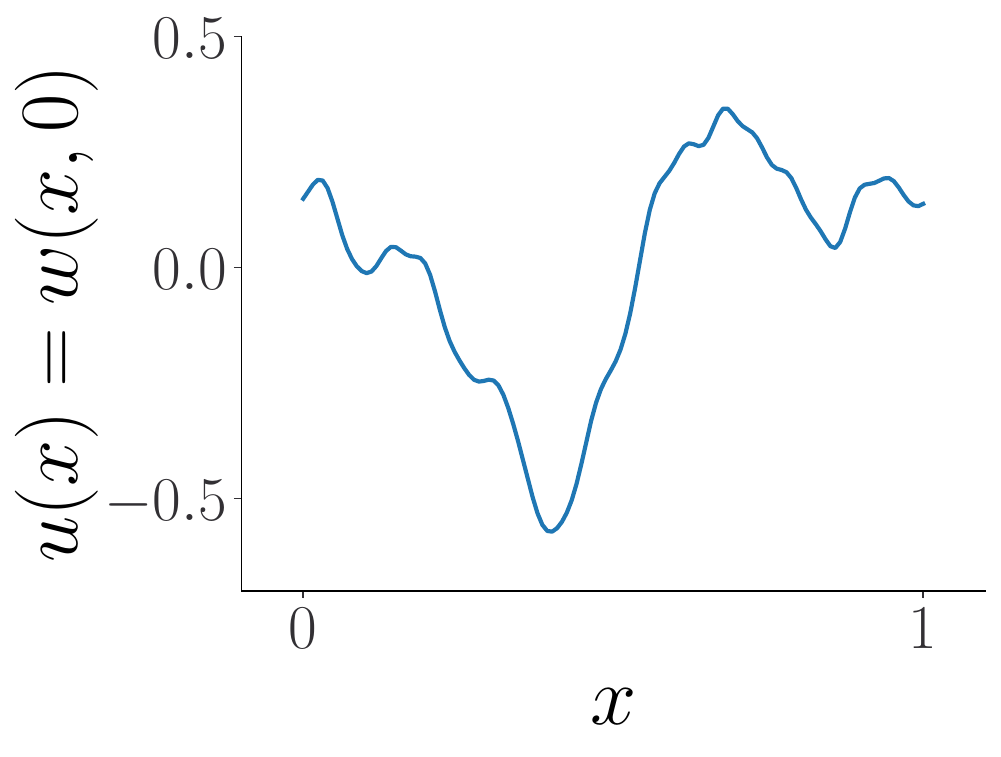} }}%
    \subfloat[\centering \footnotesize Training output]{{\includegraphics[width=0.195\columnwidth]{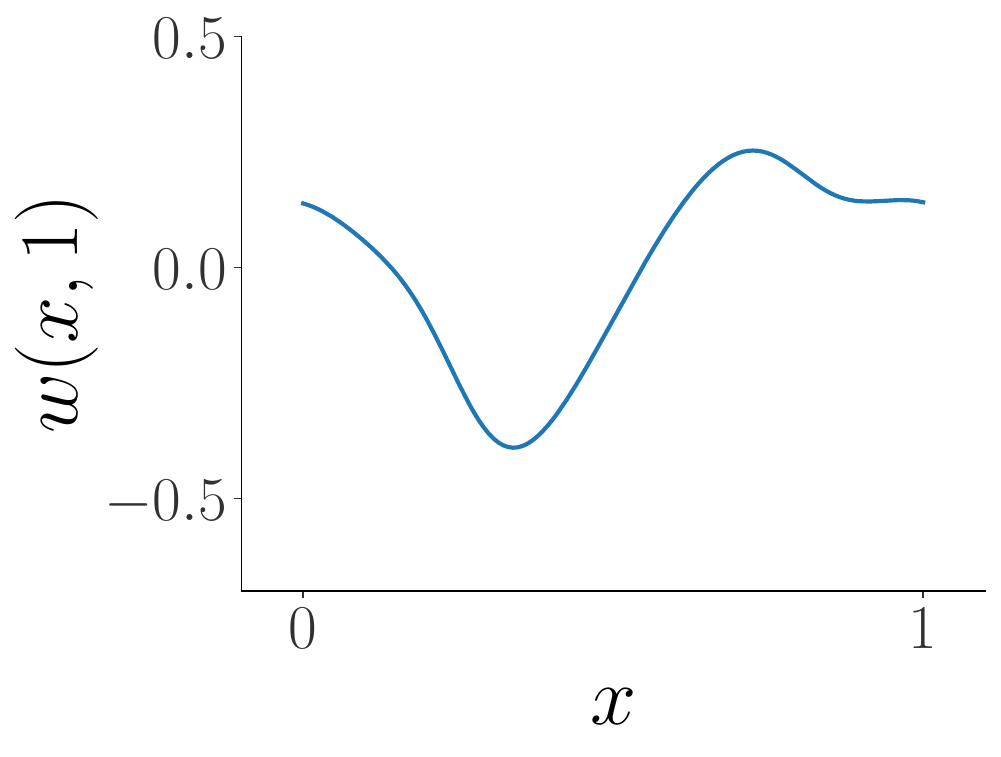} }}%
     %Example of input and output of the learned map $\mathcal{G}^\dagger$ in the Burger's equation problem
%     }%
%     \label{fig:BurgersA}%
% \end{figure}
% \begin{figure}[H]%
%     \centering
    \subfloat[\centering \footnotesize True test]{{\includegraphics[width=0.195\columnwidth]{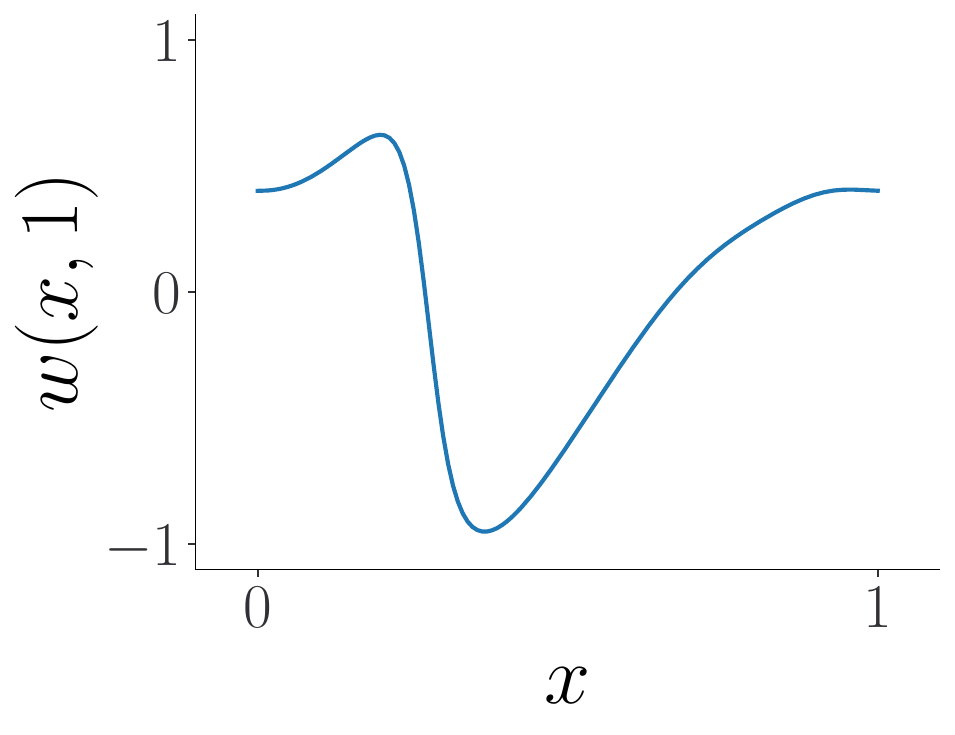}}}%
    \subfloat[\centering \footnotesize Predicted test]{{\includegraphics[width=0.195\columnwidth]{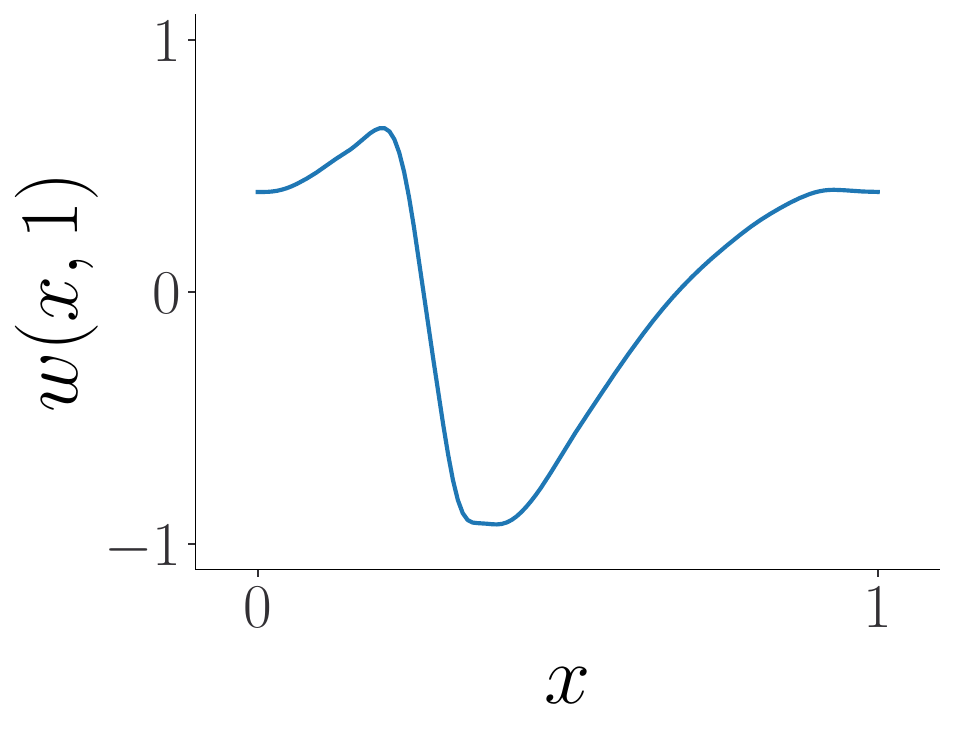}}}%
    \subfloat[\centering \footnotesize Pointwise error]{{\includegraphics[width=0.195\columnwidth]{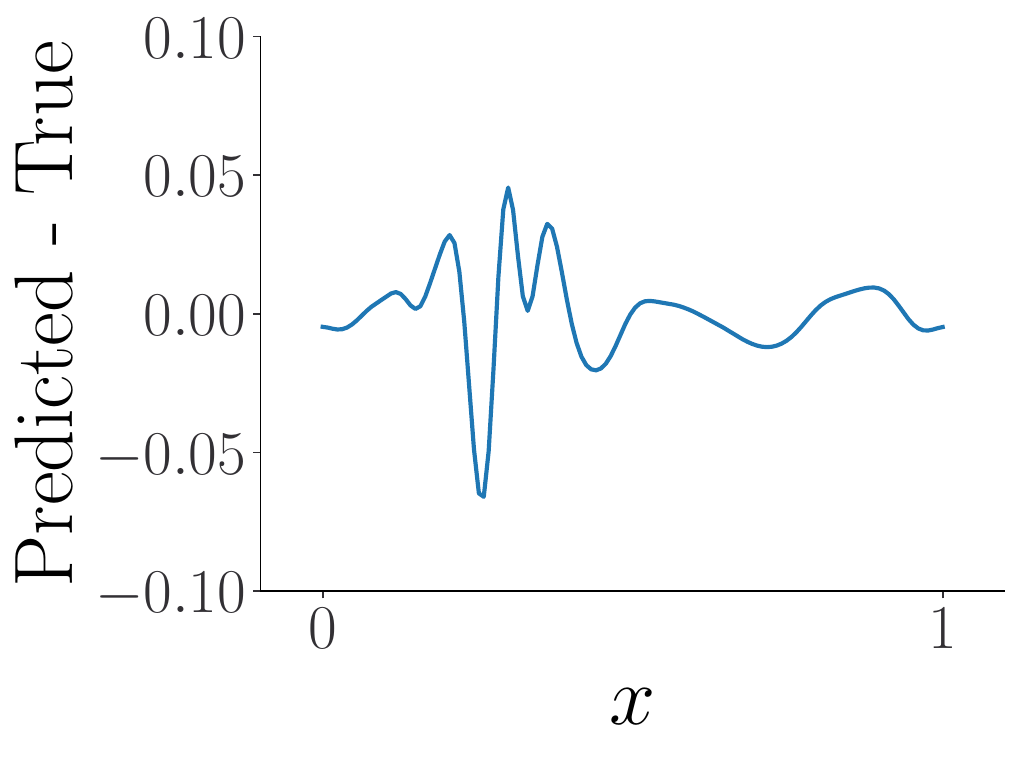}}}%
     %\caption{\md{Example of training data, test prediction, pointwise errors for Burger's equation \eqref{eq:burgers}.}}%
    \caption{Example of training data and test prediction and pointwise errors for the Burger's equation \eqref{eq:burgers}.}%
    \label{fig:BurgersB}%
\end{figure}

% \begin{figure}[H]
%     \centering
%     \includegraphics[height = 5cm, width = 15cm]{images/burger_prediction_gp.png}
%     \caption{Burgers' equation. An example prediction by a GP with Mat\'{e}rn kernel. \pb{Can we change this input on the left, output on the right?}}
%     \label{fig: burger prediction }
% \end{figure}

\subsubsection{Darcy flow}\label{sec: darcy problem}
Consider the two-dimensional Darcy flow problem \eqref{eq:Darcy-1D}.
% \begin{equation}\label{eq:darcy-flow}
% \begin{aligned}
%     - \nabla \cdot (u(x)\nabla v(x)) &= f, \quad x \in \mathcal{D} \\
%     u(x) &= 0, \quad \partial \mathcal D
% \end{aligned}
% \end{equation}
% with $\mathcal D = (0,1)^2$.
Recall that in this example, we are interested in learning the mapping from the permeability field $u$ to the solution $v$ and the source term $w$ is assumed 
to be fixed, hence $\mathcal{D} \equiv \Omega  = (0,1)^2$ and 
$\Gc^\dagger: u \mapsto v$.
% (here $f$ is considered fixed):
% \begin{equation}
%     \mathcal{G}^\dagger: u(x) \mapsto v(x).
% \end{equation}
The coefficient $u$ is sampled by setting $u \sim \log \circ h_\sharp \mu$ where $\mu = \mathcal{GP}(0, (-\Delta + 9I)^{-2})$ is a GP and $h$ is binary function mapping positive inputs to $12$ and negative inputs to $3$. The resulting permeability/diffusion coefficient $e^u$ is therefore piecewise constant. As in \cite{DeepONet-FNO}, we use a discretized grid of resolution $29 \times 29$, with the data generated by the MATLAB PDE Toolbox. We use 1000 points for training and 200 points for testing. \Cref{fig: darcy_combined_figure} shows an example of training input and output of the map $\mathcal{G}^\dagger$, and an example 
of predictions along with pointwise error at the test stage.

\subsubsection{Advection equations (I and II)}\label{sec: advection 1}

Consider the one-dimensional advection equation:
\begin{equation}\label{eq:advection}
\begin{aligned}
    \frac{\partial w}{\partial t } + \frac{\partial w}{\partial x} &= 0  &&x \in (0,1), t \in (0,1 ] \\
    w(x, 0) &= u(x) &&x \in (0,1)
\end{aligned}
\end{equation}
with $\mathcal D = (0,1)$ and periodic boundary conditions.
Similar to the example for Burgers' equation, we learn the 
mapping from the initial condition $u$ to 
$v = w(\cdot, 0.5)$, the solution at $t=0.5$, i.e., 
$\Gc^\dagger: w(\cdot, 0) \mapsto w(\cdot, 0.5)$.

% \begin{equation}
%     \mathcal{G}^\dagger: u_0(x) \mapsto u(x,0.5).
% \end{equation}

This problem was considered in \cite{DeepONet-FNO, de2022cost} with different distributions $\mu$ for the initial condition. We will show in the following section how these different distributions lead to different performances. In \cite{DeepONet-FNO}, henceforth referred to as Advection I, the initial condition is a square wave centered at $x = c$ of width $b$
and height $h$:
\begin{equation}
    u(x) = h \bold{1}_{\{c - \frac{b}{2} ,c + \frac{b}{2} \}},
\end{equation}
where the parameters $(c, b, h) \sim \mathcal{U}([0.3, 0.7] \times [0.3, 06] \times [1,2])$.
In \cite{de2022cost}, henceforth referred to as Advection II, the initial condition is
\begin{equation}
    u = -1 + 2 \bold{1}\{\tilde{u}_0 \geq 0\}
\end{equation}
where $\tilde{u}_0 \sim \mathcal{GP}(0, (-\Delta+3^2I)^{-2})$.

For Advection I, the spatial grid was of resolution 40, and we used 1000 instances for training and 200 instances for testing. For Advection II, the resolution was of 200 and we used 20000 training and test instances, 
following \cite{de2022cost}.

\Cref{fig:Adv1B,fig:Adv2A} show an example of training input and output 
for Advection the I  and II problems, respectively. Observe that the functional samples from the distribution in Advection I will have exactly two discontinuities almost surely, but the samples for Advection II can  have many more jumps. We observe that prediction is challenging around discontinuities, hence Advection II is a significantly harder problem (across all benchmarked methods) than Advection I. \Cref{fig:Adv1B,fig:Adv2A}  also show an instance of a test sample, along with 
a prediction and the pointwise errors.

\begin{figure}[htb]%
    \centering
    \subfloat[\centering Training input]{{\includegraphics[width=0.195\columnwidth]{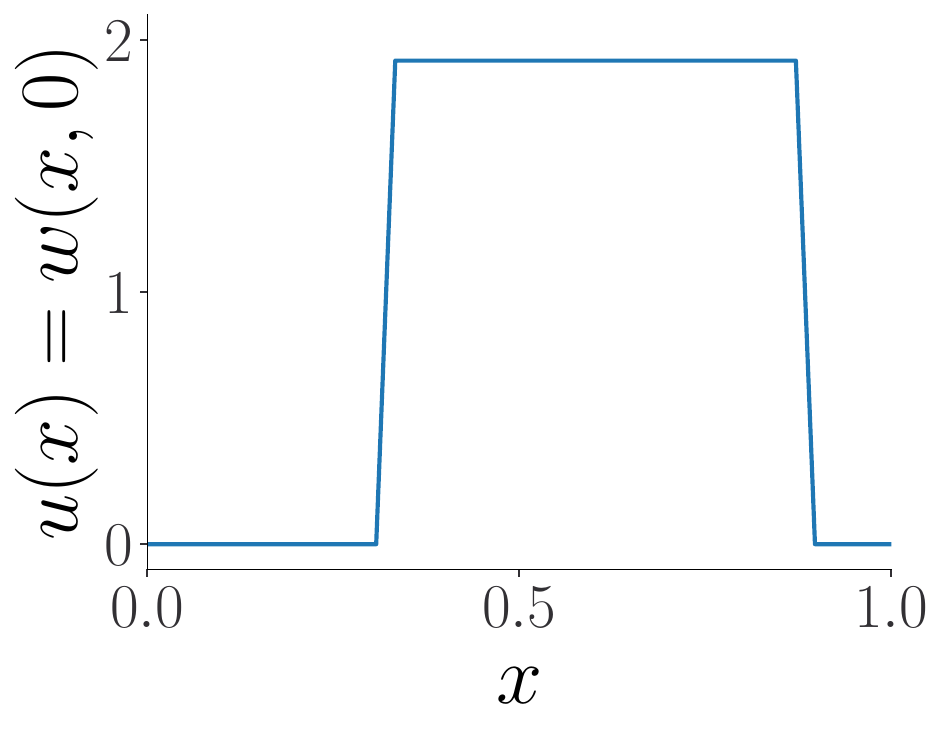} }}%
    % \quad
    \subfloat[\centering Training output]{{\includegraphics[width=0.195\columnwidth]{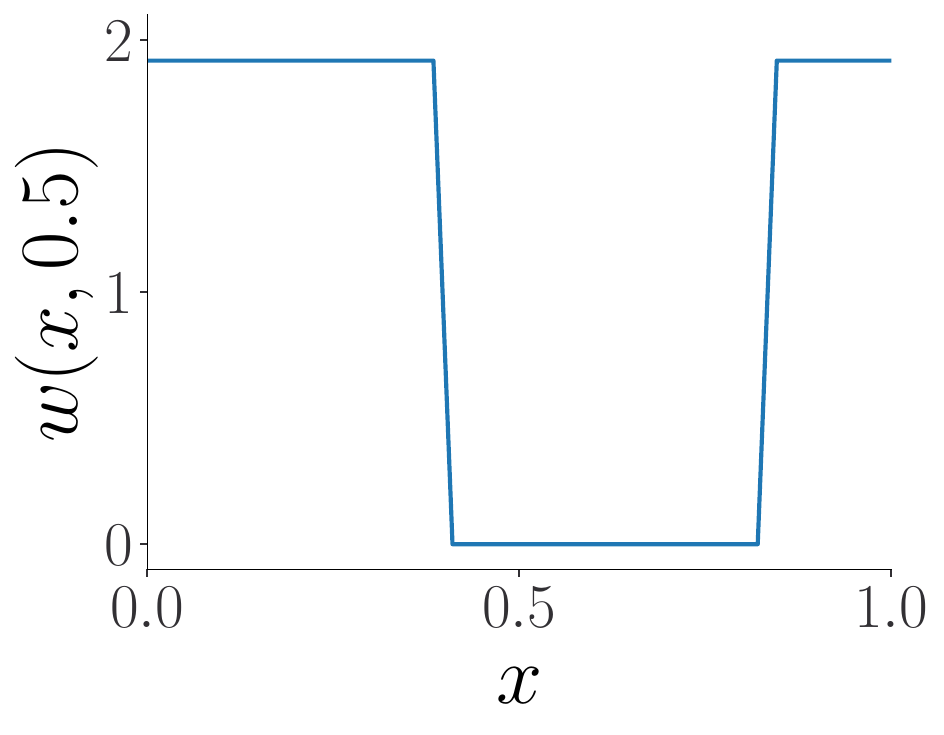} }}%
%     \caption{}%
%     \label{fig:Adv1A}%
% \end{figure}
% \begin{figure}[htb]%
%     \centering
    \subfloat[\centering True test]{{\includegraphics[width=0.195\columnwidth]{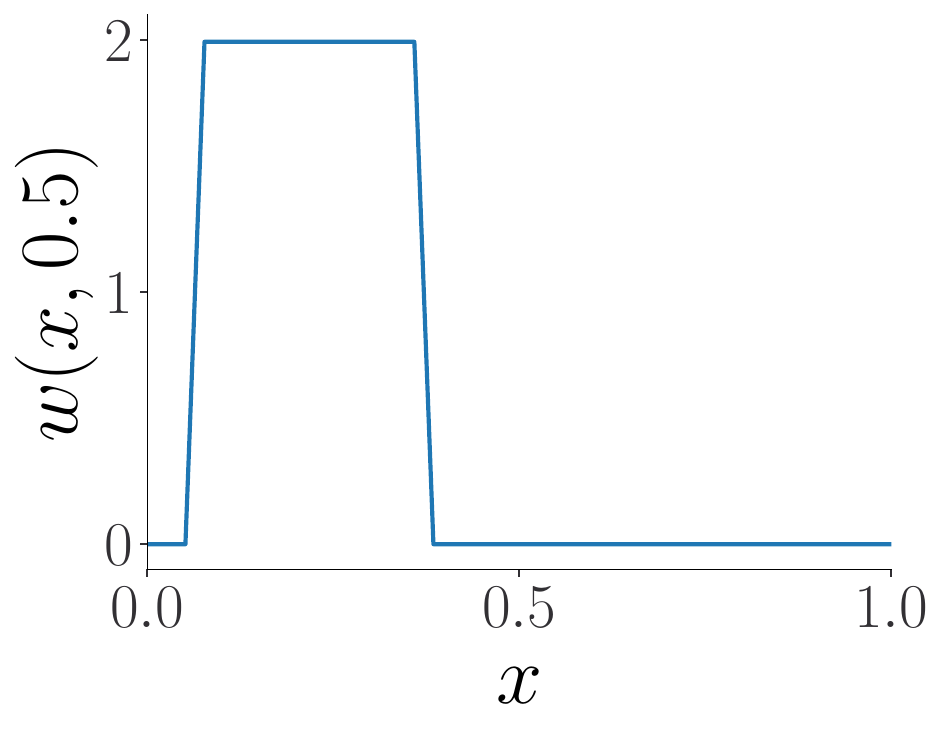}}}%
    \subfloat[\centering Predicted test]{{\includegraphics[width=0.195\columnwidth]{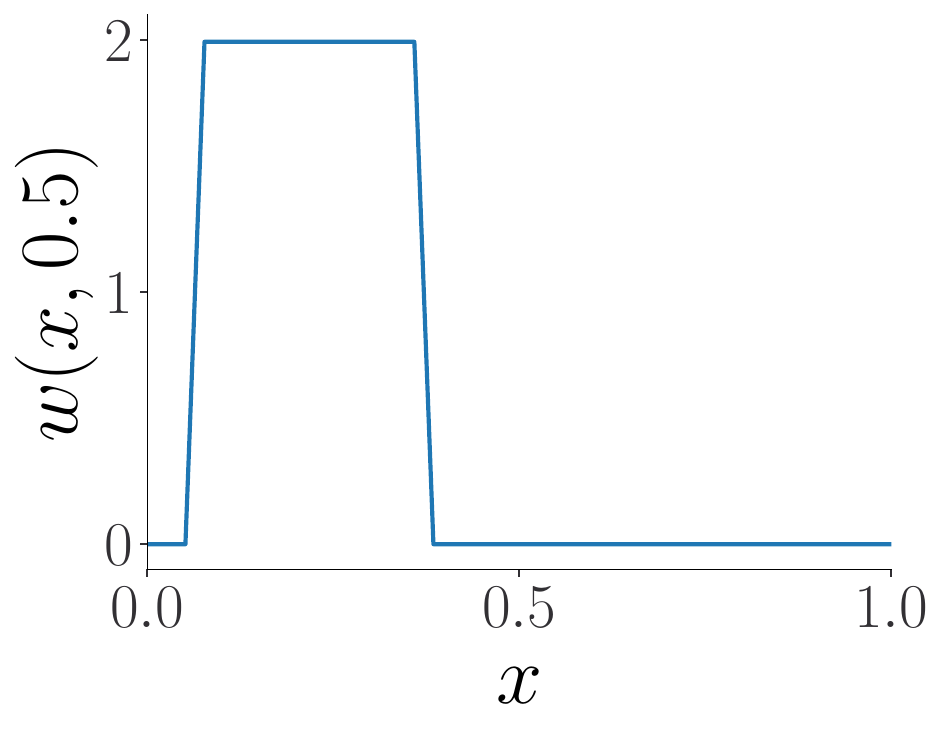}}}%
    \subfloat[\centering Pointwise error]{{\includegraphics[width=0.195\columnwidth]{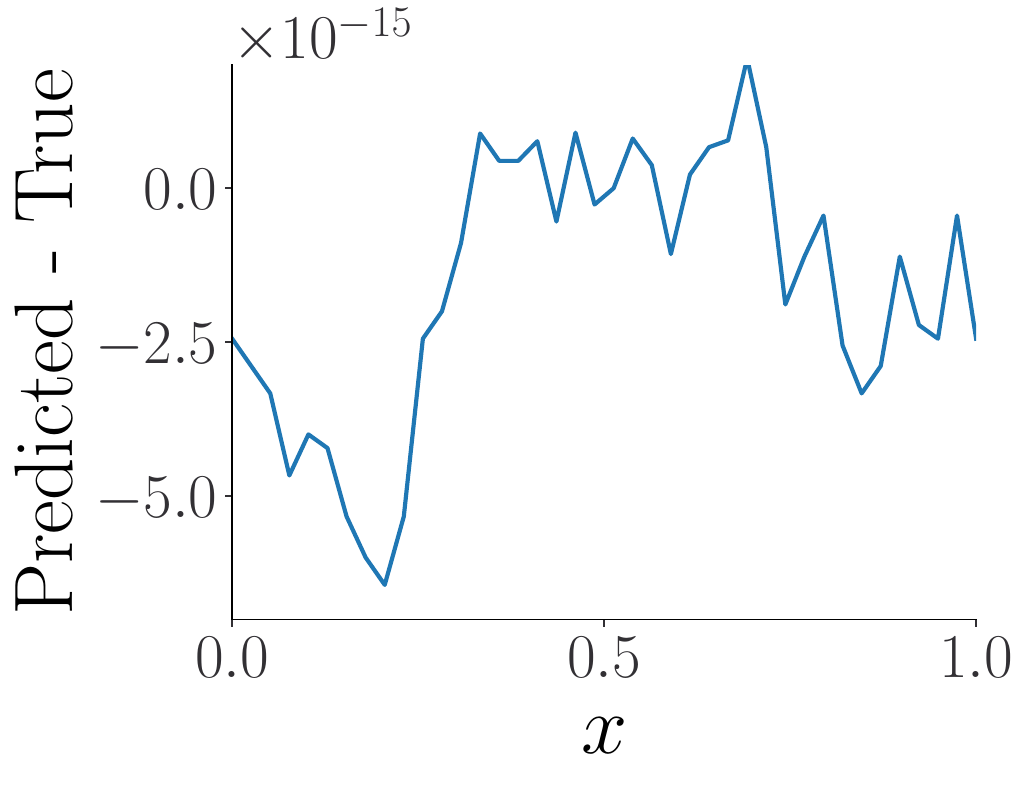}}}%
    \caption{ Example of training data and test prediction and pointwise errors for the Advection problem \eqref{eq:advection}-I.}
    \label{fig:Adv1B}%
\end{figure}
%\todo{Update captions and labels}

\begin{figure}[htb]%
    \centering
    \subfloat[\centering Training input]{{\includegraphics[width=0.195\columnwidth]{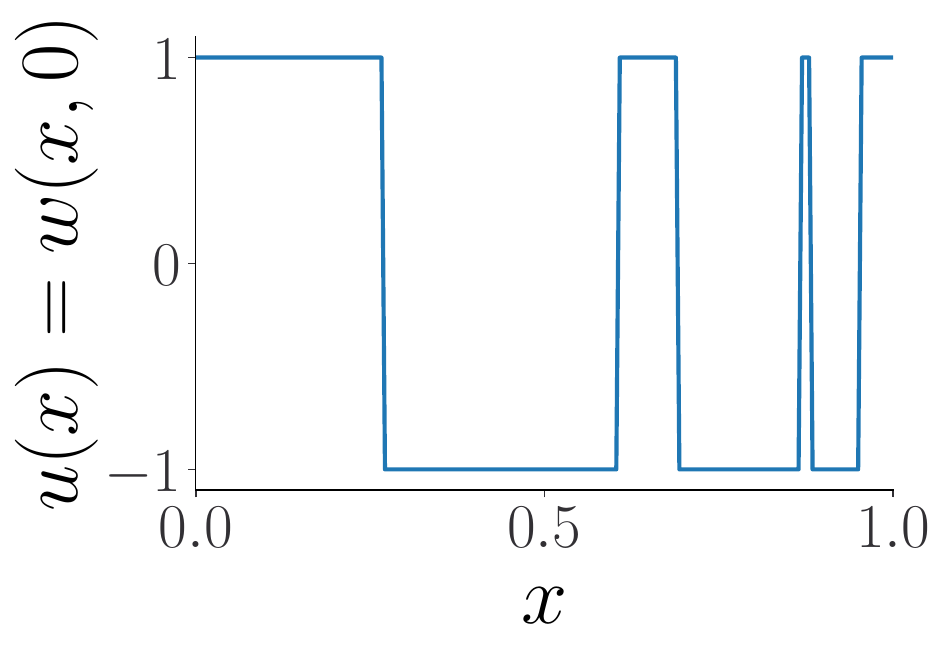} }}%
    \subfloat[\centering Training output]{{\includegraphics[width=0.195\columnwidth]{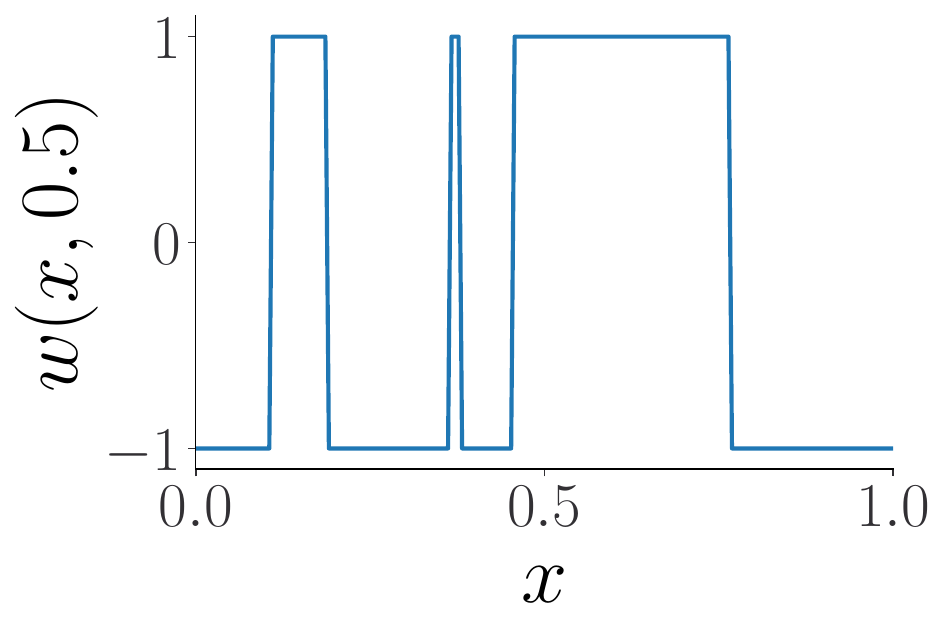} }}%
    \subfloat[\centering True test]{{\includegraphics[width=0.195\columnwidth]{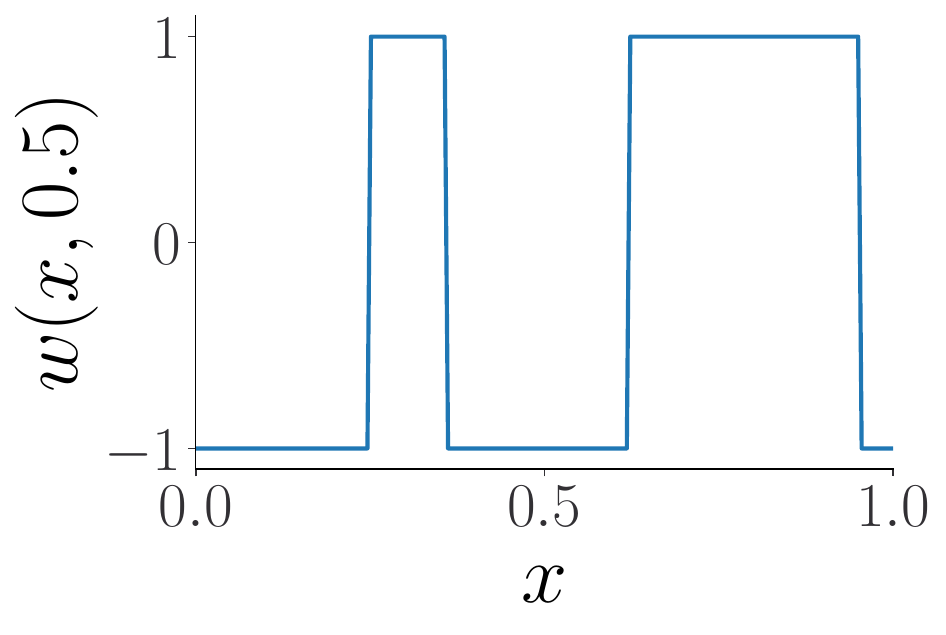}}}%
    \subfloat[\centering Predicted test]{{\includegraphics[width=0.195\columnwidth]{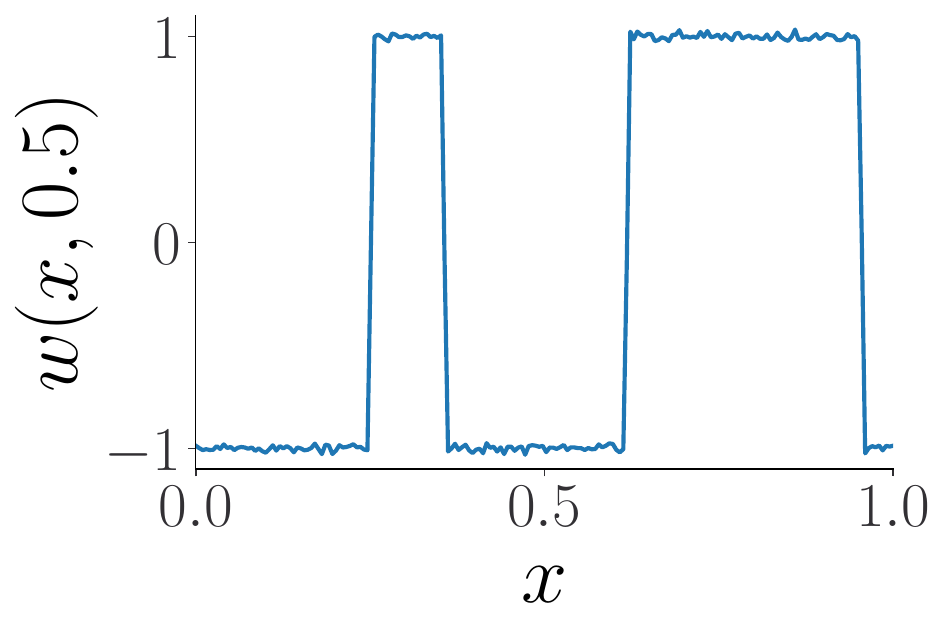}}}%
    \subfloat[\centering Pointwise error]{{\includegraphics[width=0.195\columnwidth]{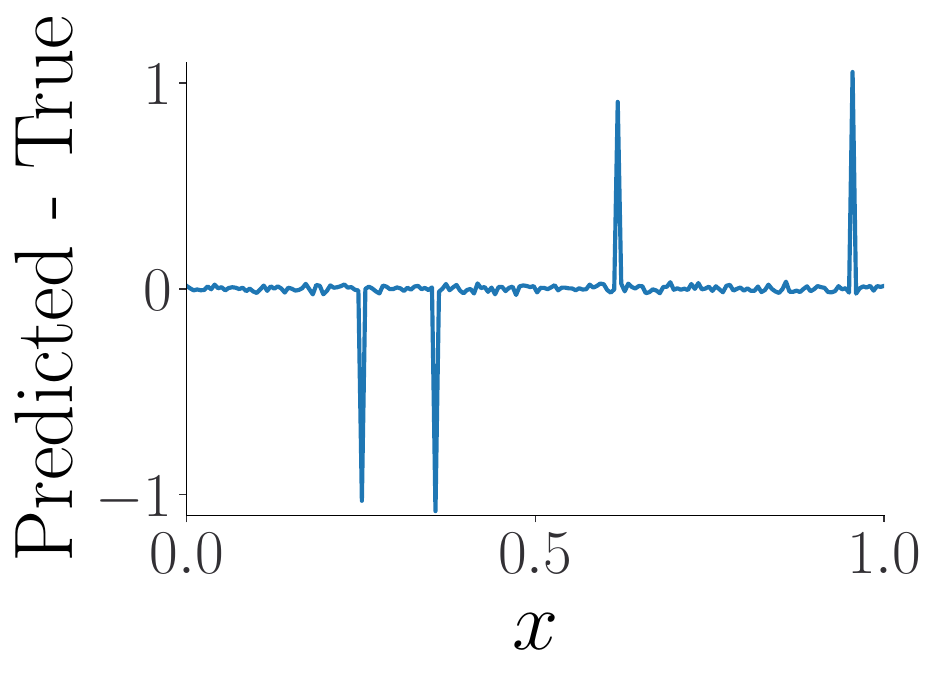}}}%
     \caption{Example of training data and test prediction and pointwise errors for the Advection problem \eqref{eq:advection}-II.}%
    \label{fig:Adv2A}%
\end{figure}
%\todo{update captions and tables}

% \begin{figure}[htb]%
%     \centering

%     \caption{Example of test set prediction of the Advection II problem}%
%     \label{fig:Adv2B}%
% \end{figure}
% \begin{figure}[H]

%      \begin{subfigure}{0.7\textwidth}
%      \advance\rightskip-5cm
%        \includegraphics[height = 5cm, width = 15cm]{images/advection1_training_data.png}
%          \label{fig: advection 1 data}
%         \caption{An example data point.}
%      \end{subfigure}
%         \hfill
%      \begin{subfigure}{0.7\textwidth}
%           \advance\rightskip-5cm
%         \includegraphics[height = 5cm, width = 15cm]{images/advection1_prediction_lr.png}
%          \caption{An example prediction by linear regression.}
%          \label{fig: advection 1 prediction}
%      \end{subfigure}
%     \caption{Advection I.}
%     \label{fig: advection 1 example}
% \end{figure}

\subsubsection{Helmholtz's equation}\label{sec: helmholtz}
 For a given frequency $\omega$ and wavespeed field $u: \mathcal{D} \to \mathbb R$,
 with $\mathcal{D} = (0, 1)^2$,
 the excitation field $v: \mathcal{D}\to \mathbb R$ solves
\begin{equation}  \label{eq: helmholtz}
\begin{aligned}
    &\left (-\Delta - \frac{\omega^2}{u^2(x)} \right )v = 0, \quad x \in (0,1)^2 \\
   & \frac{\partial v}{\partial n} = 0, \quad x \in \{0, 1\}\times[0,1] \cup [0,1] \times \{0\} \quad \text{and} \quad 
    \frac{\partial v}{\partial n} = v_N, \quad x \in [0,1]\times \{1\}
\end{aligned}
\end{equation}
In the results that follow, we take $\omega = 10^3$, $v_N = \bold{1}_{\{0.35 \leq x \leq 0.65\}}$, and we aim to learn the map $\Gc: u \mapsto v$, i.e., the 
mapping from the wavespeed field to the excitation field.
% sending $c$ to $u$:
% \begin{equation}
%     \mathcal G^\dagger: c(x) \mapsto u(x)
% \end{equation}
The distribution $\mu$ is specified as the law of 
$u(x) = 20 + \tanh(\tilde{u}(x))$, where $\tilde{u}$ is drawn from the 
GP, $\mathcal{GP}(0,  (-\Delta +3^2I)^{-2})$. The training and test data were generated by solving \eqref{eq: helmholtz} with a Finite Element Method on a discretization of size $100\times100$ of the unit square. \Cref{fig:HellA} shows an example of training input and output, a test prediction, and pointwise errors.
 
\begin{figure}[htp]%
    \centering
    \subfloat[\centering Training input]{{\includegraphics[width=0.18\columnwidth]{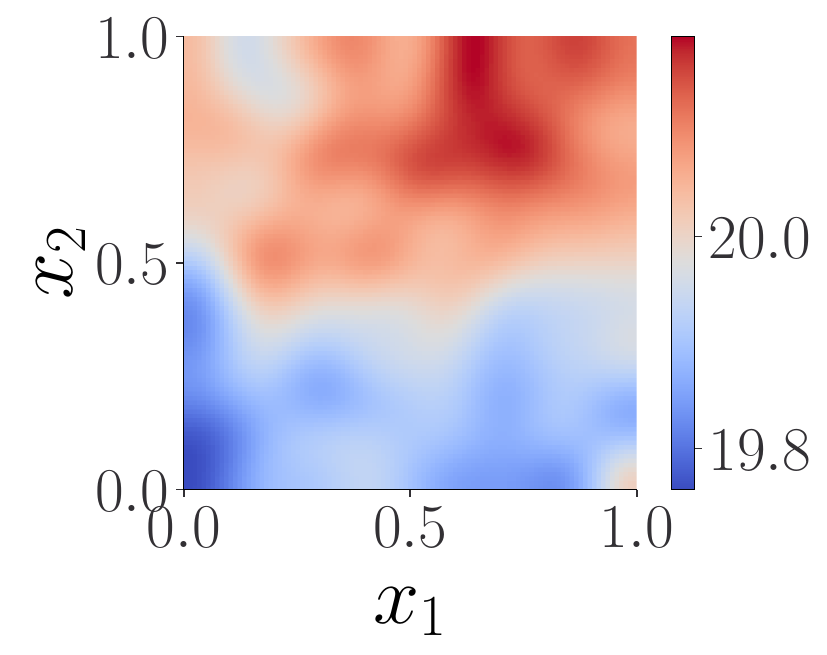} }}%
    \subfloat[\centering Training output]{{\includegraphics[width=0.2\columnwidth]{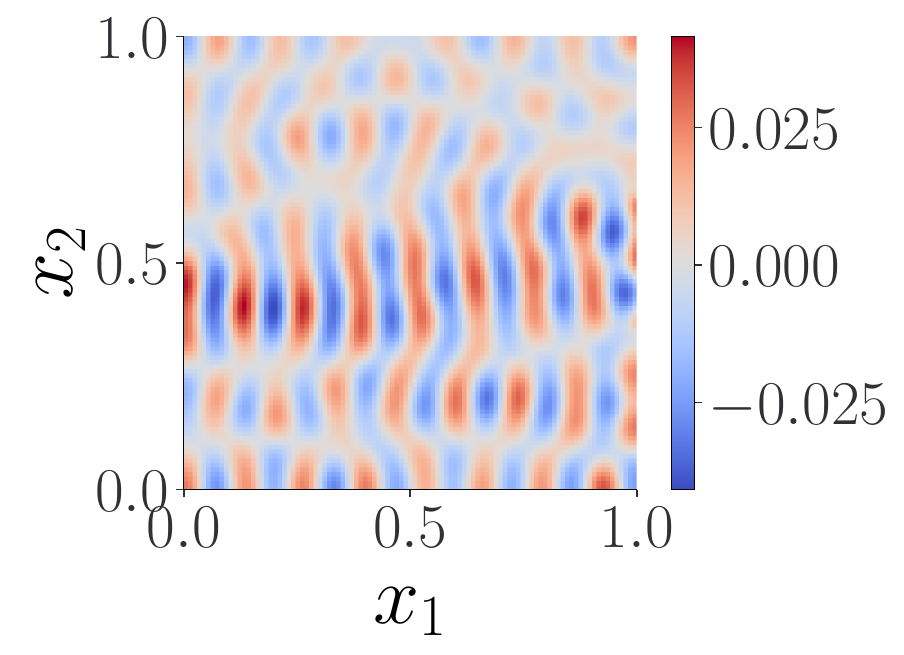} }}%
    \subfloat[\centering True test]{{\includegraphics[width=0.2\columnwidth]{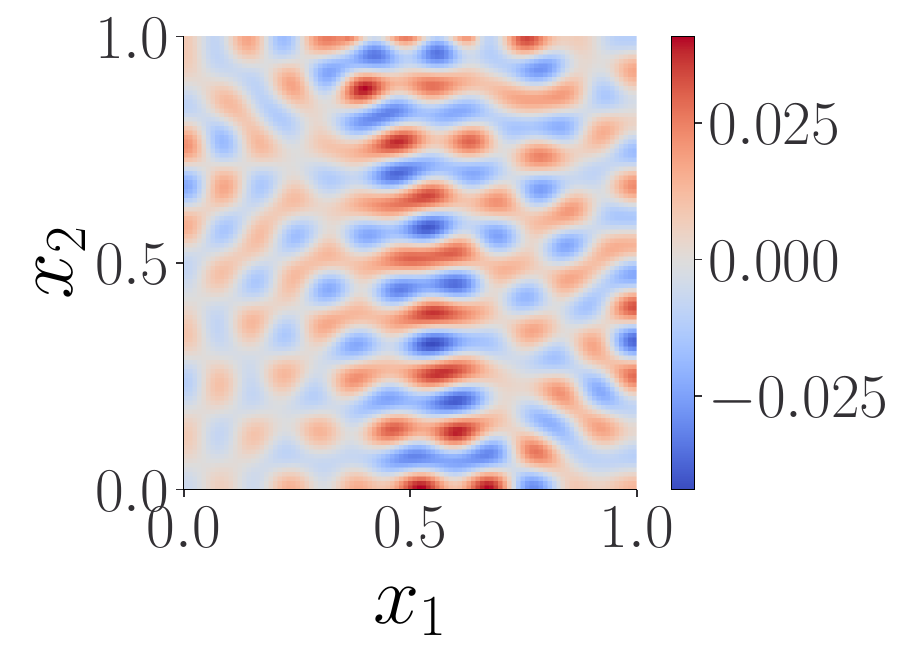}}}%
    \subfloat[\centering Predicted test]{{\includegraphics[width=0.2\columnwidth]{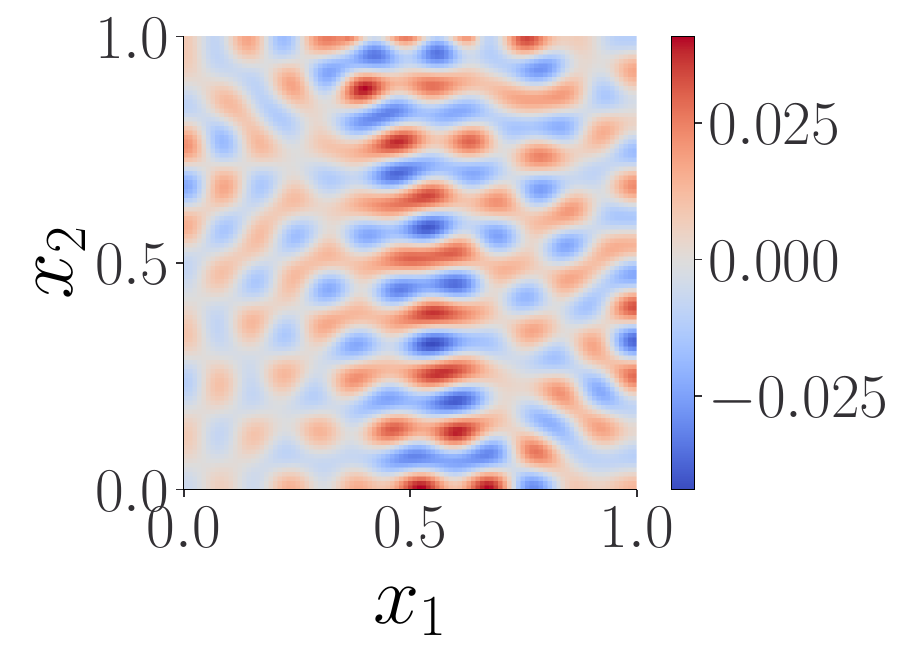}}}%
    \subfloat[\centering Pointwise error]{{\includegraphics[width=0.182\columnwidth]{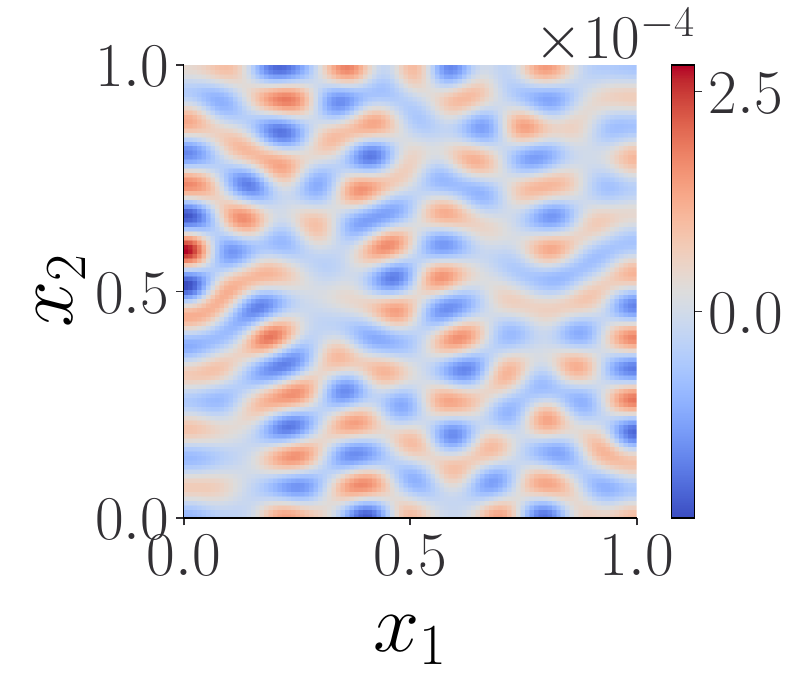}}}%
    \caption{Example of training data and test prediction and pointwise errors for the Helmholtz problem \eqref{eq: helmholtz}.}%
    \label{fig:HellA}%
\end{figure}

% \begin{figure}[htb]%
%     \centering

%     \caption{Example of test set prediction of the Helmholtz problem}%
%     \label{fig:HellB}%
% \end{figure}
 \subsubsection{Structural mechanics}\label{sec: structural eq}
We let $\Omega = [0,1]$ $D = [0,1]^2$, the equation that governs the displacement vector $w$ in an elastic solid undergoing infinitesimal deformations is
\begin{equation}\label{eq:structural-mechanics}
    \begin{aligned}
    \nabla \cdot \sigma = 0 \quad \text{in } (0,1)^2, \qquad 
    w  = \bar{w}, \quad \text{on } \Gamma_w,  \qquad 
 \nabla \cdot n = u \quad \text{on } \Gamma_u,
    \end{aligned}
\end{equation}
where the boundary $\partial D$ is split in $[0,1]\times{1} = \Gamma_t$ (the part of the boundary subject to stress) and its complement $\Gamma_u$.

The goal is to learn the operator that maps the one-dimensional load $u$ on $\Gamma_u$ to the two-dimensional von Mises stress field $v$ on $\Omega$, i.e., 
$\Gc: u \mapsto v$.
Here the distribution $\mu$ is $\mathcal{GP}(100, 400^2(-\Delta+3^2 I)^{-1})$, with $\Delta$ being the Laplacian subject to homogeneous Neumann boundary conditions on the space of zero-mean functions. The function $v$ was obtained by a finite element code, see \cite{de2022cost} for implementation details and the constitutive model used. \Cref{fig:StrucA} shows an example of training input and outputs, a test prediction, and pointwise errors.

\begin{figure}[htp]%
    \centering
    \subfloat[\centering Training input]{{\includegraphics[width=0.195\columnwidth]{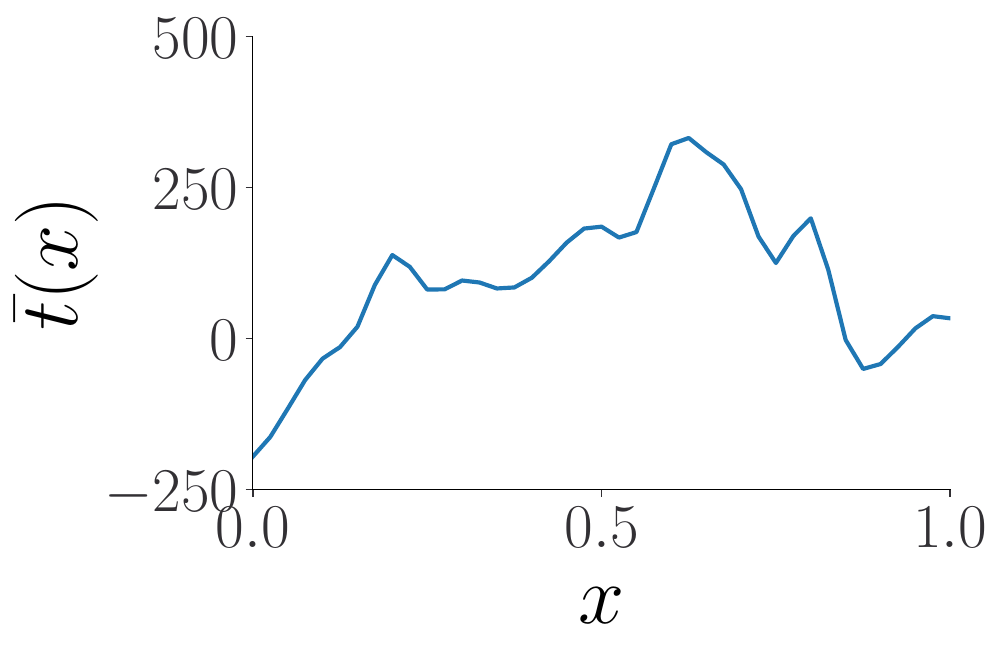} }}%
    \subfloat[\centering Training output]{{\includegraphics[width=0.195\columnwidth]{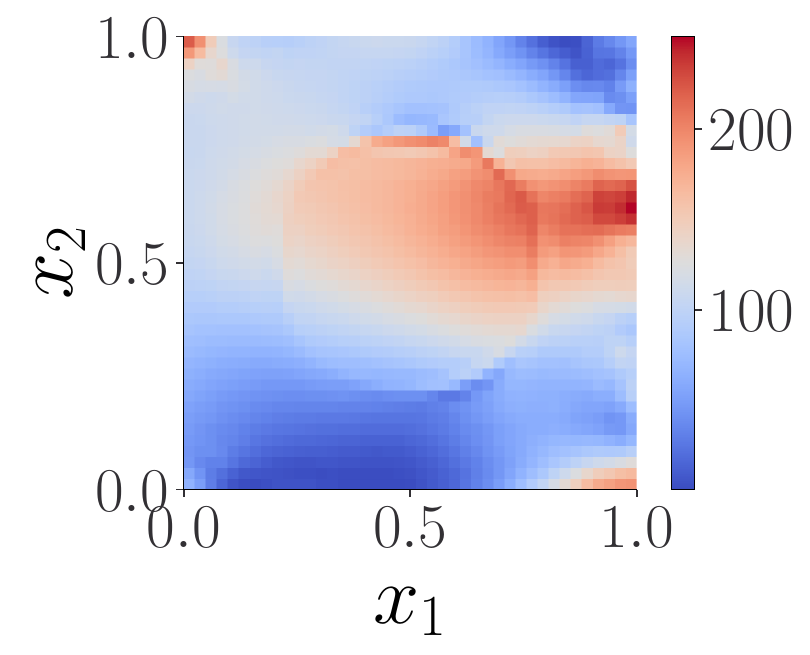} }}%
    \subfloat[\centering True test]{{\includegraphics[width=0.195\columnwidth]{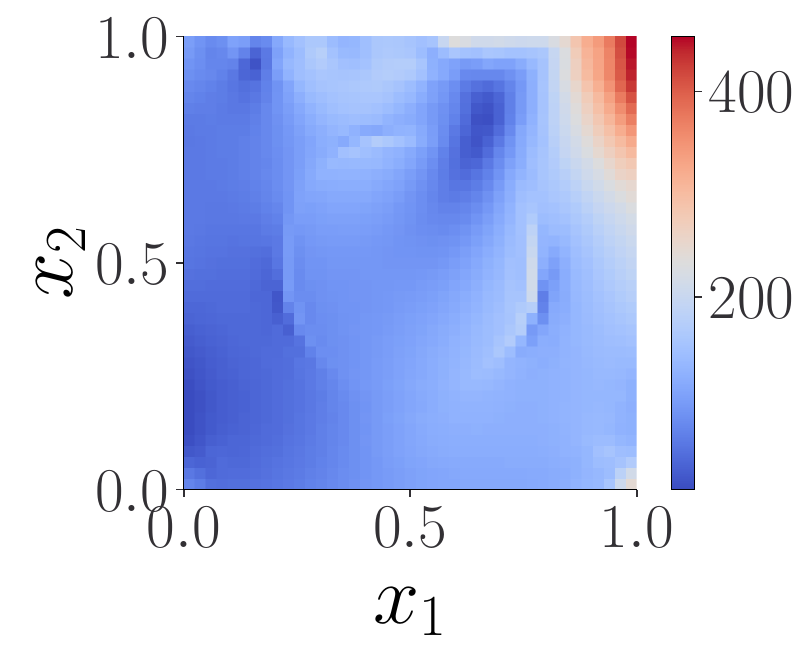}}}%
    \subfloat[\centering Predicted test]{{\includegraphics[width=0.195\columnwidth]{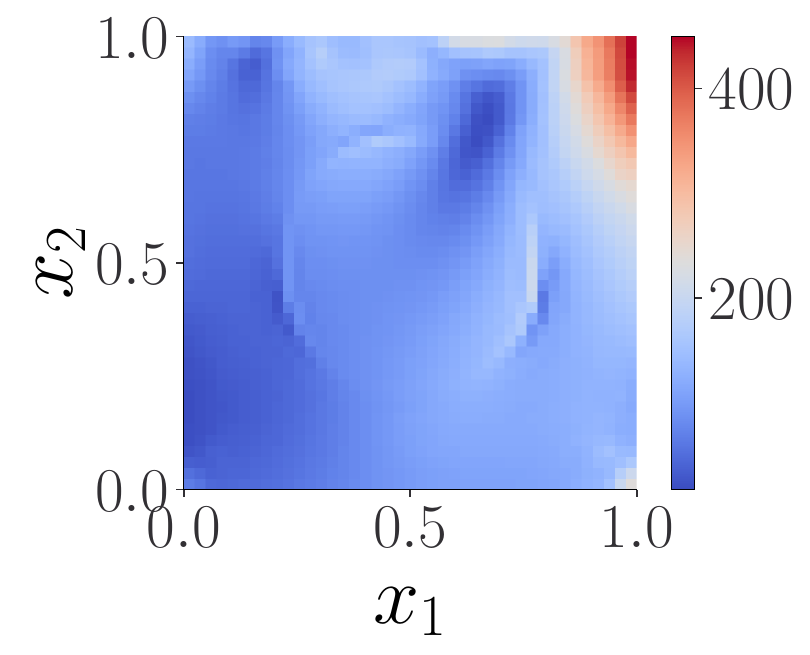}}}%5
    \subfloat[\centering Pointwise error]{{\includegraphics[width=0.185\columnwidth]{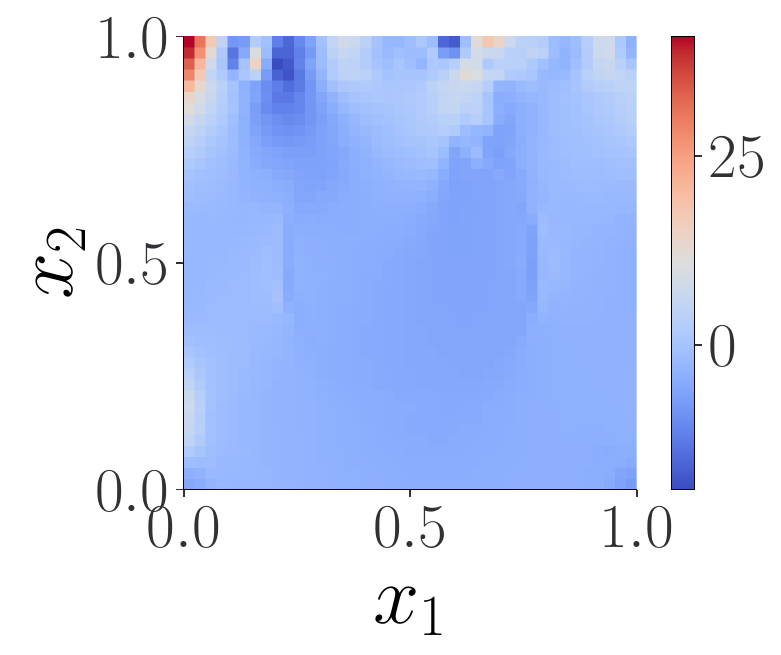}}}%
     \caption{Example of training data and test prediction and pointwise errors for the Structural Mechanics problem \eqref{eq:structural-mechanics}.}
    \label{fig:StrucA}%
\end{figure}

% \begin{figure}[htb]%
%     \centering

%     \caption{Example of test set prediction of the Structural Mechanics problem}%
%     \label{fig:StrucB}%
% \end{figure}
 
\subsubsection{Navier-Stokes equations}\label{sec: navier stokes}
Consider the vorticity-stream $(\omega, \psi)$ formulation of the incompressible Navier-Stokes equations:
\begin{equation}\label{eq:Navier-Stokes}
    \begin{aligned}
    \frac{\partial \omega}{\partial t} + (c \cdot \nabla)\omega - \nu \Delta \omega &= u, \quad
    \omega &= -\Delta\psi, \quad 
    \int_D \psi &= 0, \quad 
    c &= \left(\frac{\partial \psi}{\partial x_2}, -\frac{\partial \psi}{\partial x_1} \right)
    \end{aligned}
\end{equation}
where $\mathcal{D} = [0,2\pi]^2$,  periodic boundary conditions are considered 
and the initial condition $w(\cdot, 0)$ is fixed.
Here we are interested in the mapping from the forcing term $u$ to $v= \omega(\cdot, T)$, the 
vorticity field at a given time $t = T$, i.e., $\Gc^\dagger: u \mapsto \omega(\cdot, T)$.

% The map of interest is the map from the forcing term to the vorticity field at a given time $t = T$:
% \begin{equation}
%     \mathcal{G}^\dagger: f \mapsto w(\cdot, T)
% \end{equation}
% with the initial condition $w(\cdot, 0)$ fixed throughout. 
% In the dataset, the forcing is sampled from a centered gaussian field, 
The distribution $\mu$ is 
$\mathcal{GP}(0, (-\Delta +3^2 I)^{-4})$. The viscosity $\nu$ is fixed and equal to $0.025$, and the equation is solved on a $64\times 64$ grid with a pseudo-spectral method and Crank-Nicholson time integration; see \cite{de2022cost} for further implementation details. \Cref{fig:NSA} shows an example of input and output in the test set, along with an example of test prediction and pointwise errors.
% of the map $\mathcal{G}^\dagger$, and an example prediction and prediction error made at a test point.
\begin{figure}[htb]%
    \centering
    \subfloat[\centering Training input]{{\includegraphics[width=0.195\columnwidth]{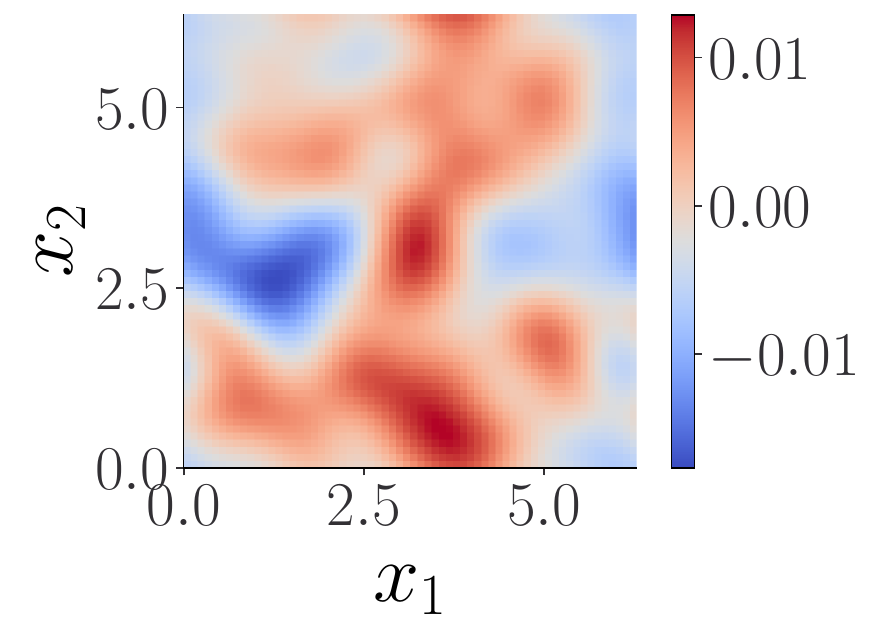} }}%
    \subfloat[\centering Training output]{{\includegraphics[width=0.195\columnwidth]{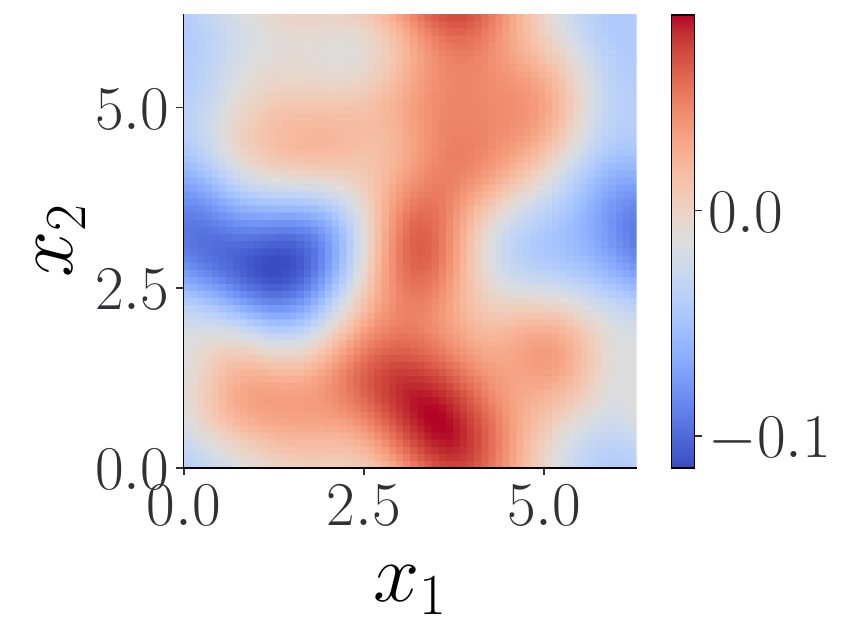} }}%
        \subfloat[\centering True test]{{\includegraphics[width=0.195\columnwidth]{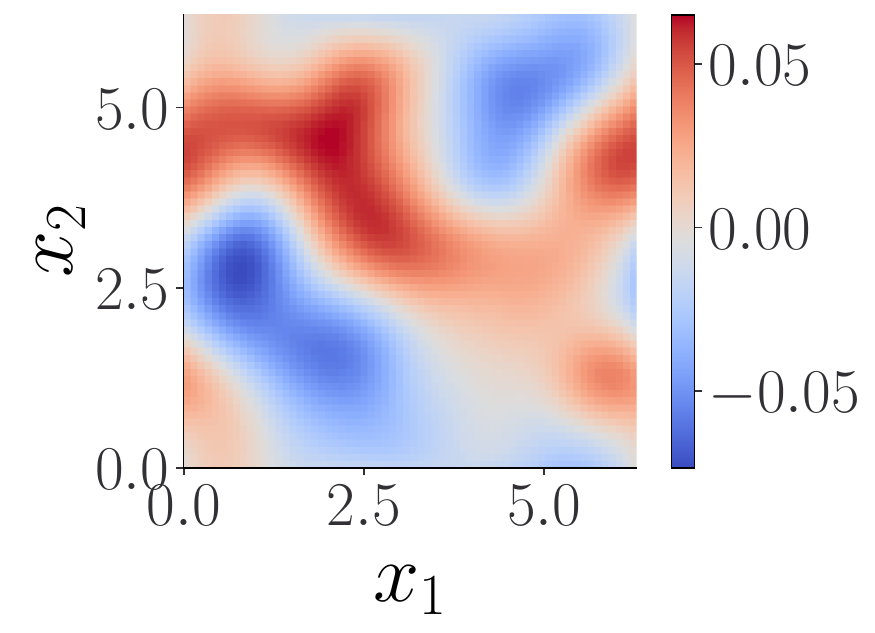}}}%
    \subfloat[\centering Predicted test]{{\includegraphics[width=0.195\columnwidth]{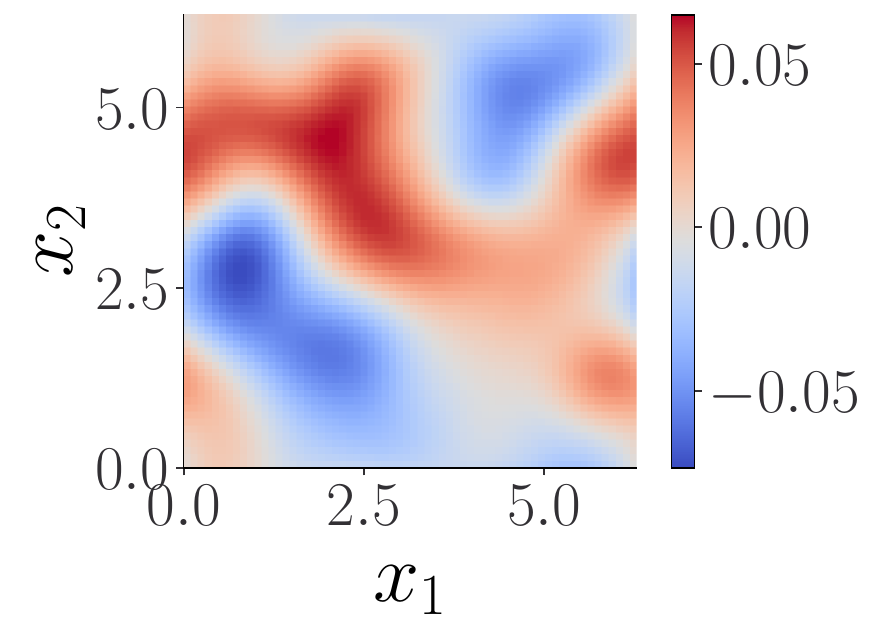}}}%
    \subfloat[\centering Pointwise error]{{\includegraphics[width=0.185\columnwidth]{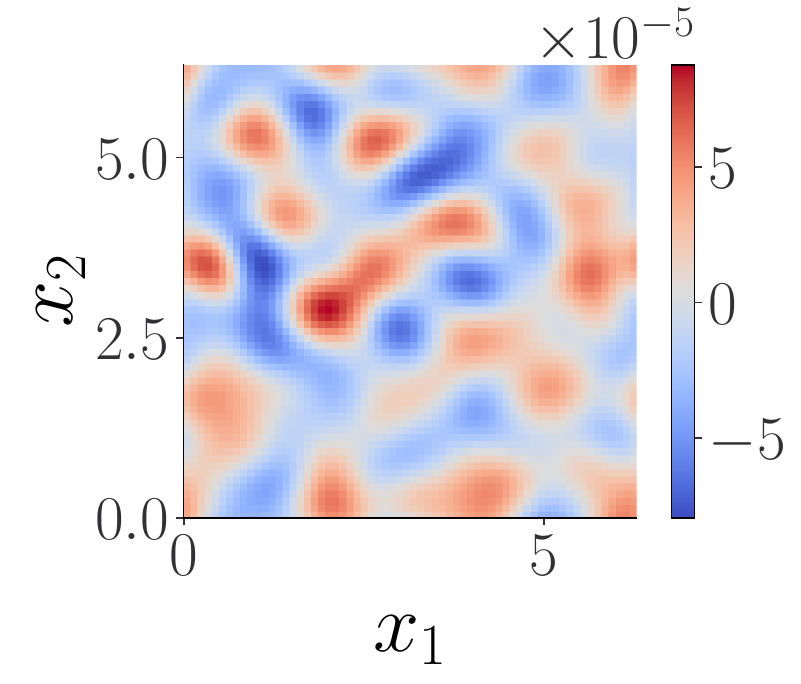}}}%
    \caption{Example of training data and test prediction and pointwise errors for the Navier-Stokes problem \eqref{eq:Navier-Stokes}.}%
    \label{fig:NSA}%
\end{figure}

% \begin{figure}[htb]%
%     \centering

%     \caption{Example of test set prediction of the Navier Stokes problem}%
%     \label{fig:NSB}%
% \end{figure}

\subsection{Results and discussion} \label{sec: discussion}
% \bh{ In practice, we did not observe a significant difference in performance once the kernel parameters were tuned, indicating that a large class of kernels are effective for these problems. Furthermore, the parameter tuning step is robust, we observe our results to be consistent in a reasonable range of lengthscale.}
Below we discuss our main findings in benchmarking our kernel method 
against state-of-the-art NN based techniques

\subsubsection{Performance against NNs}
\Cref{results} summarizes the $L^2$ relative test error of our vanilla
implementation of the kernel method along with those of DeepONet, 
FNO, PCA-Net, and PARA-Net. We observed that our vanilla kernel method was 
reliable in terms of accuracy across all examples. In particular, 
observe that between the Mat{\'e}rn or rational quadratic kernel,  we always managed to get close 
to the other methods, see for example the results for 
the Burgers' equation or Darcy problem, 
and even outperform them in several examples 
such as  Navier-Stokes and Helmholtz. 
Overall we observed that the performance of the kernel method is stable across all examples suggesting that our method is reliable and provides a good baseline for a large class of problems. Moreover, we  did not observe a significant difference in performance  
in terms of the choice of the particular kernel family once the hyper-parameters were tuned. This indicates that a large class of kernels are effective for these problems. Furthermore,  we found the hyper-parameter tuning to be robust, i.e., 
results were consistent in a reasonable range of parameters such as length scales.

In the high data regime, we found the vanilla kernel method to be 
the most accurate, although this comes with a greater cost, as seen in 
\Cref{fig: acc-comp}. However, the kernel method appears to provide the highest accuracy for its level of complexity as the accuracy of NNs typically stagnates or 
even decreases after a certain level of complexity; see the Navier-Stokes 
and Helmholtz panels of \Cref{fig: acc-comp} where most of the NN methods 
seem to plateau after a certain complexity level.

We also observed that the linear model did not provide the best accuracy as it quickly saturated in performance. Nonetheless, it provided surprisingly good accuracy at low levels of complexity: for example, in the case of Navier-Stokes, the linear 
kernel provided the best accuracy below $10^6$ FLOPS of complexity. This indicates that while simple, the linear model can be a valuable low-complexity model. Another notable example is the  Advection equation (both I and II), where the operator $\mathcal{G}^\dagger$ is linear. In this case, the linear kernel had the best accuracy and the best complexity-accuracy tradeoff. We note, however, that while the linear model was close to machine precision on Advection I (error on the order $10^{-13}\%$), its performance was significantly worse on Advection II (error on the order of $10\%$). Moreover, the gap between the linear kernel and all other models was significantly smaller for Advection I; we conjecture this difference in performance is likely 
due to the setup of these problems. 

% This observation supports the principle that the operator learning problem depends both on the operator $\mathcal{G}^\dagger$ and the measure $\mu$ from which the data is sampled. 

Finally, we note that the most challenging problem for our kernel method was the Structural Mechanics example. In this case, the vanilla kernel method has higher complexity but did not beat the NNs. In fact, the NNs seem to be able to 
reduce complexity without loss of accuracy compared to our  method.

\begin{table}[htb]
\centering
\resizebox{\textwidth}{!}{
\begin{tabular}{|l|c|c|c|c|c|c|c|}
\hline
&
\multicolumn{3}{|c|}{\textbf{Low-data regime}}& 
\multicolumn{4}{|c|}{\textbf{High-data regime}} \\
\hline
& \textbf{Burger's} & \textbf{Darcy problem} & \textbf{Advection I} & \textbf{Advection II} & \textbf{Hemholtz} & \textbf{Structural Mechanics} & \textbf{Navier Stokes} \\ \hline
\textbf{DeepONet} & 2.15\% & 2.91\% & 0.66\% & 15.24\% & 5.88\% & 5.20\% & 3.63\% \\ \hline
\textbf{POD-DeepONet} & 1.94\% & 2.32\% & 0.04\% & n/a & n/a & n/a & n/a \\ \hline
\textbf{FNO} & 1.93\% & 2.41\% & 0.22\% & 13.49\% & 1.86\% & 4.76\% & 0.26\% \\ \hline
\textbf{PCA-Net} & n/a & n/a & n/a & 12.53\% & 2.13\% & 4.67\% & 2.65\% \\ \hline
\textbf{PARA-Net} & n/a & n/a & n/a & 16.64\% & 12.54\% & 4.55\% & 4.09\% \\ \hline
\hline
\textbf{Linear} & 36.24\% & 6.74\% & $2.15\times10^{-13}\%$ & 11.28\% & 10.59\% & 27.11\% & 5.41\% \\ \hline
% \textbf{Quadratic} & 33.3\% & 10.01\% & 0.011\% & WIP & WIP & WIP & WIP \\ \hline
% \textbf{Cubic} & 24.4\% & n/a? & n/a? & WIP & WIP & WIP & WIP \\ \hline
\textbf{Best of Matérn/RQ} & 2.15\% & 2.75\% & $2.75\times 10^{-3}\%$ & 11.44\% & 1.00\% & 5.18\% & 0.12\% \\ \hline
%\textbf{Inv. quadratic} & n/a? & 2.87\% & n/a? & WIP & WIP & WIP & WIP \\ \hline
\end{tabular}
}
\caption{Summary of numerical results:  we report the $L^2$ relative test error of our numerical experiments and compare the kernel approach with variations of DeepONet , FNO, PCA-Net and PARA-Net. We considered two choices of  the kernel $S$, the rational quadratic and the Mat\'{e}rn, but we observed little difference between the two.}
\label{results}
\end{table}

\begin{figure}[htp]
    \centering
    \includegraphics[width = .95\columnwidth]{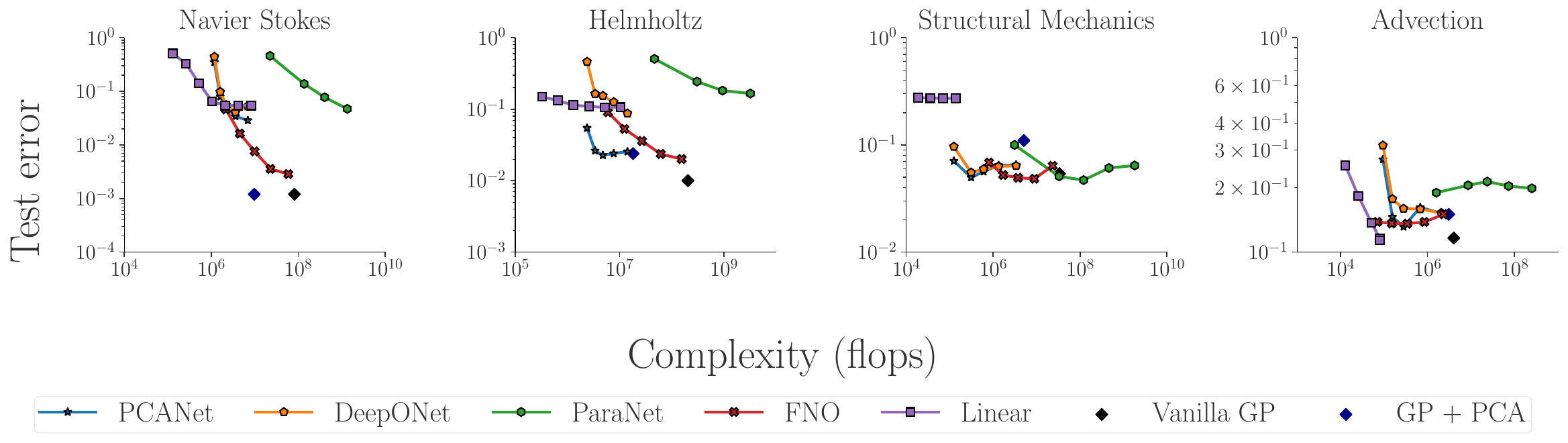}
    \caption{Accuracy complexity tradeoff achieved in the problems in \cite{de2022cost}. Data for NNs was obtained from the aforementioned 
    article. 
   Linear model refers to the linear kernel, 
    vanilla GP is our implementation with the nonlinear kernels and 
    minimal preprocessing, GP+PCA
    corresponds to preprocessing through PCA both the input and the output to reduce complexity.}
    
    % In linear models, the number of PCA components is changed. Vanilla GP refers to the standard gaussian process regression, when the input is preprocessed through PCA only when needed for computational reasons, and GP + PCA corresponds to preprocessing through PCA both the input and the output to reduce complexity}
    \label{fig: acc-comp}
\end{figure}

\subsubsection{Effect of preconditioners}
\Cref{table: cho-pca} compares the performance of our method with the 
Mat{\'e}rn kernel family using various preconditioning steps. 
Overall we observed that both PCA and Cholesky preconditioning improved 
the performance of our vanilla kernel method. 

The  Cholesky preconditioning generally offers the greatest improvement. However, we observed that getting the best results from the Cholesky approach required careful tuning of the parameters of the kernels $K$ and $Q$ which we did using 
cross-validation. While tuning the parameters does not increase the inference complexity, it does increase the training complexity. 

On the other hand, the PCA approach was more robust to changes in hyperparameters, 
i.e., the number of PCA components following \Cref{sec: dim red}. 
We observed that applying PCA on the input and output reduces complexity and has varying levels of effectiveness in providing a better cost-accuracy tradeoff. For example, for Navier-Stokes, it greatly reduced the complexity without affecting accuracy. But for the Helmholtz and Advection equations, PCA reduced the accuracy while remaining competitive with NN models. For structural mechanics, however, PCA significantly reduced accuracy and was worse than other models. We hypothesize that the loss in accuracy can be related to the decay of the eigenvalues of the PCA matrix 
in that example.

\md{

\begin{table}[htp]
\centering
\begin{tabular}{|l|c|c|c|} 
\hline
& \textbf{Advection II} &\textbf{Burger's} & \textbf{Darcy problem}  \\ \hline
\textbf{No preprocessing} & 14.37\% & 3.04\% & 4.47\%\\ \hline
\textbf{PCA} & 14.50\% & 2.41\% & 2.89\%\\ \hline
\textbf{Cholesky} & 11.44\% & 2.15\% & 2.75\%\\ \hline
\end{tabular}
\caption{Comparison between Cholesky preconditioning and PCA dimensionality reduction on three examples for our vanilla kernel implementation with the Mat\'{e}rn kernel.}
\label{table: cho-pca}
\end{table}

% \paragraph{Preconditioning} Overall, we observe that both PCA and Cholesky preconditioning improve the performance of the vanilla kernel method: see Table \ref{table: cho-pca}. The  Cholesky preconditioning generally offers the greatest improvement. However, we observed that getting the best results from the Cholesky approach required careful tuning of the parameters of the kernels $K$ and $Q$. In practice, we tune these parameters using cross-validation. While tuning the parameters does not increase the complexity at inference, it does increase the complexity of training. On the other hand, the PCA approach was more robust to changes in hyperparameters (in this case, the number of PCA components). Choosing the number of PCA components as in Section \ref{sec: dim red} or with grid search is a reliable and cheap way to obtain good performance and could be preferred in cases where low training complexity. 
}

\section{Conclusions}\label{sec: conclusions}

In this work we presented a kernel/GP framework for the learning of 
operators between function spaces. We presented an abstract formulation of 
our kernel framework along with convergence proofs and error bounds in 
certain asymptotic limits. Numerical experiments and benchmarking against 
popular NN based algorithms revealed that our vanilla implementation of 
the kernel approach is competitive and either matches the performance of 
NN methods or beats them in several benchmarks. Due to simplicity of implementation, 
flexibility, and the empirical results, we suggest that the proposed kernel methods 
are a good benchmark for future, perhaps more sophisticated, algorithms. Furthermore, these methods can be used to guide practitioners in the design of 
new and challenging benchmarks (e.g, identify problems where  vanilla kernel methods do not perform well).
Numerous directions of future research exist. In the theoretical direction it is 
interesting to remove the stringent \Cref{cond-input-data-simplified}
and we anticipate this to require a particular selection of the kernel employed to obtain
the map $\bar{f}$. Moreover, obtaining error bounds for more 
general measurement functionals beyond pointwise evaluations would be interesting. 
One could also adapt our framework to non-vanilla kernel methods such as random features or inducing point methods to provide a low-complexity alternative to NNs 
in the large-data regime. Finally, since the proposed approach is essentially a generalization of GP Regression to the infinite-dimensional setting, we anticipate that some of the hierarchical techniques of \cite{owhadi2017multigrid, schafer2021KL, schafer2021compression} could be extended to this setting and provide
 a better cost-accuracy trade-off than current methods.

% {\color{red}
% 

\subsection*{Acknowledgments}
MD, PB, and HO acknowledge support by the Air Force Office of Scientific Research under MURI award number FA9550-20-1-0358 (Machine Learning and Physics-Based Modeling and Simulation).
BH acknowledges support by the National Science Foundation grant number 
NSF-DMS-2208535 (Machine Learning for Bayesian Inverse Problems).
HO also acknowedges support by the Department of Energy under award number DE-SC0023163 (SEA-CROGS: Scalable, Efficient and Accelerated Causal Reasoning Operators, Graphs and Spikes for
Earth and Embedded Systems). We thank F. Sch\"{a}fer for comments and references.

\bibliographystyle{siamplain}
\bibliography{references}
\appendix 

\section{Review of operator valued kernels and GPs}\label{secOpvalker}
We review  the theory of operator valued kernels and GPs 
\cite{owhadi2023ideas}
as these are utilized throughout the article.
Operator-valued kernels were introduced in \cite{kadri2016operator} as a generalization of vector-valued kernels \cite{alvarez2012kernels}.

\subsection{Operator valued kernels}
 Let $\U$ and $\V$ be  separable Hilbert spaces endowed with the inner products $\<\cdot,\cdot\>_\U$ and $\<\cdot,\cdot\>_\V$.
Write $\Lc(\V)$ for the set of bounded linear operators mapping $\V$ to $\V$.

\begin{Definition}
We call
$
G\,:\, \U\times \U\rightarrow \Lc(\V)
$
  an \say{operator-valued kernel} if
\begin{enumerate}
\item
$G$ is Hermitian, i.e.
$G(u,u')=G(u',u)^T$ for all $u,u'\in \U\,$, 
writing $A^T$ for the adjoint of the operator $A$ with respect 
to $\<\cdot,\cdot\>_\V$.

\item $G$ is non-negative, i.e.,  for all $m\in \mathbb{N}$ and 
any set of points $(u_i,v_i)_{i=1}^m \subset \U \times \V $ it holds that
 $\sum_{i,j=1}^m   \<v_i, G(u_i,u_j) v_j\>_\V\geq 0.$
 \end{enumerate}
 \end{Definition}
We call $G$ non-degenerate if $\sum_{i,j=1}^m   \<v_i, G(u_i,u_j) v_j\>_\V= 0$ implies $v_i=0$ for all $i$ whenever  $u_i\not=u_j$ for $i\not=j$.

\subsection{RKHSs}

Each non-degenerate, locally bounded and separately continuous  operator-valued kernel $G$  is in one to one correspondence with an RKHS $\Hc$ of continuous operators $\Gc\,:\, \U\rightarrow \V$  obtained as the closure of the linear span of the 
maps $z \mapsto G(z,u)v$  with respect to the inner product identified by the reproducing property
\begin{equation}\label{eqrepprop}
\<g, G(\cdot,u) v\>_\Hc=\<g(u),v\>_\V
\end{equation}

\subsection{Feature maps}\label{secfmaps}
Let $\F$ be a  separable Hilbert space (with inner product $\<\cdot,\cdot\>_\F$ and norm $\|\cdot\|_\F$) and let
$\psi\,:\, \U \rightarrow \Lc(\V,\F)$ be a continuous function mapping $\U$ to the space of bounded linear operators from
$\V$ to $\F$.
\begin{Definition}
We say that $\F$ and $\psi\,:\, \U \rightarrow \Lc(\V,\F)$ are a \emph{feature space} and  a \emph{feature map} for the kernel $G$ if, for all $(u,u',v,v')\in \U^2 \times \V^2$,
\begin{equation*}
\< v,  G(u,u')v' \>= \<\psi(u) v,\psi(u') v'\>_\F\,.
\end{equation*}
\end{Definition}
Write $\psi^T(u)$, for the adjoint of $\psi(u)$ defined as the linear function mapping $\F$ to $\V$ satisfying
\begin{equation*}
\<\psi(u) v, \alpha\>_\F=\< v, \psi^T(u)\alpha\>_\V
\end{equation*}
 for
 $u,v,\alpha \in \U\times \V \times \F$. Note that $\psi^T\,:\, \U\rightarrow \Lc(\F,\V)$ is therefore a function mapping $\U$ to the space of bounded linear functions from $\F$ to $\V$. Writing $\alpha^T \alpha':=\<\alpha,\alpha'\>_\F$ for the inner product in $\F$  we can ease our notations by writing
 \begin{equation}\label{eqkledkjdkejdd}
 G(u,u')= \psi^T(u) \psi(u')\,
\end{equation}
 which is consistent with the finite-dimensional setting and $v^T G(u,u') v'=(\psi(u) v)^T (\psi(u')v')$ (writing $v^T v'$ for the inner product in $\V$).
 For $\alpha \in \F$ write $\psi^T \alpha$ for the function  $\U\rightarrow \V$ mapping $u\in \U$ to the element $v\in \V$ such that
 \begin{equation*}
 \<v',v\>_\V=\<v',\psi^T(u)\alpha\>_\V=\<\psi(u) v',\alpha\>_\F \text{ for all }v'\in \V\,.
 \end{equation*}
 We can, without loss of generality, restrict $\F$ to be the range of $(u,v)\rightarrow \psi(u)v$ so that the RKHS $\Hc$ defined by
 $G$ is the closure of the pre-Hilbert space spanned by $\psi^T \alpha $ for $\alpha \in \F$.
 Note that the reproducing property \eqref{eqrepprop} implies that for $\alpha\in \F$
\begin{equation*}
\<\psi^T(\cdot) \alpha, \psi^T(\cdot) \psi(u) v\>_\Hc=\<\psi^T(u) \alpha,v\>_\V=\<\alpha, \psi(u) v\>_\F
\end{equation*}
for all $u,v\in \U\times \V$, which leads to the following theorem.

\begin{Theorem}\label{thmkjhkejhddjhd}
The RKHS $\Hc$ defined by the kernel \eqref{eqkledkjdkejdd} is the linear span of $\psi^T \alpha$ over  $\alpha \in \F $ such that $\|\alpha\|_\F< \infty$. Furthermore, $\<\psi^T(\cdot) \alpha, \psi^T(\cdot) \alpha'\>_\Hc=\<\alpha,\alpha'\>_\F$ and
\begin{equation*}
\|\psi^T(\cdot) \alpha\|^2_\Hc=\|\alpha\|^2_\F
\text{ for }\alpha,\alpha'\in \F\,.
\end{equation*}
\end{Theorem}

\subsection{Interpolation}
Let us consider the interpolation problem in operator valued RKHSs.
\begin{Problem}\label{pb828827hee}
Let $\Gc^\dagger$ be an unknown continuous operator mapping $\U$ to $\V$. Given the information\footnote{For a $N$-vector $\ub=(u_1,\ldots,u_N)\in \U^N$ and a function $\Gc\,:\, \U\rightarrow \V$, write $\Gc(\ub)$ for the $N$ vector with entries
$\big(\Gc(u_1),\ldots,\Gc(u_N)\big)$.}
  $\Gc^\dagger(\ub)=\vb$ with the data $(\ub,\vb)\in \U^N\times \V^N$, approximate $\Gc^\dagger$.
\end{Problem}

Using the relative error in $\|\cdot\|_{\Hc}$-norm as a loss, the minimax optimal recovery solution of \Cref{pb828827hee}
 is, by \cite[Thm.~12.4,12.5]{owhadi_scovel_2019}, given by
\begin{equation}\label{eqhgvygvgyvo}
\begin{cases}
\text{Minimize }&\|\Gc\|_{\Hc}^2\\
\text{subject to }&\Gc(\ub)=\vb
\end{cases}
\end{equation}
The minimizer  is then of the form $
 \Gc(\cdot)=\sum_{j=1}^N G(\cdot,u_j) w_j\,,$
  where the coefficients $w_j \in \V$ are identified by solving the  system of linear equations
 $\sum_{j=1}^N G(u_i,u_j) w_j=v_i\text{ for all }i\in \{1,\ldots,N\}\,.$ 
 Using our compressed notation we can rewrite this equation as 
 $G(\ub,\ub) {\bf w}=\vb$ where ${\bf w}=(w_1,\ldots,w_N),\,\vb=(v_1,\ldots,v_N)\in \V^N$ and $G(\ub,\ub)$ is the $N\times N$ block-operator matrix
 \footnote{ For $N\geq 1$ let $\V^N$ be the N-fold product space endowed with the
 inner-product $\<\vb,{\bf w}\>_{\V^N}:=\sum_{i,j=1}^N \<v_i,w_j\>_\V$ for
 $\vb=(v_1,\ldots,v_N), {\bf w}=(w_1,\ldots,w_N) \in \V^N$.
  ${\bf A}\in \Lc(\V^N)$ given by
  $
  {\bf A}=\begin{pmatrix}A_{1,1}&\cdots & A_{1,N}\\ \vdots & & \vdots\\A_{N,1}& \cdots & A_{N,N} \end{pmatrix}
 $
 where $A_{i,j}\in \Lc(\V)$, is called a block-operator matrix. Its adjoint  ${\bf A^T}$ with respect to  $\<\cdot,\cdot\>_{\V^N}$ is the
 block-operator matrix with entries $(A^T)_{i,j}=(A_{j,i})^T$.} 
 with entries $G(u_i,u_j)$. Therefore, writing $G(\cdot,\ub)$ for the vector $(G(\cdot,u_1),\ldots,G(\cdot,u_N))\in \Hc^N$,
 the optimal recovery interpolant is given by
 \begin{equation}\label{eqhgvygvgyv2b}
 \bar{\Gc}(\cdot)= G(\cdot,\ub)  G(\ub,\ub)^{-1} \vb\,,
 \end{equation}
which implies that the value of \eqref{eqhgvygvgyvo} at the minimum is
\begin{equation}\label{eqkjhejdehgdkdd}
\|\bar{\Gc}\|_\Hc^2=\vb^T G(\ub,\ub)^{-1} \vb\,,
\end{equation}
 where $G(\ub,\ub)^{-1}$ is the inverse of $G(\ub,\ub)$, whose existence is implied by the non-degeneracy of $G$ combined with
 $u_i\not=u_j$ for $i\not=j$.

\subsection{Ridge regression}\label{app:operator-valued-ridge-regression}
Let $\gamma>0$.
A ridge regression (approximate) solution  to Problem \ref{pb828827hee} 
can be found as the minimizer of
\begin{equation}\label{eqledhehdiudh}
\inf_{\Gc\in \Hc}\lambda\,\|\Gc\|_\Hc^2+\gamma^{-1}\sum_{i=1}^N\|v_i-\Gc(u_i)\|_\V^2\,.
\end{equation}
This minimizer is given by the formula
\begin{equation}\label{eqajkjwdhjbdjehr}
\bar{\Gc}(u)=G(u,\ub)\big(G(\ub,\ub)+\gamma I\big)^{-1} \vb\,,
\end{equation}
writing $I$ for the identity matrix. We can further compute directly
% and the value of \eqref{eqledhehdiudh} at the minimum is
\begin{equation*}%\label{eqlkjdkjweedjkb}
 \| \bar{\Gc} \|_\mathcal{H}^2 = \vb^T \big(G(\ub,\ub)+\gamma I\big)^{-1} \vb\,.
\end{equation*}

\subsection{Operator-valued GPs}\label{subsecuideydiueyd}

The following definition of operator-valued Gaussian processes is a natural extension of scalar-valued Gaussian fields \cite{owhadi_scovel_2019}.
\begin{Definition}\label{dejheekfklfhrf}\cite[Def.~5.1]{owhadi2023ideas}
Let $G\,:\, \U\times \U\rightarrow \Lc(\V)$ be an operator-valued kernel. Let $m$ be a function mapping $\U$ to $\V$.
We call
$\xi\,:\,  \U\rightarrow \Lc(\V,{\bf H})$ an operator-valued GP if
 $\xi$ is a function mapping
$u \in \U$ to $\xi(u)\in \Lc(\V,{\bf H})$ where ${\bf H}$ is a Gaussian space and $\Lc(\V,{\bf H})$ is the space of bounded linear operators from $\V$ to ${\bf H}$. Abusing notations we write $\<\xi(u),v\>_\V$ for $\xi(u)v$. We say that $\xi$ has mean $m$ and covariance kernel $G$ and write $\xi \sim \cN(m,G)$ if $\<\xi(u),v\>_\V \sim \cN\big(m(u),v^T G(u,u)v\big)$ and
\begin{equation}
\Cov\big(\<\xi(u),v\>_\V, \<\xi(u'),v'\>_\V\big)=v^T G(u,u')v'\,.
\end{equation}
We say that $\xi$ is centered if it is of zero mean.
\end{Definition}
If $G(u,u)$ is trace class ($\Tr[G(u,u)]<\infty$) then $\xi(u)$  defines a measure  on $\V$, i.e. a $\V$-valued random variable \footnote{Otherwise it only defines a (weak) cylinder-measure in the sense of Gaussian fields.}.

\begin{Theorem}\cite[Thm.~5.2]{owhadi2023ideas}
The law of an operator-valued GP is uniquely determined by its mean $m$ and covariance kernel $G$. Conversely given  $m$ and $G$ there exists an operator-valued GP having mean $m$ and covariance kernel $G$. In particular if $G$ has feature space $\F$ and map $\psi$, the $e_i$ form an orthonormal basis of $\F$,  and the $Z_i$ are i.i.d. $\cN(0,1)$ random variables, then
$\xi=m +\sum_i Z_i \psi^T e_i$ 
is an operator-valued GP with mean $m$ and covariance kernel $G$.
\end{Theorem}

\begin{Theorem}\cite[Thm.~5.3]{owhadi2023ideas}
Let $\xi$ be a centered operator-valued GP with covariance kernel $G\,:\,\U\times \U\rightarrow \Lc(\V)$. Let $\ub,\vb\in \U^N\times \V^N$. Let
$Z=(Z_1,\ldots,Z_N)$ be a random Gaussian vector, independent from $\xi$, with i.i.d. $\cN(0,\gamma I_\V)$ entries ($\gamma\geq 0$ and $I_\V$ is the identity map on $\V$). Then $\xi$ conditioned on $\xi(\ub)+Z$ is an operator-valued GP with mean
\begin{equation}\label{eqlkwjldkjhjedkhe}
\E\big[\xi(u)\big| \xi(\ub)+Z=\vb\big]=G(u,\ub)\big(G(\ub,\ub)+\gamma I_\V\big)^{-1} \vb=\eqref{eqajkjwdhjbdjehr}
\end{equation}
and conditional covariance operator
\begin{equation}\label{eqkjwbldkdbkld}
G^\perp(u,u'):=G(u,u')-G(u,\ub)\big(G(\ub,\ub)+\gamma I_\V\big)^{-1}G(\ub,u')\,.
\end{equation}
In particular, if $G$ is trace class, then
\begin{equation}\label{eqewjhlkefjlwkf}
\sigma^2(u):=\E\Big[\big\|\xi(u) -\E[\xi(u)| \xi(\ub)+Z=\vb]\big\|_\V^2\Big| \xi(\ub)+Z=\vb\Big]=\Tr\big[G^\perp(u,u)\big]\,.
\end{equation}
\end{Theorem}

\subsection{Deterministic error estimates for operator-valued regression}\label{subsecprobahoeuwgediu}
The following theorem shows that the standard deviation \eqref{eqewjhlkefjlwkf} provides deterministic a prior error bounds on the accuracy of the ridge regressor \eqref{eqlkwjldkjhjedkhe} to $\Gc^\dagger$ in Problem \ref{pb828827hee}. Local error estimates such as \eqref{eqllbhjdebjdbd} below are classical in the Kriging literature \cite{wu1993local} where $\sigma^2(u)$ is known as the power 
function/kriging variance; see also
\cite{owhadi2015bayesian}[Thm.~5.1] for applications to PDEs.

\begin{Theorem}\label{thmalkjhbahkjedb}\cite[Thm.~5.4]{owhadi2023ideas}
Let $\Gc^\dagger$ be the unknown function of Problem \ref{pb828827hee} and let $\Gc(u)=\eqref{eqlkwjldkjhjedkhe}=\eqref{eqajkjwdhjbdjehr}$ be its ridge regressor. Let $\Hc$ be the RKHS associated with $G$ and let $\Hc_\gamma$ be the RKHS associated with the kernel $G_\gamma:=G+\gamma I_\V$. It holds true that
\begin{equation}\label{eqllbhjdebjdbd}
\big\|\Gc^\dagger(u)-\Gc(u) \big\|_\V \leq \sigma(u) \|\Gc^\dagger\|_\Hc
\end{equation}
and
\begin{equation}\label{eqllbhjdebddjdbd}
\big\|\Gc^\dagger(u)-\Gc(u) \big\|_\V \leq \sqrt{\sigma^2(u)+\gamma \operatorname{dim}(\V)} \|\Gc^\dagger\|_{\Hc_\gamma}\,,
\end{equation}
where $\sigma(u)$ is the standard deviation \eqref{eqewjhlkefjlwkf}.
\end{Theorem}

\section{An alternative regularization of operator regression}\label{sec:alternative-regularization}
For $\gamma>0$, the regularization implied by \eqref{eqiheuddh1} is equivalent to adding noise on the
$\varphi(\vb)$ measurements. If one could observe $\vb$ (and not just $\varphi(\vb)$), then an alternative approach to regularizing the problem is to add noise to $\xi(\ub)$. To describe this
let $Z'=(Z_1',\ldots,Z_N')$ be a random block-vector, independent from $\xi$, with i.i.d. 
entries $Z'_j \sim \cN(0,\gamma I_\V)$ for $j = 1, \dots, N$ (where $I_\V$ denotes the identity map on $\V$).
Then the GP $\xi$ conditioned on $\xi(\ub)=\vb+Z'$ is a GP with conditional covariance kernel \eqref{eqkjwbldkdbkld} and conditional mean $\tilde{\Gc}_\gamma$=\eqref{eqajkjwdhjbdjehr} that is also the minimizer of \eqref{eqledhehdiudh}.
 Observing\footnote{This follows from $\varphi(Z_i')\sim \cN(0, \gamma \varphi \varphi^T )$ where $\varphi^T$ is the adjoint of $\varphi$ identified as the linear map from $\R^{m}$ to $\V$ satisfying $\< W,\varphi(w)\>_{\R^{m}}=\< \varphi^\perp W,w\>_{\V}$ for $w\in \V$ and $W\in \R^{m}$ (i.e., $\varphi^T(W)=(K\varphi) W $).} that $\varphi(Z_i')\sim \cN(0, \gamma K(\varphi,\varphi))$ , we deduce that
 $\tilde{\Gc}_\gamma=\chi\circ \tilde{f}_\gamma\circ \phi$ where $\tilde{f}_\gamma$  minimizes
\begin{equation}\label{eqiheuddh2guyg}
\begin{cases}
\text{Minimize }& \|f\|_\Gamma^2+\gamma^{-1} \sum_{i=1}^N (f(U_i)-V_i)^T K(\varphi,\varphi )^{-1} (f(U_i)-V_i)\,.\\
\text{Over } &f \in \Hc_\Gamma\,.
\end{cases}
\end{equation}
Furthermore, the distribution of
$\xi$ conditioned on $\xi(\ub)=\vb+Z'$ is that of $\chi \circ \tilde{\zeta}^\perp \circ \phi$ where $\tilde{\zeta}^\perp\sim \cN(\tilde{f}_\gamma, \tilde{\Gamma}^\perp)$ is the GP $\zeta$ conditioned on $\zeta(\Ub)=\Vb+\varphi(Z')$, whose mean is $\tilde{f}_\gamma$ and conditional covariance kernel is
$\tilde{\Gamma}^\perp(U,U')=\Gamma(U,U')-\Gamma(U,\Ub) (\Gamma(\Ub,\Ub)+\gamma A)^{-1} \Gamma(\Ub,U')$ where $A$ is a $N\times N$ block diagonal matrix with $K(\varphi,\varphi )$ as diagonal entries.

% \subsection{Linear regression and linear kernels}
% We will also consider a linear model
% \begin{align}
%     A: \:  &\mathbb{R}^{d_\phi} \to \mathbb{R}^{\varphi}  \\
%     &  \pmb{u} \mapsto \pmb{v}
% \end{align}
% The linear model is given by the (unique) solution of the least-squares problem
% \begin{equation}
%     A^{\dagger} = \argmin_{A \in \mathbb{R}^{d_\phi \times d_\varphi}} \norm{\bold{U}A - \bold{V}}^2_{2}
% \end{equation}
% where $\norm{B}_2^2 = \sum_{i,j=1}^N b^2_{ij}$.

% For $\lambda \geq 0$, the regularized solution is given by

% \begin{equation}
%     A^{\dagger} = \argmin_{A \in \mathbb{R}^{d_\phi \times d_\varphi}} \norm{\bold{U}A - \bold{V}}^2_{2} + \lambda \norm{A}^2_2.
% \end{equation}
% The solution to this problem is implemented in various packages. We use the implementation of \cite{jax} and \cite{scikit-learn}. See \href{https://jax.readthedocs.io/en/latest/_autosummary/jax.numpy.linalg.lstsq.html}{JAX} and \href{https://scikit-learn.org/stable/modules/generated/sklearn.linear_model.LinearRegression.html}{Scikit-Learn} for greater details.

% \begin{remark}
% The linear models  can be shown to be mathematically equivalent to a kernel method with a kernel $\gamma$ defined by
% \begin{equation}
%     \gamma(\pmb{u}_1, \pmb{u}_2)  =\pmb{u}_1^\intercal \pmb{u}_2.
% \end{equation}
% While mathematically equivalent, it is usually more computationally efficient to consider the above formulation with the dimension $d_\phi, d_\varphi$ as opposed to the number of data points N.
% \end{remark}

\section{Expressions for the kernels used in experiments}\label{app: kernels}

Below we collect the expressions for the kernels that were referred to 
in the article or utilized for our numerical experiments. 
These can be found in many standard textbooks on GPs such as \cite{rasmussen}.

\subsection{The linear kernel} The linear kernel has 
the simple expression $K_{\text{linear}}(x, x') =  \< x, x'\> $ and 
may be defined on any inner product space. It has no hyper-parameters.

\subsection{The rational quadratic kernel} The rational quadratic kernel has the expression
$K(x, x') = k_{\text{RQ}}( \| x - x' \|)$ where 
\begin{equation}
    k_{\text{RQ}}(r)  = \bigg( 1+ \frac{r^2}{2 l^2} \bigg)^{-\alpha}.
\end{equation}
It has hyper-parameters $\alpha >0$ and $l$. 

\subsection{The Mat\'{e}rn parametric family}
The Mat{\'e}rn kernel family is of the form $K(x, x') = k(\| x - x'\|$  
where
\begin{equation}
    k_{\nu}(r) = \exp\bigg(-\frac{\sqrt{2\nu} r}{l}\bigg)\frac{\Gamma(p+1)}{\Gamma(2p+1)}\sum_{i=0}^p \frac{(p+1)!}{i!(p-i)!}\bigg( \frac{\sqrt{8\nu}r}{l}\bigg)^{p-i},
\end{equation}
for $\nu = p + \frac{1}{2}$. This kernel has hyper-parameters 
$p \in \mathbb{Z}_{+}$ and $l >0$.
In the limiting case where $\nu \to \infty$, the Mat\'{e}rn kernel, we obtain the Gaussian or squared exponential kernel:
\begin{equation}
    k_{\infty}(r) = \exp\bigg(-\frac{r^2}{2l^2}\bigg),
\end{equation}
with hyper-parameter $l >0$.

\end{document}